\documentclass[11pt]{article}

\usepackage{chngcntr}

\usepackage{amsmath}
\usepackage{enumerate}
\usepackage{fancyhdr}
\usepackage{color}
\usepackage{listings}
\usepackage{graphicx}
\usepackage{amssymb}
\usepackage{mathtools}

\usepackage{hyperref}
\usepackage{booktabs}

\usepackage[round]{natbib}

\usepackage[utf8]{inputenc}

\usepackage{amsfonts}
\usepackage{amssymb}
\usepackage{amsmath}
\usepackage{amsthm}
\usepackage{enumerate}
\usepackage{algorithm}
\usepackage{algorithmic}
\usepackage{tikz}
\usepackage{booktabs}
\usepackage{nicefrac}       
\usepackage{makecell}
\usepackage{array, multirow, boldline}

\newcommand{\vx}{\mathbf{x}}
\newcommand{\vy}{\mathbf{x}}
\newcommand{\vw}{\mathbf{w}}

\newcommand{\E}{\mathbb{E}}
\newcommand{\Var}{\mathbb{V}\text{ar}}

\newcommand{\cE}{\mathcal{E}}

\newcommand{\cN}{\mathcal{N}}

\newcommand{\bP}{\mathbb{P}}

\newcommand{\I}{\mathbb{I}}

\newcommand{\R}{\mathbb{R}}
\newcommand{\norm}[1]{\lvert\lvert #1\rvert\rvert}
\newcommand{\cO}[1]{\mathcal{O}\bigg(#1\bigg)}

\newcommand{\abs}[1]{\lvert #1 \rvert}
\newcommand{\ip}[2]{\langle #1,#2\rangle}

\newcommand{\floor}[1]{\lfloor #1 \rfloor}
\newcommand{\ceil}[1]{\lceil #1 \rceil}

\DeclareMathOperator*{\argmax}{arg\,max}
\DeclareMathOperator*{\argmin}{arg\,min}

\definecolor{lightgray}{gray}{0.5}

\usepackage[papersize={8.5in,11in},top=1in,bottom=0.75in,left=0.75in,right=0.75in]{geometry}

\usepackage{enumitem}
\usepackage{tablefootnote}

\usepackage[utf8]{inputenc} 
\usepackage[T1]{fontenc}    
\usepackage{hyperref}       
\usepackage{url}            
\usepackage{booktabs}       
\usepackage{amsfonts}       
\usepackage{nicefrac}       
\usepackage{microtype}      
\usepackage{xcolor}

\usepackage{amsmath}
\usepackage{amssymb}
\usepackage{mathtools}
\usepackage{amsthm}

\usepackage[capitalize,noabbrev]{cleveref}

\theoremstyle{plain}
\newtheorem{theorem}{Theorem}[section]

\newtheorem{lemma}[theorem]{Lemma}

\theoremstyle{definition}

\newtheorem{assumption}[theorem]{Assumption}
\theoremstyle{remark}
\newtheorem{remark}[theorem]{Remark}

\usepackage[textsize=tiny]{todonotes}

\usepackage{subfigure}
\usepackage{paralist}
\usepackage{tikz}
\usepackage{algorithmic}
\usepackage[subtle]{savetrees}

\title{LocalKMeans: Convergence of Lloyd's Algorithm \\ with Distributed Local Iterations}

\author{Harsh Vardhan\\ Computer Science and Engineering, \\
University of California, San Diego\\
hharshvardhan$@$ucsd.edu
\and Heng Zhu \\ Electrical and Computer Engineering, \\
University of California, San Diego\\
hez007$@$ucsd.edu\\
\and Avishek Ghosh \\
Computer Science and Engineering,\\
Indian Institute of Technology, Bombay,\\
avishek\_ghosh$@$iitb.ac.in 
\and 
Arya Mazumdar  \\
Hal\i c\i o\u glu Data Science Institute,\\
University of California, San Diego\\
 arya$@$ucsd.edu
}

\begin{document}
\maketitle
\begin{abstract}
In this paper, we  analyze the classical $K$-means alternating-minimization algorithm, also known as Lloyd's algorithm (Lloyd, 1956), for a mixture of Gaussians in a data-distributed setting that incorporates local iteration steps.  Assuming unlabeled data distributed across multiple machines, we propose an algorithm, LocalKMeans, that performs Lloyd's algorithm in parallel in the machines by running its iterations on local data, synchronizing only every $L$ of such local steps. We characterize the cost of these local iterations against the non-distributed setting, and show that the price paid for the local steps is a higher required signal-to-noise ratio. While local iterations were theoretically studied in the past for gradient-based learning methods, the  analysis of unsupervised learning methods is more involved owing to the presence of latent variables, e.g. cluster identities, than that of an iterative gradient-based algorithm. To obtain our results, we adapt a virtual iterate method to work with a non-convex, non-smooth objective function, in conjunction with a tight statistical analysis of Lloyd steps.

\end{abstract}
\vspace{-4mm}
\section{Introduction}
\vspace{-2mm}

The scale of data processing in modern applications has increased substantially over the last few decades. In response to this, distributed computing has become an integral component of machine learning and leveraging the power of local computing nodes has become an important direction of research. It is well known that in such a system, a major source of latency is the communication time between the compute nodes and the central server. Moreover, in framework such as Federated Learning~\cite{konecny2016federated}, this communication cost is directly linked with the bandwidth usage of the participating compute nodes and hence resource constrained.

One natural way to reduce communication cost is via \emph{local steps}, where the compute nodes (or {\em machines}) run several iterates of the local training algorithm before communicating to the central server. Indeed, when the training algorithm is stochastic gradient descent (SGD), the resulting algorithm, namely localSGD is immensely popular in distributed learning~\cite{stich2018local,bao2024provable,gorbunov_local_2021,haddadpour_local_2019,woodworth_minibatch_2020}.

The concept of local steps is not only pertinent to a particular learning algorithm (like SGD) or paradigm 
(like supervised learning), and can be extended to any data-driven iterative algorithm where data is distributed. In this paper, we use local steps for unsupervised methods, and quantify the benefit of local steps. To the best of our knowledge, such an approach has not been studied theoretically. 

Classically, for an unsupervised problem like clustering or mixture models, canonical algorithms like Alternating Minimization (AM) (see \cite{lloyd,vattani,yi_two,yi2016solving} or its soft variant Expectation Maximization (EM) (see \cite{balakrishnan_statistical_2017,das_ten} are employed. In particular, for the $K$-means clustering  (where the task is to partition a given dataset into  $K$ clusters such that a squared-error objective is minimized), the popular AM algorithm (or Lloyd's algorithm) is widely used. The algorithm alternatively estimates the $K$ cluster centers and assigns data points based on its distance from the cluster centers. It has been shown in \cite{lu_statistical_2016} that Lloyd's algorithm indeed converges when the dataset is a mixture of high-dimensional Gaussians, provided some technical conditions on separation and initialization are met.

In this work, we analyze the Lloyd's algorithm in the distributed learning framework with local steps. Our algorithm, namely LocalKMeans, performs $L$ local updates of Lloyd's on compute nodes before communicating to the central server and aggregates the estimates. As we can see, LocalKMeans reduces the communication cost by a factor of $L$. We quantify the cost of the local steps and show theoretically (as well as experimentally) that we do not get a significant penalty in terms of both the clustering error and the number of iterations, while gaining significantly in communication cost.

\textbf{Notation:} We use $[m]$ to denote the set $\{1,2,\ldots, m\}$ for any $m\in \mathbb{N}$. A $d$-dimensional random variable $X$ is dubbed $\sigma^2$ sub-Gaussian, denoted by $subG(\sigma^2)$ if for any $\lambda \in \R^d_{+}$, we have $\E[e^{\ip{\lambda}{X}}] \leq e^{- \sigma^2 \norm{\lambda}^2/2}$. We use $\norm{\cdot}$ to denote the $\ell_2$ norm and use superscript $(t)$ to index the iterations. 
Further, we use $\exp(-x)$ to denote  $ce^{-c'x}$ for some constants $c, c'>0$.

\subsection{Statistical Model}
\label{sec:setup}
We consider a distributed learning model with one central server and $m$ compute nodes each containing $n$ datapoints. Our statistical model extends \cite{lu_statistical_2016} to the distributed learning. The compute nodes can communicate with the server and cannot communicate among themselves. Let $\vx_{i,j} \in \R^d$ denote the $j^{th}$ datapoint on the $i^{th}$ compute node where $i\in [m]$ and $j\in [n]$. Each datapoint is generated independently from the mixture of $K$ sub-Gaussians, as described below.
\begin{align}
    \vx_{i,j} = \theta_{z_{i,j}} + \vw_{i,j},~\label{def:mix_subg}
\end{align}
where $z_{i,j}\in [K]$ is the true cluster identity of each datapoint  and $\{\theta_k, k \in [K]\}$ is the set of cluster centers.  Furthermore, $\vw_{i,j}$ denotes an iid zero mean $subG(\sigma^2)$ noise. 

Our goal is to recover estimates of the $K$ cluster centers $\{\hat{\theta}_k\}_{k\in [K]}$ and estimates for the cluster identities for each datapoint, $\{\hat{z}_{i,j}\}_{i\in [m],j\in [n]}$. We measure the quality of a clustering by its misclustering error. On each compute node $i\in [m]$, $A_i \equiv \frac{1}{n}\sum_{j=1}^n\I\{\hat{z}_{i,j} \neq z_{i,j}\}$ determines the local misclustering, while $A \equiv \frac{1}{m}\sum_{i=1}^m A_i$ denotes the global misclustering.

We consider two specific instances of this problem setup to highlight our result -- i) \textbf{symmetric $2$-cluster case}, and ii) the \textbf{$K$-cluster case}. We describe these settings  below.

\textbf{Symmetric $2$-clusters case:}
We use this case as an illustrative example to elucidate our theory. We define the two centers as $\theta_1 = \theta^\star$ and $\theta_2 = -\theta^\star$ and $z_{i,j} \in \{-1,1\}$. The goal of estimating two centers $\theta_1$ and $\theta_2$ boils down to estimating a single vector $\theta^\star$. The Eq. ~\eqref{def:mix_subg} then simplifies to: $    \vx_{i,j} = z_{i,j}\theta^\star + \vw_{i,j}
$. For this part, we assume that the noise $\vw_{i,j}$ is Gaussian instead of being subGaussian, i.e., $\vw_{i,j}\sim \cN(0,\sigma^2 \I_d)$. 
Following the notation of \cite{lu_statistical_2016}, we can define the signal-to-noise-ratio (SNR) as $r \equiv  \frac{\norm{\theta^\star}}{\sigma\sqrt{1 + 9d/mn}}$. We use this particular form of SNR as even when cluster identities are known, simply estimating the mean of a $d$-dimensional Gaussian would result in an error of $\sigma\sqrt{\frac{d}{mn}}$.

\textbf{$K$-cluster case:}
We use the model defined in Eq~\eqref{def:mix_subg} following ~\cite{lu_statistical_2016}, with $\vw_{i,j}$ being zero mean $SubG(\sigma^2)$ random vectors. The cluster centers $\{\theta_k\}_{k\in [K]}$ are arbitrary, so we define the following separation parameters between two clusters: $\Gamma \equiv \min_{k\neq h \in [K]} \norm{\theta_k - \theta_h}$, $\lambda \equiv \max_{k\neq h \in [K]} \frac{1}{\Gamma} \norm{\theta_k - \theta_h}$. We use $S_{k,i}^\star = \{j \in [n] : z_{i,j} = k \}$ to denote the datapoints in the $k^{th}$ cluster on machine $i$  and $S_{k}^\star = \{(i,j) : j \in S_{k,i}^\star, i \in [m]\}$ to denote the datapoints in the $k^{th}$ cluster globally. Further, $\nu_k^\star = \abs{S_k^\star}$ and $\nu_{k,i}^\star = \abs{S_{k,i}^\star}$ denote the sizes of these clusters. Let $\alpha$ denote the size of the smallest cluster and $r_K$ the corresponding definition of SNR in this case.
\begin{align*}
    \alpha \equiv \min_{k\in [K]}\frac{\nu_k^\star}{mn} \leq \frac{1}{K},  \quad \beta \equiv \min_{i\in [m], k\in [K]} \frac{\nu_{k,i}^\star}{\nu_k^\star},\quad r_K \equiv  \frac{\Gamma}{\sigma}\sqrt{\alpha \left(1 + \frac{K d}{mn}\right)^{-1}} \, .
\end{align*}
Note that the SNR $r_K$ depends on the size of the smallest cluster, therefore, unbalanced cluster sizes result in a lower SNR. The term $\beta$ is the ratio of the local and global cluster sizes, and determines the balance of a cluster's data across machines. Note that $\beta \leq \frac{1}{m}$, with equality when each machine has exactly $\frac{\nu_{k}^\star}{m}$ datapoints from each cluster $k\in [K]$. Following the notation in Algorithm~\ref{alg:fedkmeans}, $S_{k,i}^{(t)}$ denotes the datapoints in cluster $k$ on machine $i$ at iteration $t$. Let $S_k^{(t)} = \cup_{i\in [m]} S_{k,i}^{(t)}$. To define the notion of false positives and false negatives for a cluster $k$, we define $S^{(t)}_{kh,i} = \{j\in [n]: z_{i,j} = k, \hat{z}^{(t)}_{i,j} =h\}$  and $S^{(t)}_{hk,i} = \{j\in [n]: z_{i,j} = h, \hat{z}^{(t)}_{i,j} = k\}$ for each machine $i\in [m]$. We use $\nu^{(t)}_k = \abs{S_k^{(t)}}$ and $\nu_{kh}^{(t)} = \abs{S_{kh}^{(t)}}$ to denote the sizes of these sets across all machines. The cluster-wise misclustering can then be defined as,
\begin{align*}
    G^{(t)} = \max_{k\in [K]}\max \left\{\frac{\sum_{k\neq h \in [K]} \nu_{hk}^{(t)}}{\nu_k^{(t)}}, \frac{\sum_{k\neq h \in [K]} \nu_{kh}^{(t)}}{\nu_k^\star}\right\}
\end{align*}
The first term is the false positive rate of cluster $k$ and the second term is the true negative rate of cluster $k$. Note that the misclustering $A^{(t)} \leq G^{(t)}$. Let $G_i^{(t)}$ be the cluster-wise misclustering on the machine $i\in [m]$. In the $2$ cluster case, we did not have to define these terms as the estimate $\hat{\theta}^{(t)}$ can be computed from $A^{(t)}$. This is not possible for the $K$-cluster case. We define the estimate of $\theta_k$ as $    \hat{\theta}^{(t)}_k = \frac{1}{m}\sum_{i\in [m]}\hat{\theta}^{(t)}_{k,i}
$ at any iteration $t$.  Based on this estimate, we can define the error in cluster center estimatess on machine $i\in [m]$ and globally as $\Lambda_i^{(t)} = \max_{k \in [K]} \frac{1}{\Gamma} \norm{\hat{\theta}_{k,i}^{(t)} - \theta_k}$ and $\Lambda^{(t)} = \max_{k \in [K]} \frac{1}{\Gamma} \norm{\hat{\theta}_k^{(t)} - \theta_k}$ respectively. For both $G^{(t)}$ and $\Lambda^{(t)}$, we use the notion of worst-case error over all clusters.

\textbf{Baselines:} We consider two standard baselines for our algorithm -- \textbf{centralized Lloyd's} and \textbf{Lloyd's without aggregation} (we also call the later {\em local algorithm}). In centralized Lloyd's, we set the number of local steps $L=1$, thereby communicating every round. This is equivalent to performing Lloyd's steps on all the data at the same time, as if we were in the centralized setting. In Lloyd's without aggregation, each machine performs Lloyd's algorithm locally on its data and does not perform aggregation.

\vspace{-2mm}
\subsection{Main Contributions}
\vspace{-1mm}
\label{sec:contributions}
In this section, we ignore constants in some of our theoretical results for clarity.

\textbf{Algorithm:} We propose LocalKMeans (Algorithm~\ref{alg:fedkmeans}), a distributed version of Lloyd's algorithm, which performs $L$ local Lloyd's update steps on each compute node before aggregating the estimates of the local cluster centers from all compute nodes weighted by their local cluster sizes. To initialize cluster centers for Lloyd's algorithm, we propose LocalKMeans++(Algorithm~\ref{alg:localkmeans++}), a distributed version of the KMeans++~\cite{arthur_k-means_2007}, which is equivalent to performing KMeans++ on all the data centrally.
    
\textbf{Symmteric 2-clusters:} 
    If the initial misclustering global misclustering $A^{(0)}$  satisfies Assumption~\ref{assumption:local_dev},  SNR $ r
\geq c \max\{\sqrt{\nicefrac{d}{n}}, 1\}$, $n > c' \log m$ for some sufficiently large constants $c,c'$, the final global misclustering $A^{(t)}$  of LocalKMeans converges to $\exp(-r^2) + \max\{\nicefrac{d}{n},1\}r^{-2}+ \exp(-n)$ after $t = \Omega(\log(mn) + L)$ iterations (Theorem~\ref{thm:2_cluster}). 

\textbf{$K$-clusters:} If the initial global misclustering $G^{(0)}$  satisfies Assumption~\ref{assumption:local_dev_k}, the SNR $r_K \geq  \sqrt{\nicefrac{K}{\alpha}}\max\{1 , \sqrt{\nicefrac{d}{n}}\}$, $\alpha \geq \sqrt{K\log(mn)/mn}$, $\beta  = \Omega(m^{-1})$ and $\lambda = \mathcal{O}(1)$, we obtain a final global misclustering error $A^{(t)} = \exp(-n) + \max\{\frac{d}{n},1\}K (r_K^2 \alpha)^{-1} + \exp(-r_K^2)$  after $t = \Omega(\log(mn)+ L)$ iterations (Theorem~\ref{thm:k_cluster}). 

\textbf{Proof Techniques:} To establish convergence of global misclustering $A^{(t)}$ at every local step $t$ for $L>1$, we apply the virtual iterate method used for analysis of local steps in supervised learning~\cite{karimireddy_scaffold_2020, stich2018local,pmlr-v119-koloskova20a} to the analysis of Lloyd's algorithm for mixture of Gaussians~\cite{lu_statistical_2016}. This extension is not straight-forward, as virtual iterate in supervised learning requires each local (on each machine) objective function to be smooth, and at least convex, if we want convergence in the objective function. For KMeans, the local misclustering is a sum of  indicator functions which is neither smooth nor convex, and requires good local initialization to show progress in a single Lloyd's step. Additionally, Lloyd's algorithm has no tunable parameters like step-size in gradient descent for supervised learning to control the error of local steps. Addressing all these issues requires a novel modification of the virtual iterate method. We carefully analyze iterations where aggregation occurs and does not occur, and obtain tight bounds on deviation between local and global quantities (models $\hat{\theta}_i - \hat{\theta}$ and misclustering $A_i - A$), to ensure a linear convergence in global misclustering at every local step.

 \textbf{Theoretical Comparison to Baselines (Sections~\ref{sec:baselines_2} and \ref{sec:baselines_k}):}
   Although LocalKMeans requires same initialization as centralized Lloyd's, it has a worse final error that scales as $\frac{d}{nr^2} + \exp(-n)$ instead of $\exp(-r^2)$, and requires $\mathcal{O}(\log(mn) + L)$ iterations to reach it instead of $\mathcal{O}(\log (mn))$. LocalKMeans comunicates only once every $L$ iterations; so it communicates $\mathcal{O}(\frac{\log(mn)}{L} + 1)$ bits per machine in total, while centralized Lloyd's ($L=1$) communicates $\mathcal{O}(\log(mn))$ bits per machine, which is about $L$ times more.

 Lloyd's without aggregation does not communicate with the server and leads to worse theoretical guarantees as there exist problem settings where LocalKMeans with finite $L$ converges theoretically but Lloyd's without aggregation doesn't. The most important such setting is when the global  error at initialization is small but some local machine has large local error at initialization. For $2$-cluster case, this corresponds to $A^{(0)}<\frac{1}{2}$, but some $i\in [m], A_i^{(0)}>\frac{1}{2}$. Similar conditions are obtained for $K$-cluster case where the initialization depends on $G$ and $\Lambda$. Another setting is obtained for $m\in (\rm poly(n), 2^n]$ and $r\in (\sqrt{\frac{d}{n}}, \sqrt{m \log m})$ with $1\leq\frac{d}{mn}\leq C$ for some constant $C$.

\textbf{Experiments:} We conducted extensive experiments on both synthetic and real-world datasets and compare with baselines of centralized Lloyd's and Lloyd's without aggregation. The synthetic dataset is a mixture of Gaussians which justifies our theoretical findings that better SNR improves the performance of LocalKMeans. The misclustering error of LocalKMeans is much better than Lloyd's without aggregation and closer to centralized Lloyd's. However, it requires more iterations than centralized Lloyd's, and at least a few aggregations to converge to this error. In terms of real datasets, we use Motion Capture Hand Posture dataset~\cite{motion_capture_hand_postures_405}, and feature embeddings of MNIST \cite{mnist} and CIFAR10 \cite{cifar10} datasets. On these datasets, LocalKMeans outperforms Lloyd's without aggregation by obtaining both lower misclustering ratio and lower KMeans objective (Fig~\ref{fig:real_dataset_ratio}, \ref{fig:real_dataset_objective}), further validating our theory.

\subsection{Related Works}
\begin{algorithm}[t!]
    \caption{LocalKMeans \label{alg:fedkmeans}}
    \begin{algorithmic}
        \REQUIRE Number of clusters $K$, Initial cluster estimates $\Theta = \{\hat{\theta}_k^{(0)}
        \}_{k=1}^{K}$, number of local steps $L$.
        \ENSURE Final cluster centers $\{\hat{\theta}_{k}^{(T)}\}_{k=1}^K$ and cluster identities $\{\hat{z}_{i,j}^{(T)}\}_{i\in [m], j\in [n]}$.
        \FOR{$t = 0$ to $T-1$}
        \STATE \underline{\texttt{Server()}}
        \IF{$L$ divides $t$}
            \STATE Receive local cluster centers and teir sizes  $\{(\abs{S_{k,i}^{(t)}}, \hat{\theta}_{k,i}^{(t)})\}_{i\in [m]}$ from machines $\forall k\in [K]$
            \STATE $\hat{\theta}_k^{(t)} \gets  (\sum_{i\in [m]}\abs{S_{k,i}^{(t)}}\hat{\theta}_{k,i}^{(t)})/(\sum_{i\in [m]} \abs{S_{k,i}^{(t)}}),\quad \forall k \in [K]$ 
            \STATE Broadcast averaged cluster centers $\hat{\theta}_k^{(t)},\,\forall k\in [K]$ to all machines 
        \ENDIF
        \STATE \underline{\texttt{Machine($i$)}}
            \IF{$L$ divides $t$}
                \STATE Send $(\abs{S_{k,i}^{(t)}}, \hat{\theta}_{k,i}^{(t)})$ to the Server for all clusters $k\in [K]$.
                \STATE Receive $\hat{\theta}_k^{(t)}$ from Server for all clusters $k\in [K]$
            \ENDIF
            \STATE $\hat{z}_{i,j}^{(t+1)} \gets \underset{k \in [K] }{\arg\min}\norm{\vx_{i,j} - \hat{\theta}_{k,i}^{(t)}}^2$ for all local datapoints $j\in [n]$. 
            \STATE $S_{k,i}^{(t+1)} \gets \{j \in [n]: \hat{z}_{i,j}^{(t+1)} = k \},\quad \forall k \in [K]$ 
            \STATE $\hat{\theta}_{k,i}^{(t+1)} \gets (\sum_{j\in S_{k,i}^{(t+1)}}\vx_{i,j})/\abs{S_{k,i}^{(t+1)}}, \quad \forall k \in [K]$ 
        \ENDFOR
    \end{algorithmic}
    
\end{algorithm}

\textbf{Clustering with KMeans:} KMeans clustering~\cite{lloyd} is a well studied problem. A prominent line of work establishes the number of iterations required for Lloyd's algorithm~\cite{vattani,arthur_how_2006} without any distributional assumptions. Another line of work focuses on finding appropriate initialization for KMeans~\cite{arthur_k-means_2007,ostrovsky_effectiveness_2013}. Under certain deterministic clusterability conditions on the data, Lloyd's algorithm can recover the clusters in the data~\cite{ostrovsky_effectiveness_2013, kumar_clustering_2010} with appropriate initialization. These guarantees have been improved for the specific case of mixture of Gaussian distributions~\cite{chaudhuri_learning_2009, awasthi_improved_2012, lu_statistical_2016}. Specifically, ~\cite{chaudhuri_learning_2009} were the first to investigate performance of KMeans for mixture of Gaussians. ~\cite{awasthi_improved_2012} improved this result by using a spectral initialization before the Lloyd's steps. To the best of our knowledge, ~\cite{lu_statistical_2016} obtain the best dependence on SNR, defined as the ratio of cluster separation and the noise variance, and establish a linear convergence rate for the misclustering error. Note that the KMeans objective is an appropriate measure of performance when finding initialization for KMeans~\cite{arthur_k-means_2007}. However, for the mixture of Gaussians problem (Eq~\eqref{def:mix_subg}), the KMeans objective can be bounded in terms of the misclustering, which we analyze. Additionally, only ~\cite{awasthi_improved_2012} has established improvement due to Lloyd's algorithm in terms of the KMeans objective, but their result is weaker than ~\cite{lu_statistical_2016} for mixture of subGaussians. Therefore, we use the analysis of ~\cite{lu_statistical_2016} with misclustering as the error metric.

\textbf{Unsupervised Distributed Learning:}

Unsupervised learning in the distributed setup has been studied empirically~\cite{fed_kmeans_1,fed_cluster_2,zhang2023federated}. Federated versions of the EM algorithm~\cite{dieuleveut2021federatedem}, specifically focusing on multi-task learning in mixture of distributions~\cite{marfoq_federated_2021,tian_towards_2024}. Distributed clustering was also considered in \cite{pan2022machine,li2022secure}.
However, none of these works incorporate the concept of local steps, rather assuming aggregation at each step. ~\cite{pmlr-v139-dennis21a} develop a federated clustering algorithm with theoretical guarantees by extending centralized clustering guarantees of ~\cite{awasthi_improved_2012} with a single aggregation step at the end. However, their approach differs from the simple Lloyd's algorithm. 

\textbf{Organization.}
We present LocalKMeans in Algorithm~\ref{alg:fedkmeans}. In Section~\ref{sec:2_cluster}, we provide theoretical analysis for the case of $2$ symmetric clusters and explain our proof techniques. In Section~\ref{sec:k_cluster_summarized}, we analyze LocalKMeans for $K$ clusters. In Section~\ref{sec:exp}, we provide experimental results for LocalKMeans on real and synthetic datasets.

\section{LocalKMeans on $2$ symmetric clusters}
\label{sec:2_cluster}

In this section, we establish the theoretical  performance of LocalKmeans for the symmetric $2$-cluster case. As this case is simpler to understand, we will also present the key ideas behind our proofs. First, we state the initialization condition required on global misclustering for LocalKMeans to converge.

\begin{assumption}[Initialization]\label{assumption:local_dev}
$A^{(0)} \leq \,\frac{1}{2} - \frac{1}{\sqrt{mn}} - \frac{2.56 + \sqrt{\log r}}{r}$.
\end{assumption}

The initialization condition is same as centralized Lloyd's from ~\cite[Theorem~3.1]{lu_statistical_2016} for $mn$ datapoints. This initialization is slightly better than random initialization. Note that a good initialization is often used to show convergence of alternating minimization methods like Lloyd's algorithm~\cite{awasthi_improved_2012,kumar_clustering_2010,das_ten} for non-convex misclustering objective. Our initialization is much weaker than Lloyd's without aggregation which requires the same initialization for each machine with $mn$ replaced by $n$. In our case, some machine $i\in [m]$ can have $A_i^{(0)} = 1$, as long as $A^{(0)}$ satisfies initialization.

To analyze LocalKMeans, we can naively bound the progress in local misclustering $A_i^{(t)}$ at every local step using ~\cite{lu_statistical_2016}, and then unroll it for $L$ local steps to bound the progress in global misclustering $A$. First, note that this would not be possible without a local initialization condition, which we do not assume. Second, this would lead to a slowdown of $\nicefrac{1}{L}$ in the convergence rate. With this initialization,  centralized Lloyd's would have achieved  linear decrease in $A^{(t)}$  at every local step instead of every $L$ steps.

\textbf{Virtual Iterate Method~\cite{stich2018local}.} To ensure progress at every local step, we borrow the virtual iterate method from FedAvg~\cite{karimireddy_scaffold_2020, stich2018local}. We compute the global estimate for $\theta^\star$ for all iterations as $\hat{\theta}^{(t)} = \frac{1}{m}\sum_{i=1}^m \hat{\theta}_i^{(t)}$. When $t$ is divisible by $L$, this is exactly equal to our aggregated estimate, however, for $t$ not divisible by $L$, this estimate is not actually computed. By performing a single local Lloyd's step at every machine, each $\hat{\theta}_i^{(t)}$ is updated, which in turn updates $\hat{\theta}^{(t)}$. By analyzing progress in terms of $\hat{\theta}^{(t)}$, we should obtain some decrease in every local step.  However, as we do not aggregate until $L$ steps, this progress should be different from centralized Lloyd's due to the differences between $\hat{\theta}_i^{(t)}$ and $\hat{\theta}^{(t)}$. The following Lemma provides the progress for a single step of LocalKMeans. 

\begin{lemma}[Single-step progress]\label{lem:single_step}
Suppose Assumption~\ref{assumption:local_dev} holds, and $r > c\cdot \max\{\sqrt{\nicefrac{d}{n}},1\}$, $n > c' \log m$ for some constants $c,c'>0$. 
Let, $\Phi_1 \leq c'' (\sqrt{\frac{\log(mn)}{mn}} + \frac{1}{r^2})$ and $\Phi_2 \leq c''(\frac{d}{nr^2} + \frac{1}{mn})$, for some constant $c''>0$.
Then  $\forall t>1$ in  Algorithm~\ref{alg:fedkmeans}, with probability $1 - (mn)^{-3} - \exp(- r^2)  - \exp(- n)$, we have $  A^{(t)} < \frac{1}{4}$, and there exists a constant $\delta_1\in (0, \frac{1}{2}-\epsilon), \text{  such that  } A_i^{(t)} \leq A^{(t)} + \frac{\delta_1}{2} - \frac{\sqrt{\log r}}{2r},\quad \forall i \in [m]$ where $\epsilon>0$ is a     small constant. Further, $\forall t\geq1$, with the same probability, we have,  
  $A^{(t+1)}\leq \frac{1}{2}A^{(t)} + \Phi_1$, if $L$ divides $t$;   and $A^{(t+1)}\leq \frac{11}{20}A^{(t)} + \frac{1}{5}A^{(t-1)} + \Phi_1 + \Phi_2$ otherwise. 
 \end{lemma}
When $L$ divides $t$, then it is an aggregation step. Therefore, the progress then is the same as that for centralized Lloyd's. When there is no aggregation, we establish two conditions that ensure each machine is not very bad-- i) the local misclustering $A_i^{(t)}$ is not much larger than global misclustering $A^{(t)}$, ii) Global misclustering is always $\leq \frac{1}{4}$. Under these conditions, we can analyze the progress when $L$ doesn't divide $t$, obtaining additional terms of misclustering at the previous step, $A^{(t-1)}$, and a local term depending on $n$, $\Phi_2$. Fortunately, this bound can still obtain the same linear decrease at every step, and the term $\Phi_2$ depends on $\frac{1}{r^2}$, so it is small when $r$ is large. Further, local steps also reduce the probability of error from $\exp(-mn)$ to $\exp(-n)$, additionally requiring $n > \log m$.

\paragraph{Proof Sketch} At any local step $t+1$, the data point $\vx_{i,j}$ is misclustered, if $\hat{z}_{i,j}^{(t+1)}=-z_{i,j}$. This happens if $\norm{z_{i,j}\vx_{i,j} - \hat{\theta}_i^{(t)}}^2 \geq\norm{-z_{i,j}\vx_{i,j} - \hat{\theta}_i^{(t)}}^2\implies \ip{\theta^\star + \vw'_{i,j}}{\hat{\theta}_i^{(t)}}\leq 0$, where $\vw'_{i,j}=z_{i,j}\vw_{i,j}$. Note that the datapoint $\vx_{i,j}$ uses the local model $\hat{\theta}_i^{(t)}$ for local steps. However, when $L$ divides $t$, $\hat{\theta}_i^{(t)}$ is replaced by the aggregated global model $\hat{\theta}^{(t)}$, and we obtain the same progress as single-step of centralized Lloyd's as long as $A^{(t)} \leq A^{(0)}$. Note that this is automatically satisfied at $t=0$, which implies $A^{(1)} \leq \frac{1}{4}$ for large $mn$ and $r$. When $L$ doesn't divide $t$, we need to account for the difference $\hat{\theta}_i^{(t)} - \hat{\theta}^{(t)}$, where $\hat{\theta}^{(t)}$ is the virtual iterate. Since the misclustering $A^{(t+1)} = \frac{1}{mn}\sum_{i\in [m]}\sum_{j\in [n]} \I\{\ip{\vw'_{i,j} + \theta^\star}{\hat{\theta}_i^{(t)}}\leq 0\}$, where $\I$ is the indicator function, we can obtain the following bound (Eq~\eqref{eq:2_cent_decomp}).
{\small
\begin{align*}
A^{(t+1)} \leq \frac{1}{mn}\sum_{i,j}\I\{\beta_1\norm{\theta^\star}^2 \leq -\ip{\vw'_{i,j}}{\theta^\star}\} + \frac{A^{(t)}}{4} + \Phi_1 + \frac{1}{mn}\sum_{i,j}\I\{\beta_2 \norm{\theta^\star}^2 \leq -\ip{\vw'_{i,j}}{\theta^\star}\}  + \Phi_2  + \frac{\Delta^{(t)}}{30}
\end{align*}
}
The first $3$ terms are similar to those for centralized Lloyd's, with a different $\beta_1$ and the last $3$ terms are the price paid for local steps. We need $\beta_1, \beta_2 >0$, for the first and fourth term to be $<\frac{1}{2}$. These are only possible if $A^{(t)} < \frac{1}{2} - \frac{\delta_1}{2} - \delta_2$ and $A_i^{(t)} - A^{(t)} \leq \frac{\delta_1}{2}-\frac{\sqrt{\log r}}{r},\forall i\in [m]$ for a constant $\delta_2>0$. This is possible if $A^{(t)}\leq \frac{1}{4}$, but note that it is not satisfied for $t=0$. Therefore, using this analysis for all steps would require a worse global initialization. The fifth term, is an upper bound on $\frac{1}{m}\sum_{i=1}^m \norm{\bar{\vw'}_i}^2$, where $\bar{\vw'}_i = \frac{1}{n}\sum_{j\in [n]} \vw'_{i,j}$. The expectation of this term is $\leq\frac{d}{nr^2}$, which provides $\Phi_2$ dependent on $n$ and a dependence of $\exp(-n)$ in the probability of error. The last term $\Delta^{(t)} = \frac{1}{m}\sum_{i=1}^m \norm{\hat{\theta}_i^{(t)} - \hat{\theta}^{(t)}}^2$ is the deviation term commonly used in the analysis of virtual iterate~\cite{karimireddy_scaffold_2020}. Existing analyses~\cite{karimireddy_scaffold_2020} of $\Delta^{(t)}$ unroll it to the last aggregation step, $\tau = \floor{t/L}\cdot L$, and use $\Delta^{(\tau)} = 0$. However, this forces $\Delta^{(t)}$ to grow exponentially with $L$. Existing analysis~\cite{karimireddy_scaffold_2020} set step size in gradient descent proportional to $\frac{1}{L}$ to handle this. As we do not have any tunable parameters in our algorithm, we use the fact that $A^{(t)} < \frac{1}{4}$, and a weaker bound on $\Delta^{(t)}$ in terms of $A^{(t)}$ and $A^{(t-1)}$ by unrolling it for only $1$ step. Further, we still need to prove the condition $A_i^{(t)} - A^{(t)}\leq \frac{\delta_1}{2} - \frac{\sqrt{\log r}}{2r}$, which uses an inductive argument by analyzing each local step on each machine, requiring large $n$.

\begin{remark}[Special case of centralized Lloyd's]
Note that if we set $L=1$,  then $t$ is always divisible by $L$, so Lemma~\ref{lem:single_step} recovers the correct single-step progress for centralized Lloyd's.
\end{remark}
The above remark shows that our analysis is tight for the case of $L=1$. However, note that the subsequent analysis in this section assumes that $L>1$.

\paragraph{Unrolling the recursion.}
Unrolling Lemma~\ref{lem:single_step} is not straightforward, as we need to account for the term $A^{(t-1)}$. However, we can still show that the misclustering decreases linearly at every step from initialization, with the additional terms $\Phi_1 + \Phi_2$ added at every iteration. Therefore, after a certain number of iterations, we can achieve a misclustering of $\Phi_1 + \Phi_2$. The following Lemma exactly characterizes this.
\begin{remark}
    Under the conditions of Lemma~\ref{lem:single_step}, after $t\geq 2(\ceil{\log(mn)} + L)$ steps, $A^{(t)} \leq 4 \Phi_1 + 4 \Phi_2$ 
\end{remark}
As $t\geq 2L$, we need atleast $2$ aggregations to converge to misclustering of $\Phi_1+\Phi_2$. Therefore, as $L$ increases, we need more number of iterations to converge to required error. Following the analysis of ~\cite{lu_statistical_2016}, the final error can be improved by a tighter analysis for $t\geq 2(\ceil{\log(mn)} + L)$ that improves terms of $\Phi_1$ to $\exp(-r^2)$. We perform a similar analysis for LocalKMeans, accounting for the additional terms due to local steps that we incurred in Lemma~\ref{lem:single_step}.  The following Theorem obtains the tightest bound on the final misclustering error.

\begin{theorem}[Final Error]\label{thm:2_cluster}
    Under the conditions of Lemma~\ref{lem:single_step}, after $t$ steps of Algorithm~\ref{alg:fedkmeans}, for $L>1$, where $t \geq (\ceil{q/\log(2)}(\ceil{\log(mn)} + L)+t_0$,   for some $q>0$ and $t_0 = 2\ceil{\log(mn)} + 2L$, with probability $1 - \zeta$, we have
\begin{align*}
 A^{(t)} = \mathcal{O}\left(\frac{1}{\zeta}\left(\exp(-n) + \max\left\{\frac{d}{n},1\right\}\frac{1}{r^2} + \frac{e^{-qL}}{(mn)^q} + \exp(-r^2)\right)\right).
 \end{align*}
\end{theorem}
After running $2$ rounds of LocalKMeans, any additional steps, which decide the value of $q$, can decrease the final error. For a large constant $q$, we can make the term dependent on $mn$ arbitrarily small. There are also additional terms of $\exp(-n)$ and $\max\{\nicefrac{d}{n},1\}r^{-2}$, which are not present for the centralized case. Therefore, local steps requires us to run more iterations, at least $2L$, and requires larger SNR $r$ to converge to the same final misclustering. We provide a proof for all theoretical results in this section in Appendix~\ref{sec:2_cluster_proof}.

\subsection{Comparisons to Baselines}
\label{sec:baselines_2}
Theoretical analysis of centralized Lloyd's is obtained by using $mn$ data points in ~\cite[Theorem~3.1]{lu_statistical_2016}. This yields the same initialization as Assumption~\ref{assumption:local_dev}. Further, it obtains $\exp(-r^2)$ final misclustering, in only $\mathcal{O}(\log(mn))$ iterations, with probability of error depending on $\exp(-r^2)$ and $mn$. Our results obtain higher final misclustering, $r^{-2}$, take more iterations $\mathcal{O}(\log(mn) + L)$, and have a probability of error depending on $\exp(-n)$. However, LocalKMeans requires $L$ times less communication than Centralized Lloyd's. To compare with Lloyd's without aggregation, we set the number of datapoints to $n$ in ~\cite[Theorem~3.1]{lu_statistical_2016}, and take a union bound over all machines $i\in [m]$. While this method does not communicate at all, it's theoretical guarantees are much worse than LocalKMeans. It's initialization requirement needs to hold for each machine, it's probability of error varies as $mn^{-3}$  and $m\exp(-r_n^2)$, where $r_n = \frac{\norm{\theta^\star}}{\sigma}(1 + \frac{9d}{n})^{-1} = \frac{\sqrt{mn+9d}}{\sqrt{n+9d}\sqrt{m}}r$ ( $r_n$ is the SNR on $n$ datapoints, and $r$ is the SNR on $mn$ datapoints). This forces $m = \max\{\text{poly}(n),\exp(r_n^2)\}$. Therefore, Lloyd's without aggregation cannot work with large $m$, small $r$ and bad  initialization in any machine; however, LocalKMeans can.

\section{LocalKMeans on K clusters}
\label{sec:k_cluster_summarized}

For the $K$-cluster case, theoretical analysis is qualitatively similar to $2$-cluster case. Due to lack of space, we only state the initialization requirement, the single-step progress and the final error and defer further discussion and comparison against baselines to Appendix~\ref{sec:k_cluster}, and the proof to Appendix~\ref{sec:k_cluster_proof}.  
\begin{assumption}[Initialization]\label{assumption:local_dev_k} $G^{(0)}  < \left(\frac{1}{2}  - \frac{6}{\sqrt{r_K}}\right)\frac{1}{\lambda},\quad \text{ or } \quad \Lambda^{(0)} \leq \frac{1}{2} - \frac{4}{\sqrt{r_K}}$.
\end{assumption}
\begin{lemma}[Single-step progress]\label{lem:single_step_k}
    Suppose Assumption~\ref{assumption:local_dev_k} holds, $\lambda  \leq c_1,\alpha  \geq c_2\sqrt{\frac{K\log(mn)}{mn}}), \beta \geq c_3 m^{-1}, n \geq c_4 \log m$ and $r_K \geq c_5\sqrt{\frac{K}{\alpha}}\max\{ \sqrt{\frac{d}{n}}, 1\})$ for some constants $c_1>1$ and $c_2,\ldots,c_5>0$. Let $\Psi_1 \leq c_6 (\frac{1}{r_K^2}+ \sqrt{\frac{K\log(mn)}{\alpha^2 mn}}), \Psi_2\leq c_7\frac{K\max\{\frac{d}{n},1\}}{\alpha r_K^2}$ for some constants $c_6, c_7>0$. Then, at time step $t$ in Algorithm~\ref{alg:fedkmeans}, with probability $1  - (mn)^{-3} - \exp(- n)-\exp(-r_K^2)$,  we have, $\forall t\geq 1, G^{(t)} \leq 0.18, \Lambda^{(1)}\leq 0.2 $ and there exits a constant  $\delta_1 \in (0, 0.1-\epsilon')$ such that  $\Lambda_i^{(t)} \leq  \Lambda^{(t)} + \frac{\delta_1}{2} ,\, \forall i \in [m]$ for some constant $\epsilon'>0$. Further, for some constants $c_8, c_9>0$, $\forall t\geq 1,  \text{ if $L$ divides $t$,}\quad
         G^{(t+1)} \leq \frac{c_8}{r_K^2} G^{(t)}  + \Psi_1,  \quad \text{otherwise,}\quad
        G^{(t+1)} \leq  \frac{c_9}{r_K^2} G^{(t)} + \frac{c_9}{r_K^2} G^{(t-1)} + \Psi_1  + \Psi_2.$
 \end{lemma}

 \begin{remark}[Progress in $\Lambda^{(t)}$]
     We can obtain a recursion for $\Lambda^{(t+1)}$ similar to the one we have for $G^{(t+1)}$.
 \end{remark}

\begin{theorem}[Final Error]\label{thm:k_cluster}
If the conditions in Lemma~\ref{lem:single_step_k} hold, after running Algorithm~\ref{alg:fedkmeans} for $t$ steps, where $t\geq q(\lceil \log(mn) \rceil + L) + t_0$, for some $q>0$ and $t_0 = 2\ceil{\log(mn)} + 2L$, with probability $1-\xi$, we obtain, 
\begin{align*}
A^{(t)} = \mathcal{O}\left(\frac{1}{\xi}\left(\exp(-n) + \max\left\{\frac{d}{n},1\right\}\frac{K}{\alpha r_K^2 } + \frac{e^{-qL}}{(mn)^q}+\exp(-r_K^2)\right) \right) 
\end{align*}
\end{theorem}
Note that this initialization is exactly the same as required for running centralized Lloyd's and is much weaker than the per machine initialization required for Lloyd's without aggregation. Again, we require at least $2\log(mn) + 2L$ iterations, i.e., $\geq2$ rounds, to get a small final error, which now scales as $K/r_K^2$ instead of $\exp(-r_K^2)$. An additional condition is that $\beta  = \Omega(m^{-1})$, which makes each cluster to be balanced across machines, and $\lambda = \mathcal{O}(1)$, which balances the distance between true cluster centers.

\section{Experiments}
\label{sec:exp}

We ran the LocalKMeans algorithm on synthetic Gaussian data to validate our theory and on real datasets to show it's practical effectiveness. In all our experiments, we compare against the baselines, centralized Lloyd's ($L=1$) and Lloyd's without aggregation. These are compared to LocalKMeans with small $L$ ($L=2,3$) and large $L$ ($L=\frac{T}{2}$, where $T$ is the total number of iterations). Note that in the large $L$ case we only perform $2$ aggregations. 
 
\subsection{Synthetic Data}
\label{sec:synthetic data}

The synthetic data is generated by a mixture of Gaussians, which exactly matches the statistical model (\ref{def:mix_subg}). In this experiment, the ground truth centers $\{\theta_i\}_{i=1}^K$ are orthonormal vectors with dimension $d$. The added noise $\vw_{i,j}$ follows Gaussian distribution whose elements satisfy $\mathcal{N}(0, \sigma^2)$. The SNR $r$ is defined as $r \triangleq  \frac{\norm{\theta^\star}}{\sigma\sqrt{1 + 9d/mn}}$. We have two different methods to initialize the LocalKMeans algorithm:

\textbf{LocalKMeans++ Initialization:} We use the LocalKMeans++ algorithm as the initialization method  (Algorithm \ref{alg:localkmeans++}). It is a distributed implementation of KMeans++ initialization \cite{arthur_k-means_2007}.

\begin{figure*}[htbp]
    \centering
    \subfigure[LocalKMeans++ Initialization, $r=3.01$.]
    {
    \includegraphics[width=0.235\textwidth]{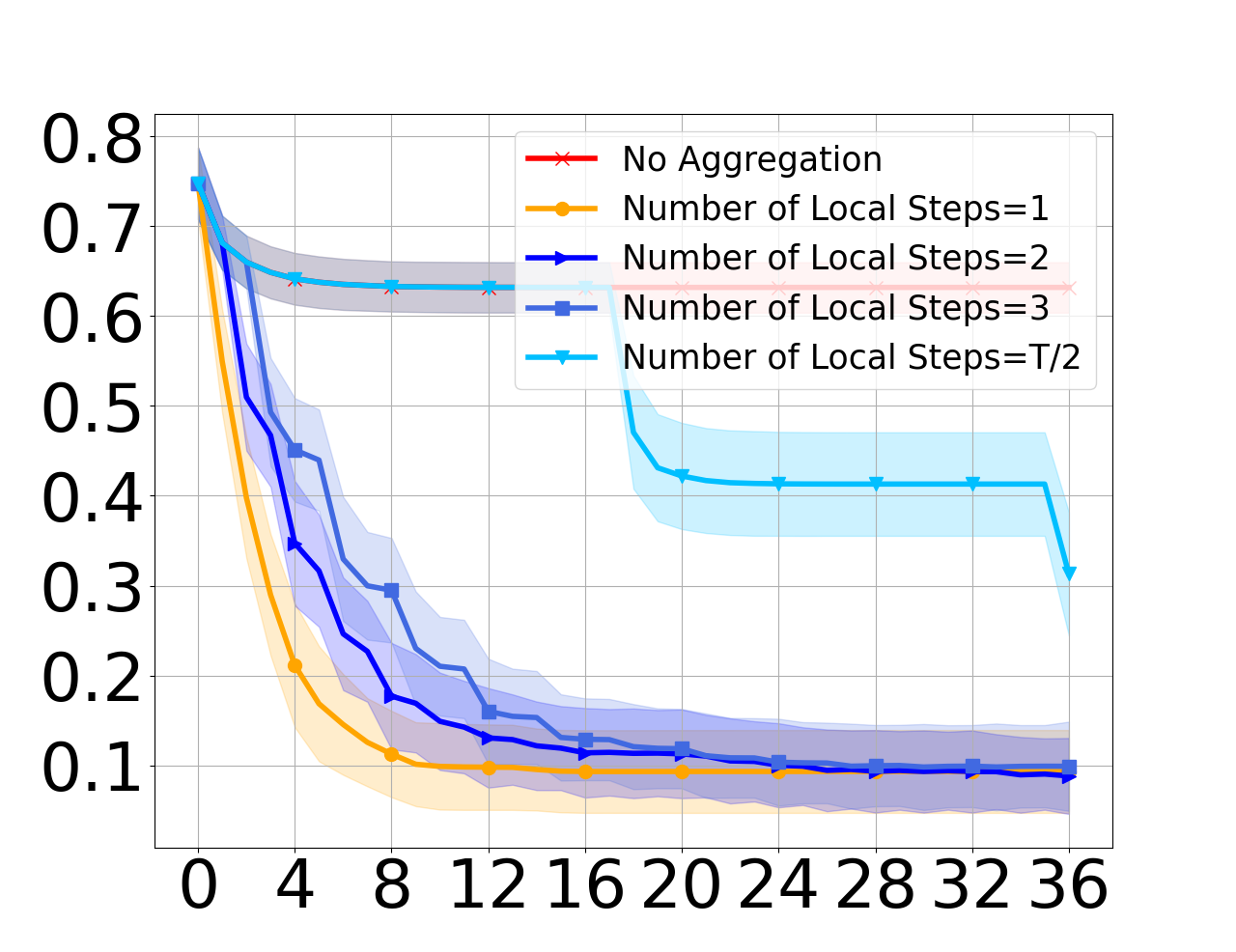}
    }
    \subfigure[LocalKMeans++ Initialization, $r=6.02$.]
    {
    \includegraphics[width=0.235\textwidth]{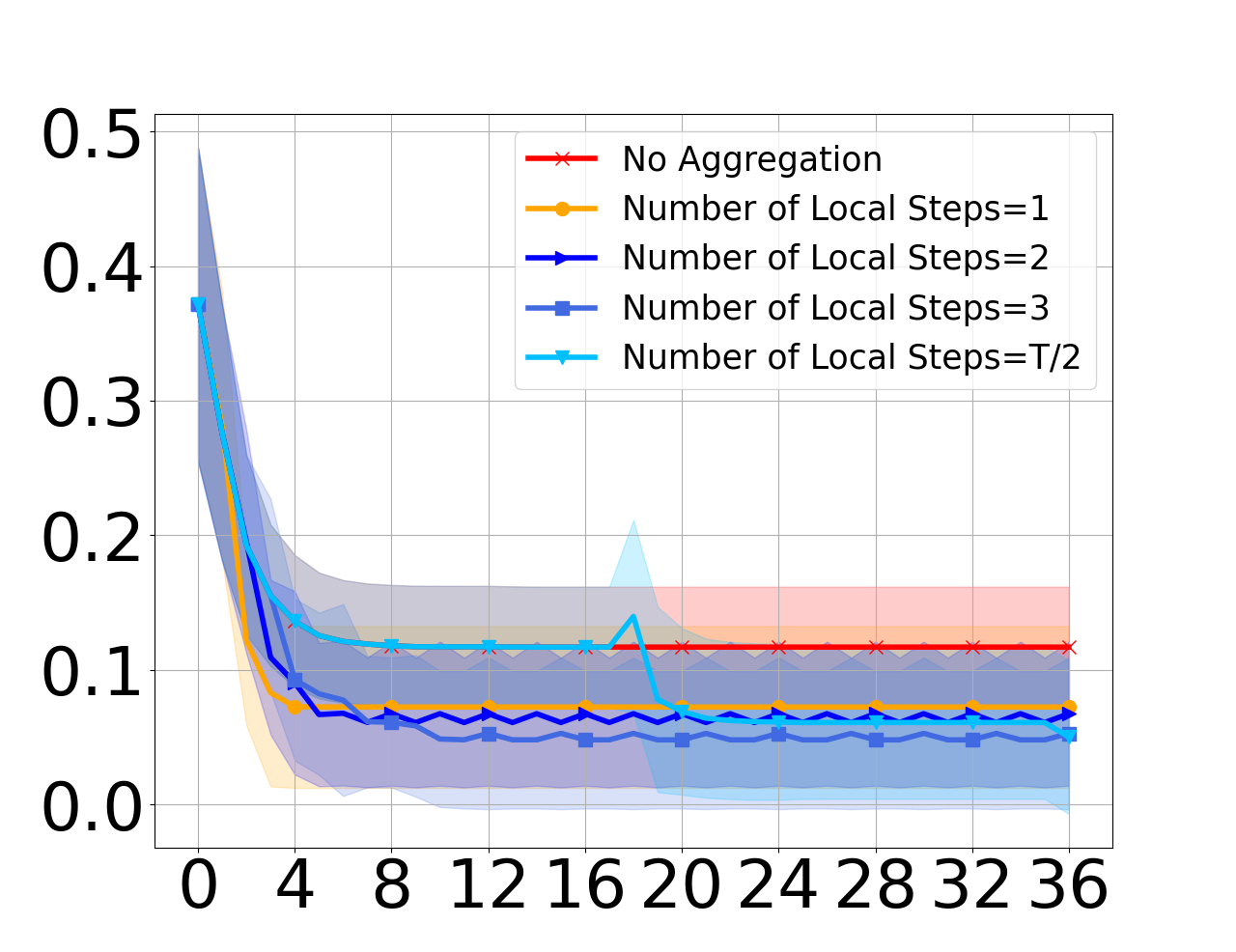}
    }
    \subfigure[Perturbed Initialization, $r=3.01$.]
    {
    \includegraphics[width=0.235\textwidth]{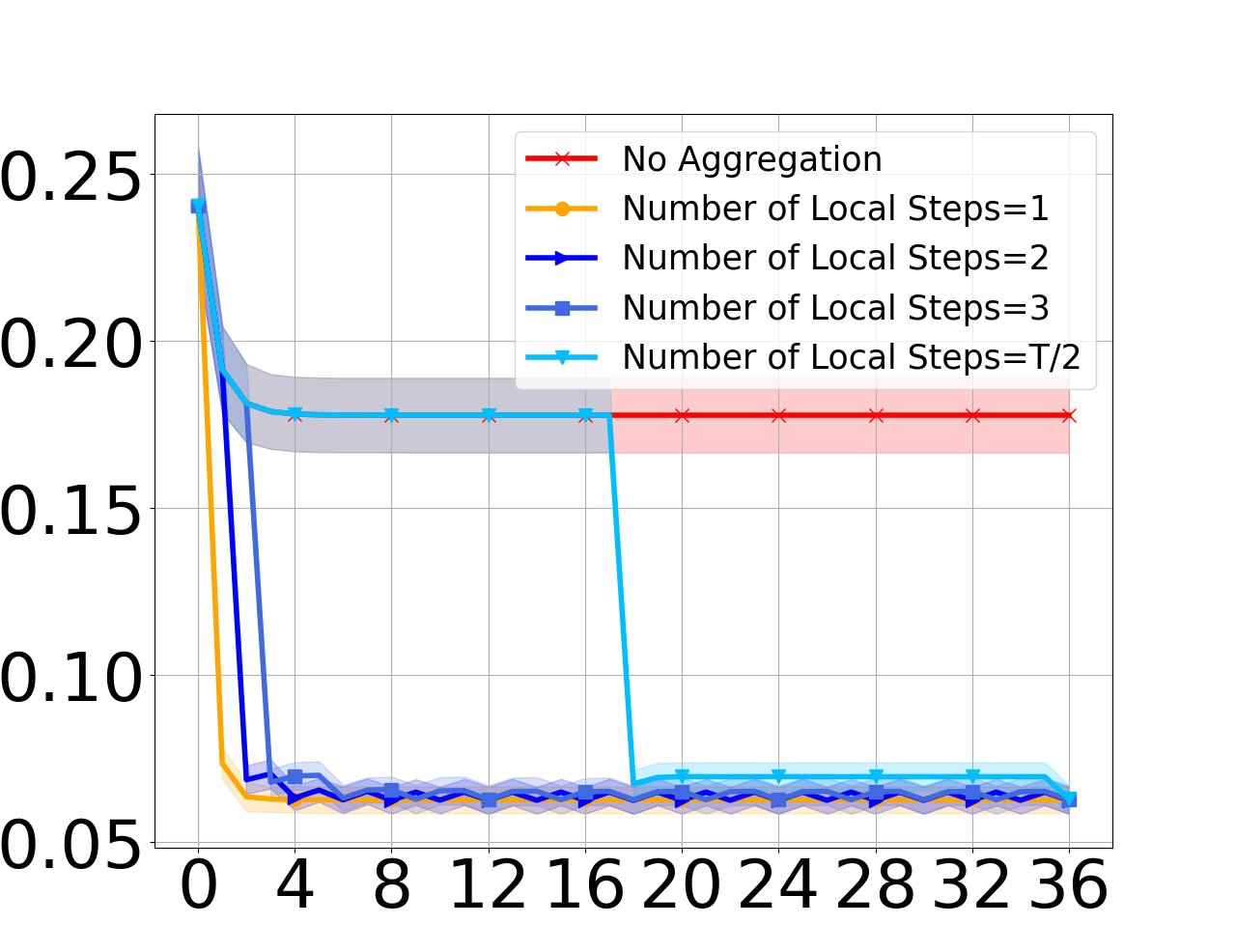}
    }
    \subfigure[Perturbed Initialization, $r=6.02$.]
    {
    \includegraphics[width=0.235\textwidth]{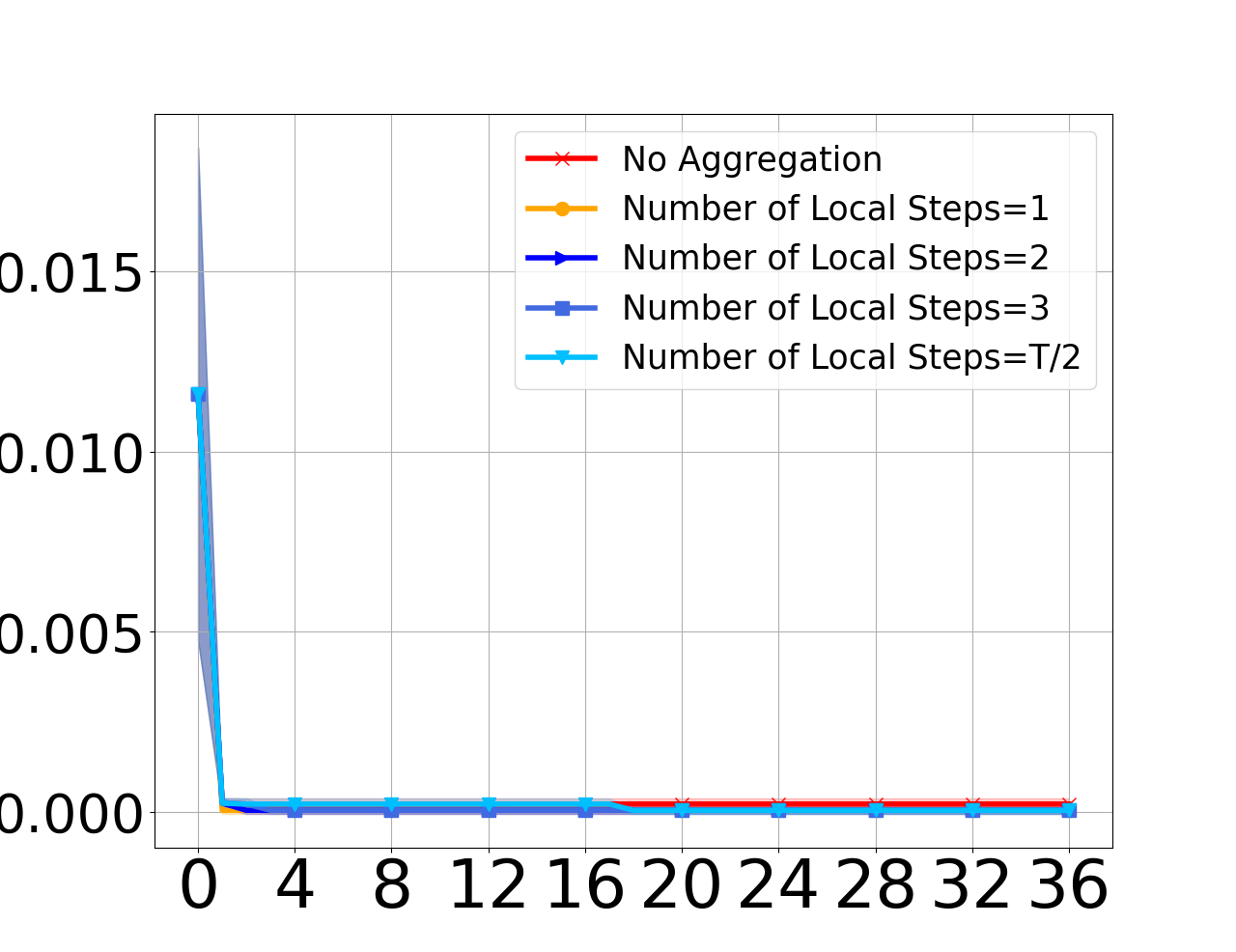}
    }
    \caption{Performance of LocalKMeans on mixture of Gaussians with different initializations and SNRs. Horizontal Axis: Number of Total Iterations (T), Vertical Axis: Misclustering Ratio.}
    \label{fig:KCluster_local}
\end{figure*}

\textbf{Perturbed Initialization:} Since we know the ground truth centers $\{\theta_i\}_{i=1}^K$ to generate the Gaussian data points, we add a small noise to these true centers and get the initialization. This initialization is primarily used to show the performance of LocalKMeans if we can have a very good initialization.

In all the experiments for synthetic Gaussian data, the experiments are repeated 20 times. We plotted the means of 20 trials with standard deviations in each figure. The x-axis in these figures is the number of total iterations, i.e., $T$ in Algorithm \ref{alg:fedkmeans}. The error metric is the misclustering. The experiments of synthetic datasets are conducted on a CPU with about 10s for each trial. Fig. \ref{fig:KCluster_local} shows the misclustering ratio of LocalKMeans with different initializations and SNRs. In this experiment, the dimension is $d=100$, number of clusters is $K=10$, number of compute nodes is $m=20$, number of samples per node is $n=200$.

We can see when SNR is relatively low (r=3.01), the misclustering ratio of Lloyd's without aggregation is still high in the end with the two aggregations. In this case the clients indeed need to collaborate to obtain a better clustering result. When increasing the number of local steps, the convergence of LocalKMeans is slower but they will converge to the same point finally. From the curves with $T/2$ local steps, it is clear that the misclustering ratio decreases significantly after every aggregation. When SNR is high, the Lloyd's without aggregation can get an acceptable performance with the worse LocalKMeans++ initialization, but not same as centralized Lloyd's. Again, doing only $2$ aggregations is sufficient to converge to the same final error as centralized Lloyd's. For large SNR, and strong perturbed initialization, all algorithms, including Lloyd's without aggregation perform well.

\subsection{Real Datasets}

\label{sec:real data}

\begin{figure*}[t!]
    \centering
    \subfigure[Motion Capture Hand Postures Dataset.]
    {
    \includegraphics[width=0.3\textwidth]{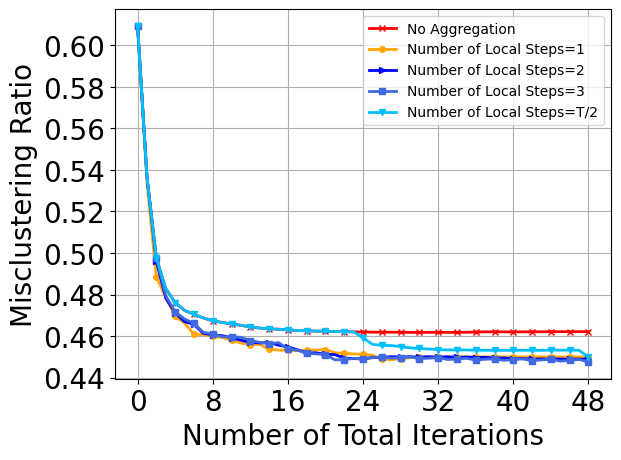}
    }
    \subfigure[MNIST Embeddings.]
    {
    \includegraphics[width=0.3\textwidth]{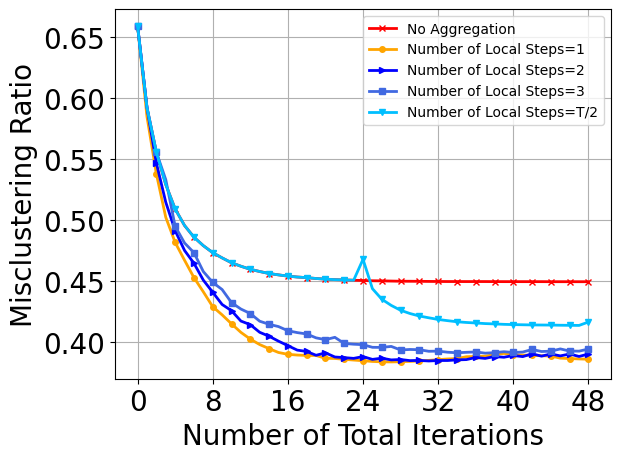}
    }
    \subfigure[CIFAR10 Embeddings.]
    {
    \includegraphics[width=0.3\textwidth]{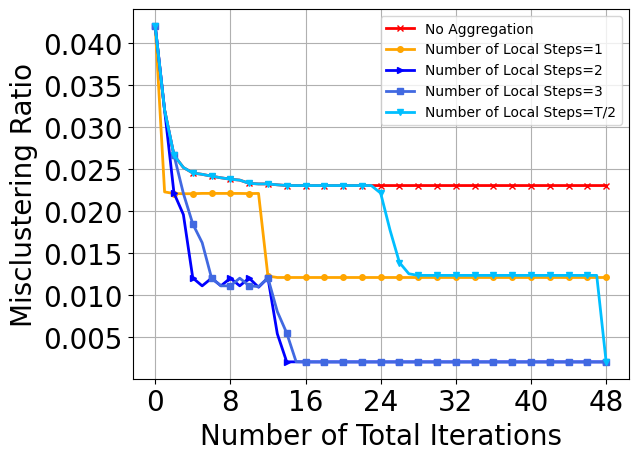}
    }
    \caption{Misclusering ratio on real datasets.}
    \label{fig:real_dataset_ratio}
    \vspace{-2mm}
\end{figure*}

\begin{figure*}[htbp]
    \centering
    \subfigure[Motion Capture Hand Postures Dataset.]
    {
    \includegraphics[width=0.3\textwidth]{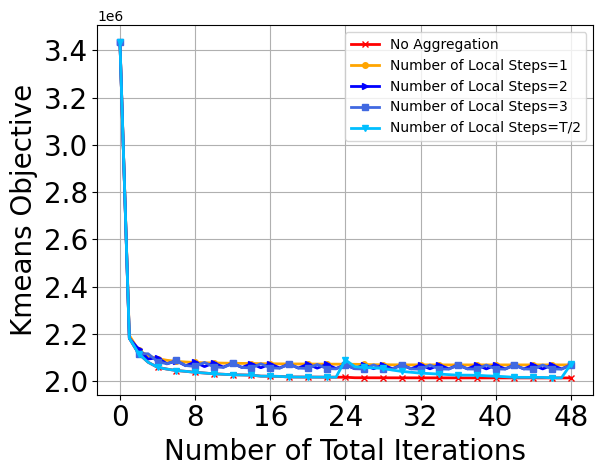}
    }
    \subfigure[MNIST Embeddings.]
    {
    \includegraphics[width=0.3\textwidth]{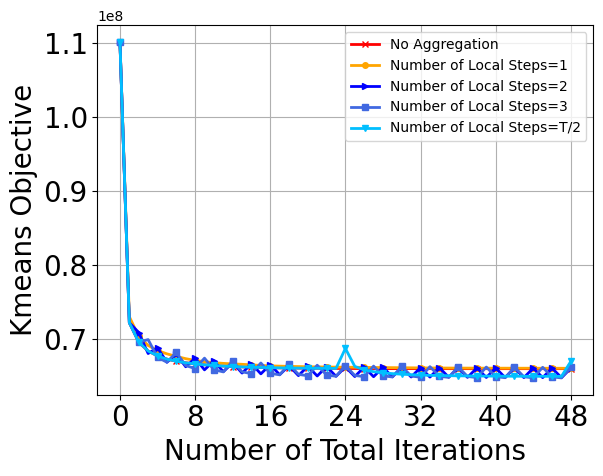}
    }
    \subfigure[CIFAR10 Embeddings.]
    {
    \includegraphics[width=0.3\textwidth]{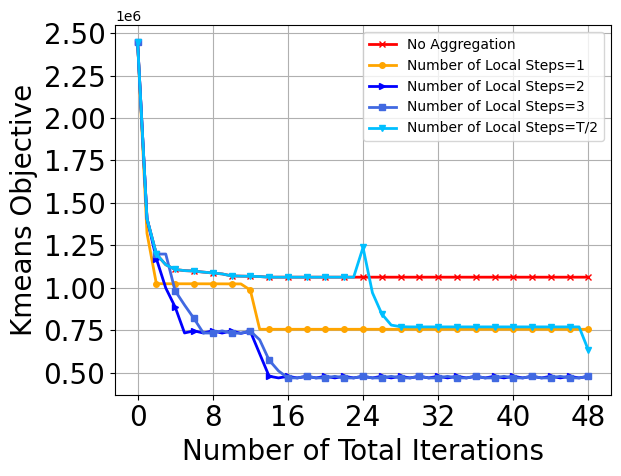}
    }
    \caption{KMeans objective on real datasets.}
    \label{fig:real_dataset_objective}
    \vspace{-2mm}
\end{figure*}

We further conducted experiments on real-world datasets. Here, we choose one clustering dataset from UCI Library \footnote{\url{https://archive.ics.uci.edu/datasets}}: Motion Capture Hand Posture (Posture) and two commonly used image classification datasets: MNIST and CIFAR10 datasets. For MNIST and CIFAR10 datasets, instead of using the raw images as the vector $\vx$, we use the embeddings of original images extracted from one intermediate layer of a trained CNN model as the datapoints $\vx$. Under this transformation, $\norm{\cdot}$ between different datapoints is more representative of their actual difference than if $\vx$ were pixel values. The detailed information about the datasets is shown in Table \ref{table-dataset}. We measure the performance of LocalKMeans with two metrics: global misclustering raio $A$, and KMeans objective. Note that real datasets are not mixture of gaussians (Eq~\eqref{def:mix_subg}), so we measure both misclustering and the KMeans objective. The experiments are performed on one GTX 1080Ti GPU with 5 minutes for each trial (clustering on embeddings of MNIST and CIFAR10 datasets).

Fig. \ref{fig:real_dataset_ratio} displays the misclustering ratio of LocalKMeans with different number of local steps on the real datasets. On all the three datasets, the misclustering ratio is higher for Lloyd's without aggregation. With aggregation the performance is much better, suggesting the benefit of collaborating with other compute nodes. For Posture dataset and MNIST embeddings, the number of local steps has minor influence, while on CIFAR10 embeddings, the performance of centralized Lloyd's ($L=1$) is worse than the performance with multiple local steps, implying there are some situations  a few local steps can benefit the distributed clustering. And the misclustering ratio is pretty low for CIFAR10 embeddings, meaning that it may be more suitable for Euclidean distance-based clustering metric. Fig. \ref{fig:real_dataset_objective} shows the KMeans objective on the three datasets. The difference between no aggregation and LocalKMeans is minor for Posture dataset and MNIST embeddings. Meanwhile, we can still observe the benefit of aggregation from CIFAR10 embeddings. And with more than one local steps, the KMeans objective is smaller on CIFAR10 embeddings. For all datasets, LocalKMeans with $L=\frac{T}{2}$ does not achieve the best possible misclustering. This shows that  a few aggregations are necessary.

\section{Conclusion}
\label{sec:conclusion}
\vspace{-1mm}
We have proposed a distributed variant of Lloyd's algorithm, which incorporates local steps for communication efficiency. We use a novel modification of the virtual iterate method to show that theoretically, it requires the same initialization as centralized Lloyd's, but, it's final error is slightly worse.

In practice, its performance is similar to centralized Lloyd's. 
Important directions for future work include improving the final error for LocalKMeans, and extending it to clients that are heterogeneous. Note that ~\cite{lu_statistical_2016} provides the tightest analysis of Lloyd's algorithm~\cite{lu_statistical_2016}, however, a weaker analysis may be more amenable to  local steps.

\bibliography{references}

\begin{thebibliography}{35}
\providecommand{\natexlab}[1]{#1}
\providecommand{\url}[1]{\texttt{#1}}
\expandafter\ifx\csname urlstyle\endcsname\relax
  \providecommand{\doi}[1]{doi: #1}\else
  \providecommand{\doi}{doi: \begingroup \urlstyle{rm}\Url}\fi

\bibitem[Arthur and Vassilvitskii(2006)]{arthur_how_2006}
David Arthur and Sergei Vassilvitskii.
\newblock How slow is the k-means method?
\newblock In \emph{Proceedings of the twenty-second annual symposium on
  {Computational} geometry}, {SCG} '06, pages 144--153, New York, NY, USA, June
  2006. Association for Computing Machinery.
\newblock ISBN 978-1-59593-340-9.
\newblock \doi{10.1145/1137856.1137880}.

\bibitem[Arthur and Vassilvitskii(2007)]{arthur_k-means_2007}
David Arthur and Sergei Vassilvitskii.
\newblock k-means++: the advantages of careful seeding.
\newblock In \emph{Proceedings of the eighteenth annual {ACM}-{SIAM} symposium
  on {Discrete} algorithms}, {SODA} '07, pages 1027--1035, USA, January 2007.
  Society for Industrial and Applied Mathematics.
\newblock ISBN 978-0-89871-624-5.

\bibitem[Awasthi and Sheffet(2012)]{awasthi_improved_2012}
Pranjal Awasthi and Or~Sheffet.
\newblock Improved {Spectral}-{Norm} {Bounds} for {Clustering}, June 2012.
\newblock arXiv:1206.3204 [cs].

\bibitem[Balakrishnan et~al.(2017)Balakrishnan, Wainwright, and
  Yu]{balakrishnan_statistical_2017}
Sivaraman Balakrishnan, Martin~J. Wainwright, and Bin Yu.
\newblock Statistical guarantees for the {EM} algorithm: {From} population to
  sample-based analysis.
\newblock \emph{The Annals of Statistics}, 45\penalty0 (1):\penalty0 77--120,
  February 2017.
\newblock ISSN 0090-5364, 2168-8966.
\newblock \doi{10.1214/16-AOS1435}.
\newblock Publisher: Institute of Mathematical Statistics.

\bibitem[Bao et~al.(2024)Bao, Crawshaw, and Liu]{bao2024provable}
Yajie Bao, Michael Crawshaw, and Mingrui Liu.
\newblock Provable benefits of local steps in heterogeneous federated learning
  for neural networks: A feature learning perspective.
\newblock In \emph{Forty-first International Conference on Machine Learning},
  2024.

\bibitem[Chaudhuri et~al.(2009)Chaudhuri, Dasgupta, and
  Vattani]{chaudhuri_learning_2009}
Kamalika Chaudhuri, Sanjoy Dasgupta, and Andrea Vattani.
\newblock Learning {Mixtures} of {Gaussians} using the k-means {Algorithm},
  December 2009.
\newblock arXiv:0912.0086 [cs].

\bibitem[Daskalakis et~al.(2017)Daskalakis, Tzamos, and Zampetakis]{das_ten}
Constantinos Daskalakis, Christos Tzamos, and Manolis Zampetakis.
\newblock Ten steps of em suffice for mixtures of two gaussians.
\newblock In Satyen Kale and Ohad Shamir, editors, \emph{Proceedings of the
  2017 Conference on Learning Theory}, volume~65 of \emph{Proceedings of
  Machine Learning Research}, pages 704--710. PMLR, 07--10 Jul 2017.

\bibitem[Dennis et~al.(2021)Dennis, Li, and Smith]{pmlr-v139-dennis21a}
Don~Kurian Dennis, Tian Li, and Virginia Smith.
\newblock Heterogeneity for the win: One-shot federated clustering.
\newblock In Marina Meila and Tong Zhang, editors, \emph{Proceedings of the
  38th International Conference on Machine Learning}, volume 139 of
  \emph{Proceedings of Machine Learning Research}, pages 2611--2620. PMLR,
  18--24 Jul 2021.

\bibitem[Dieuleveut et~al.(2021)Dieuleveut, gersende fort, Moulines, and
  Robin]{dieuleveut2021federatedem}
Aymeric Dieuleveut, gersende fort, Eric Moulines, and Genevi{\`e}ve Robin.
\newblock Federated-{EM} with heterogeneity mitigation and variance reduction.
\newblock In A.~Beygelzimer, Y.~Dauphin, P.~Liang, and J.~Wortman Vaughan,
  editors, \emph{Advances in Neural Information Processing Systems}, 2021.

\bibitem[Gardner et~al.(2014)Gardner, Selmic, and
  Kanno]{motion_capture_hand_postures_405}
A.~Gardner, R.~Selmic, and J.~Kanno.
\newblock {Motion Capture Hand Postures}.
\newblock UCI Machine Learning Repository, 2014.
\newblock {DOI}: https://doi.org/10.24432/C5TG86.

\bibitem[Gorbunov et~al.(2021)Gorbunov, Hanzely, and
  Richtarik]{gorbunov_local_2021}
Eduard Gorbunov, Filip Hanzely, and Peter Richtarik.
\newblock Local {SGD}: {Unified} {Theory} and {New} {Efficient} {Methods}.
\newblock In \emph{Proceedings of {The} 24th {International} {Conference} on
  {Artificial} {Intelligence} and {Statistics}}, pages 3556--3564. PMLR, March
  2021.
\newblock ISSN: 2640-3498.

\bibitem[Haddadpour et~al.(2019)Haddadpour, Kamani, Mahdavi, and
  Cadambe]{haddadpour_local_2019}
Farzin Haddadpour, Mohammad~Mahdi Kamani, Mehrdad Mahdavi, and Viveck Cadambe.
\newblock Local {SGD} with {Periodic} {Averaging}: {Tighter} {Analysis} and
  {Adaptive} {Synchronization}.
\newblock In H.~Wallach, H.~Larochelle, A.~Beygelzimer, F.~d' Alché-Buc,
  E.~Fox, and R.~Garnett, editors, \emph{Advances in {Neural} {Information}
  {Processing} {Systems}}, volume~32. Curran Associates, Inc., 2019.

\bibitem[Karimireddy et~al.(2020)Karimireddy, Kale, Mohri, Reddi, Stich, and
  Suresh]{karimireddy_scaffold_2020}
Sai~Praneeth Karimireddy, Satyen Kale, Mehryar Mohri, Sashank Reddi, Sebastian
  Stich, and Ananda~Theertha Suresh.
\newblock {SCAFFOLD}: {Stochastic} {Controlled} {Averaging} for {Federated}
  {Learning}.
\newblock In \emph{Proceedings of the 37th {International} {Conference} on
  {Machine} {Learning}}, pages 5132--5143. PMLR, November 2020.
\newblock ISSN: 2640-3498.

\bibitem[Koloskova et~al.(2020)Koloskova, Loizou, Boreiri, Jaggi, and
  Stich]{pmlr-v119-koloskova20a}
Anastasia Koloskova, Nicolas Loizou, Sadra Boreiri, Martin Jaggi, and Sebastian
  Stich.
\newblock A unified theory of decentralized {SGD} with changing topology and
  local updates.
\newblock In Hal~Daumé III and Aarti Singh, editors, \emph{Proceedings of the
  37th International Conference on Machine Learning}, volume 119 of
  \emph{Proceedings of Machine Learning Research}, pages 5381--5393. PMLR,
  13--18 Jul 2020.

\bibitem[Konecn{\`y} et~al.(2016)Konecn{\`y}, McMahan, Yu, Richt{\'a}rik,
  Suresh, and Bacon]{konecny2016federated}
Jakub Konecn{\`y}, H~Brendan McMahan, Felix~X Yu, Peter Richt{\'a}rik,
  Ananda~Theertha Suresh, and Dave Bacon.
\newblock Federated learning: Strategies for improving communication
  efficiency.
\newblock \emph{arXiv preprint arXiv:1610.05492}, 8, 2016.

\bibitem[Krizhevsky et~al.()Krizhevsky, Nair, and Hinton]{cifar10}
Alex Krizhevsky, Vinod Nair, and Geoffrey Hinton.
\newblock Cifar-10 (canadian institute for advanced research).

\bibitem[Kumar and Kannan(2010)]{kumar_clustering_2010}
Amit Kumar and Ravindran Kannan.
\newblock Clustering with {Spectral} {Norm} and the k-means {Algorithm}, April
  2010.
\newblock arXiv:1004.1823 [cs].

\bibitem[Kumar et~al.(2020)Kumar, V~R, and Nair]{fed_kmeans_1}
Hemant~H Kumar, Karthik V~R, and Mydhili~K Nair.
\newblock Federated k-means clustering: A novel edge ai based approach for
  privacy preservation.
\newblock In \emph{2020 IEEE International Conference on Cloud Computing in
  Emerging Markets (CCEM)}, pages 52--56, 2020.
\newblock \doi{10.1109/CCEM50674.2020.00021}.

\bibitem[Laurent and Massart(2000)]{laurent_massart}
B.~Laurent and P.~Massart.
\newblock {Adaptive estimation of a quadratic functional by model selection}.
\newblock \emph{The Annals of Statistics}, 28\penalty0 (5):\penalty0 1302 --
  1338, 2000.
\newblock \doi{10.1214/aos/1015957395}.

\bibitem[LeCun and Cortes(2010)]{mnist}
Yann LeCun and Corinna Cortes.
\newblock {MNIST} handwritten digit database.
\newblock 2010.

\bibitem[Li et~al.(2022)Li, Hou, Buyukates, and Avestimehr]{li2022secure}
Songze Li, Sizai Hou, Baturalp Buyukates, and Salman Avestimehr.
\newblock Secure federated clustering.
\newblock \emph{arXiv preprint arXiv:2205.15564}, 2022.

\bibitem[Lloyd(1982)]{lloyd}
S.~Lloyd.
\newblock Least squares quantization in pcm.
\newblock \emph{IEEE Transactions on Information Theory}, 28\penalty0
  (2):\penalty0 129--137, 1982.
\newblock \doi{10.1109/TIT.1982.1056489}.

\bibitem[Lu and Zhou(2016)]{lu_statistical_2016}
Yu~Lu and Harrison~H. Zhou.
\newblock Statistical and {Computational} {Guarantees} of {Lloyd}'s {Algorithm}
  and its {Variants}, December 2016.
\newblock arXiv:1612.02099 [cs, math, stat].

\bibitem[Marfoq et~al.(2021)Marfoq, Neglia, Bellet, Kameni, and
  Vidal]{marfoq_federated_2021}
Othmane Marfoq, Giovanni Neglia, Aurélien Bellet, Laetitia Kameni, and Richard
  Vidal.
\newblock Federated {Multi}-{Task} {Learning} under a {Mixture} of
  {Distributions}.
\newblock In M.~Ranzato, A.~Beygelzimer, Y.~Dauphin, P.~S. Liang, and
  J.~Wortman Vaughan, editors, \emph{Advances in {Neural} {Information}
  {Processing} {Systems}}, volume~34, pages 15434--15447. Curran Associates,
  Inc., 2021.

\bibitem[Ostrovsky et~al.(2013)Ostrovsky, Rabani, Schulman, and
  Swamy]{ostrovsky_effectiveness_2013}
Rafail Ostrovsky, Yuval Rabani, Leonard~J. Schulman, and Chaitanya Swamy.
\newblock The effectiveness of lloyd-type methods for the k-means problem.
\newblock \emph{Journal of the ACM}, 59\penalty0 (6):\penalty0 28:1--28:22,
  January 2013.
\newblock ISSN 0004-5411.
\newblock \doi{10.1145/2395116.2395117}.

\bibitem[Pan et~al.(2022)Pan, Sima, Prakash, Rana, and
  Milenkovic]{pan2022machine}
Chao Pan, Jin Sima, Saurav Prakash, Vishal Rana, and Olgica Milenkovic.
\newblock Machine unlearning of federated clusters.
\newblock \emph{arXiv preprint arXiv:2210.16424}, 2022.

\bibitem[Stich(2019)]{stich2018local}
Sebastian~U. Stich.
\newblock Local {SGD} converges fast and communicates little.
\newblock In \emph{International Conference on Learning Representations}, 2019.

\bibitem[Stich et~al.(2018)Stich, Cordonnier, and Jaggi]{stich_sparsified_2018}
Sebastian~U Stich, Jean-Baptiste Cordonnier, and Martin Jaggi.
\newblock Sparsified {SGD} with {Memory}.
\newblock In \emph{Advances in {Neural} {Information} {Processing} {Systems}},
  volume~31. Curran Associates, Inc., 2018.

\bibitem[Tian et~al.(2024)Tian, Weng, and Feng]{tian_towards_2024}
Ye~Tian, Haolei Weng, and Yang Feng.
\newblock Towards the {Theory} of {Unsupervised} {Federated} {Learning}:
  {Non}-asymptotic {Analysis} of {Federated} {EM} {Algorithms}, February 2024.
\newblock arXiv:2310.15330 [cs, stat].

\bibitem[Vattani(2009)]{vattani}
Andrea Vattani.
\newblock k-means requires exponentially many iterations even in the plane.
\newblock In \emph{Proceedings of the Twenty-Fifth Annual Symposium on
  Computational Geometry}, SCG '09, page 324–332, New York, NY, USA, 2009.
  Association for Computing Machinery.
\newblock ISBN 9781605585017.
\newblock \doi{10.1145/1542362.1542419}.

\bibitem[Woodworth et~al.(2020)Woodworth, Patel, and
  Srebro]{woodworth_minibatch_2020}
Blake~E Woodworth, Kumar~Kshitij Patel, and Nati Srebro.
\newblock Minibatch vs {Local} {SGD} for {Heterogeneous} {Distributed}
  {Learning}.
\newblock In \emph{Advances in {Neural} {Information} {Processing} {Systems}},
  volume~33, pages 6281--6292. Curran Associates, Inc., 2020.

\bibitem[Yi et~al.(2014)Yi, Caramanis, and Sanghavi]{yi_two}
Xinyang Yi, Constantine Caramanis, and Sujay Sanghavi.
\newblock Alternating minimization for mixed linear regression.
\newblock In Eric~P. Xing and Tony Jebara, editors, \emph{Proceedings of the
  31st International Conference on Machine Learning}, volume~32 of
  \emph{Proceedings of Machine Learning Research}, pages 613--621, Bejing,
  China, 22--24 Jun 2014. PMLR.

\bibitem[Yi et~al.(2016)Yi, Caramanis, and Sanghavi]{yi2016solving}
Xinyang Yi, Constantine Caramanis, and Sujay Sanghavi.
\newblock Solving a mixture of many random linear equations by tensor
  decomposition and alternating minimization.
\newblock \emph{arXiv preprint arXiv:1608.05749}, 2016.

\bibitem[Zhang et~al.(2023)Zhang, Kuang, Chen, You, Shen, Xiao, Zhang, Wu, Wu,
  Zhuang, et~al.]{zhang2023federated}
Fengda Zhang, Kun Kuang, Long Chen, Zhaoyang You, Tao Shen, Jun Xiao, Yin
  Zhang, Chao Wu, Fei Wu, Yueting Zhuang, et~al.
\newblock Federated unsupervised representation learning.
\newblock \emph{Frontiers of Information Technology \& Electronic Engineering},
  24\penalty0 (8):\penalty0 1181--1193, 2023.

\bibitem[Zhou and Wang(2024)]{fed_cluster_2}
Xiaochen Zhou and Xudong Wang.
\newblock Memory and communication efficient federated kernel k-means.
\newblock \emph{IEEE Transactions on Neural Networks and Learning Systems},
  35\penalty0 (5):\penalty0 7114--7125, 2024.
\newblock \doi{10.1109/TNNLS.2022.3213777}.

\end{thebibliography}
\bibliographystyle{plainnat}

\newpage
\appendix
\onecolumn
\section{Additional details for LocalKMeans on $K$-clusters }
\label{sec:k_cluster}

\paragraph{Initialization requirement(Assumption~\ref{assumption:local_dev_k}}
Note that required initializaton is same as centralized Lloyd's ~\cite{lu_statistical_2016} for $mn$ datapoints. The initialization now depends on the quantity $\alpha$ which is the relative size of the smallest cluster. Therefore, if one cluster is very small, the required SNR is large. If the clusters are balanced globally, then the required SNR is small.

\textbf{Single-Step Progress (Lemma~\ref{lem:single_step_k}):} This Lemma resembles the corresponding single step progress for symmetric $2$-cluster case (Lemma~\ref{lem:single_step}), with guarantees in terms of both misclustering $G^{(t)}$ and $\Lambda^{(t)}$. Note that the connection between error in cluster centers and misclustering was exact for $2$-clusters, but here we need to bound the worst-case error over all clusters. Further, the conditions on $A_i$ being close to $A$ and $A\leq \frac{1}{4}$ for the $2$-cluster case now translate to a uniform bound on both $G^{(t)}, G_i^{(t)}$ and $\Lambda^{(t)}$ and a bound on $\Lambda_i^{(t)}$ being a constant away from $\Lambda^{(t)}$. These are consequences of a more careful analysis of the iterations when there is aggregation and when there is no aggregation. Similar to the $2$-cluster case, local steps cost an additional error term $\Psi_2$ and the single-step progress depends on $G^{(t)}$ and $G^{(t-1)}$. Additionally, we also require the clusters to be balanced in terms of their relative difference between cluster centers, i.e. $\lambda = \mathcal{O}(1)$ and in terms of their local sizes across clients, i.e., $\beta m = \mathcal{O}(1)$. The SNR $r_K$ again depends on the term $\frac{d}{n}$. Note that a condition on $\alpha$, the minimum size of each cluster globally is also required, same as centralized Lloyd's.

\textbf{Final Error (Theorem~\ref{thm:k_cluster}):} Like $2$-cluster case, the number of iterations required to achieve the final error is $2\lceil \log(mn) \rceil + 2L$. The final error has terms depending on $\frac{d}{nr_K^2}$ and $\exp(-n)$ due to local steps. The proof of this theorem uses a two stage analysis, first establishing geometric progress until $t_0 = 2\ceil{\log(mn)} + 2L$ iterations, followed by a careful analysis beyond those iterations. Note that these additional iterations, depending on $q$, can make the term dependent on $mn$ arbitrarily small.

We provide the proof for all results in this section and Section~\ref{sec:k_cluster_summarized} in Appendix~\ref{sec:k_cluster_proof}.

\subsection{Comparison to baselines}
\label{sec:baselines_k}
 We can obtain the guarantees of Centralized Lloyd's by substituting the number of datapoints as $mn$ in ~\citep[Theorem~3.2]{lu_statistical_2016}. Note that it requires the same initialization as LocalKMeans, however, LocalKMeans requires $2L$ more iterations and it's final error scales as $\frac{d}{n r_K^2}$ and $\exp(-n)$ instead of $\exp(-r_K^2)$ for centralized Lloyd's. Additionally, we need $r_K \geq \sqrt{\frac{Kd}{n}}$, and constant $\lambda$ and $\beta m$. These are not required for centralized Lloyd's.

Guarantees of Lloyd's without aggregation can be obtained by applying  ~\citep[Theorem~3.2]{lu_statistical_2016} with number of datapoints as $n$ for each machine individually.  In terms of initialization, each machine $i\in [m]$ would need to satisfy $\Lambda_i^{(t)} \leq \frac{1}{2} - \frac{4}{\sqrt{r_{K,i}}}$, where the SNR on the $i^{th}$ machine is $r_{K,i} = \frac{\Gamma}{\sigma}\sqrt{\alpha_i\left(1 + \frac{Kd}{n}\right)^{-1}}$ and $\alpha_i \equiv \min_{k\in [K]}\frac{\nu_{k,i}^\star}{n}$ is fraction of datapoints lying in the smallest cluster on the $i^{th}$ machine and $\alpha_i \geq \sqrt{\frac{K\log(n)}{n}}$. For the high probability results to hold on all machines, we need $m = \rm poly(n)$. Compared to LocalKMeans, the initialization condition from Assumption~\ref{assumption:local_dev_k} is much weaker  as  $\frac{1}{\sqrt{r_K}} \leq \frac{1}{\sqrt{\min_{i\in [m]}r_{K,i}}})$, which is obtained from if $\alpha \geq \min_{i\in [m]} \alpha_i$. Additionally, if $ m \in (\mathcal{O}(\rm poly n), \mathcal{O}(2^n)]$ or $\min_{i\in [m]} \alpha_i \in [\sqrt{\nicefrac{K\log(mn)}{mn}},\sqrt{\nicefrac{K\log(n)}{n}} )$, Lloyd's without aggregation does not converge theoretically. However, LocalKMeans can converge. Further, $\min_{i\in [m]} \alpha_i = \alpha\beta$, for which LocalKMeans only requires $\beta m = \mathcal{O}(1)$. Therefore, theoretically even in the $K$-cluster case, LocalKMeans is much better than Lloyd's without aggregation.

\section{Additional Experimental Results}
\label{sec:add_exp}
\subsection{Local KMeans++ algorithm}
\label{sec:localkmeans++}
KMeans++~\cite{arthur_k-means_2007} provides an easy method for selecting initial cluster centers from the given set of datapoints. We first provide an extension of KMeans++ to the Federated setup in Algorithm~\ref{alg:localkmeans++}. Note that here $\Theta$ is the set of initial cluster centers with size $\abs{\Theta} = K$.

\begin{algorithm}[htbp!]
    \label{alg:localkmeans++}
    \caption{LocalKMeans++ Initialization}
    \begin{algorithmic}
    \REQUIRE 
    \ENSURE
        \STATE $\Theta \gets \{\}$
        \WHILE{$\abs{\Theta} \leq K$}
        \FOR{ all machines $i \in [m]$}
        \FOR{all datapoints $j \in [n]$}
        \IF{$\abs{\Theta} > 0$}
        \STATE $E_{i,j} \gets \min_{\vx\in \Theta} d(\vy_{i,j},\vx)$
        \ELSE
        \STATE $E_{i,j} \gets 1$
        \ENDIF
        \ENDFOR
        \STATE $E_i \gets \sum_{j=1}^n E_{i,j}$
        \ENDFOR
        \STATE Select machine $i$ with probability $ \frac{E_i}{\sum_{i=1}^m E_i}$
        \STATE Inside selected machine $i$, select $j^{th}$ datapoint with probability $\frac{E_{i,j}}{E_i}$
        \STATE $\Theta \gets \Theta \cup \{\vy_{i,j}\}$
        \ENDWHILE
    \end{algorithmic}
\end{algorithm}

\begin{table}[htbp]\footnotesize
    \centering
    
    \small
    \begin{tabular}{ccccccc}
        \toprule
         \makecell{Dataset} & \makecell{\# of \\Samples} & \makecell{\# of\\ machines} & \makecell{Samples \\per machine} & Dimension & \makecell{\# of
         \\Clusters} \\
         \midrule
         Posture\\
         ~\cite{motion_capture_hand_postures_405} &  78095 & 100 & 781 & 38 & 5 \\
         \makecell{CIFAR10 Embeddings\\ ~\cite{cifar10}}  & 50000 & 100 & 500 & 512 & 10 \\
         \makecell{MNIST Embeddings\\ ~\cite{mnist}}  & 50000 & 100 & 500 & 1568 & 10 \\
        \bottomrule
    \end{tabular}

\caption{Detailed information about real datasets\label{table-dataset}}
\end{table}

\subsection{Additional Experiments on Synthetic Dataset}
Fig. \ref{fig: KCluster_SNR} displays the performance of LocalKmeans when we change various parameters, such as SNR, number of clusters $K$, number of compute nodes $m$. In these experiments, the number of local steps is fixed as 3 and we use the LocalKmeans++ initialization. From Fig. \ref{fig: KCluster_SNR}(a) we can see when the SNR is higher, the initialization obtained by LocalKmeans++ is better, and then the final clustering result is also better. When SNR is lower than a threshold, the LocalKmeans actually cannot converge and the misclustering ratio is very high. Fig. \ref{fig: KCluster_SNR}(b) shows the performance with different number of clusters $K$ and number of samples per node is fixed as 200. When there are more clusters to be classified, the number of samples in one cluster is less and the performance of LocalKmeans is worse. Fig. \ref{fig: KCluster_SNR}(c) shows the performance with different number of compute nodes $m$ and the total number of samples $mn$ is fixed as 4000. When $m$ is 10 and 20, while $n$ is 400 and 200 correspondingly, the performances of the two cases are similar and the misclustering ratio is very low finally. However, when the number of compute nodes is larger, with less samples in one node, the performance becomes poor and the misclustering ratio is still high at the final stage.

\begin{figure*}[htbp!]
    \centering
    \subfigure[Different SNR.]
    {
    \includegraphics[width=0.3\textwidth]{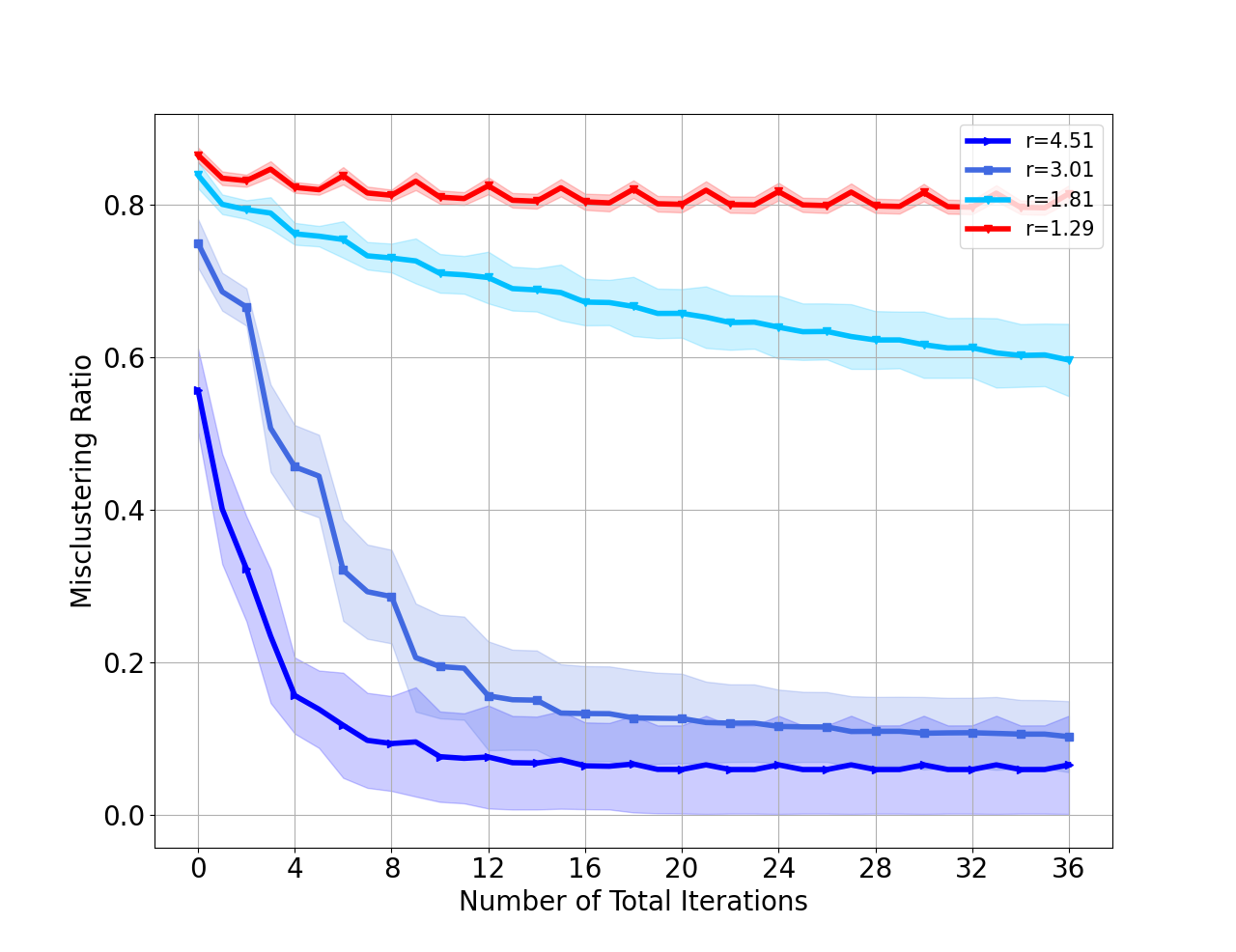}
    }
    \subfigure[Different Number of Clusters.]
    {
    \includegraphics[width=0.3\textwidth]{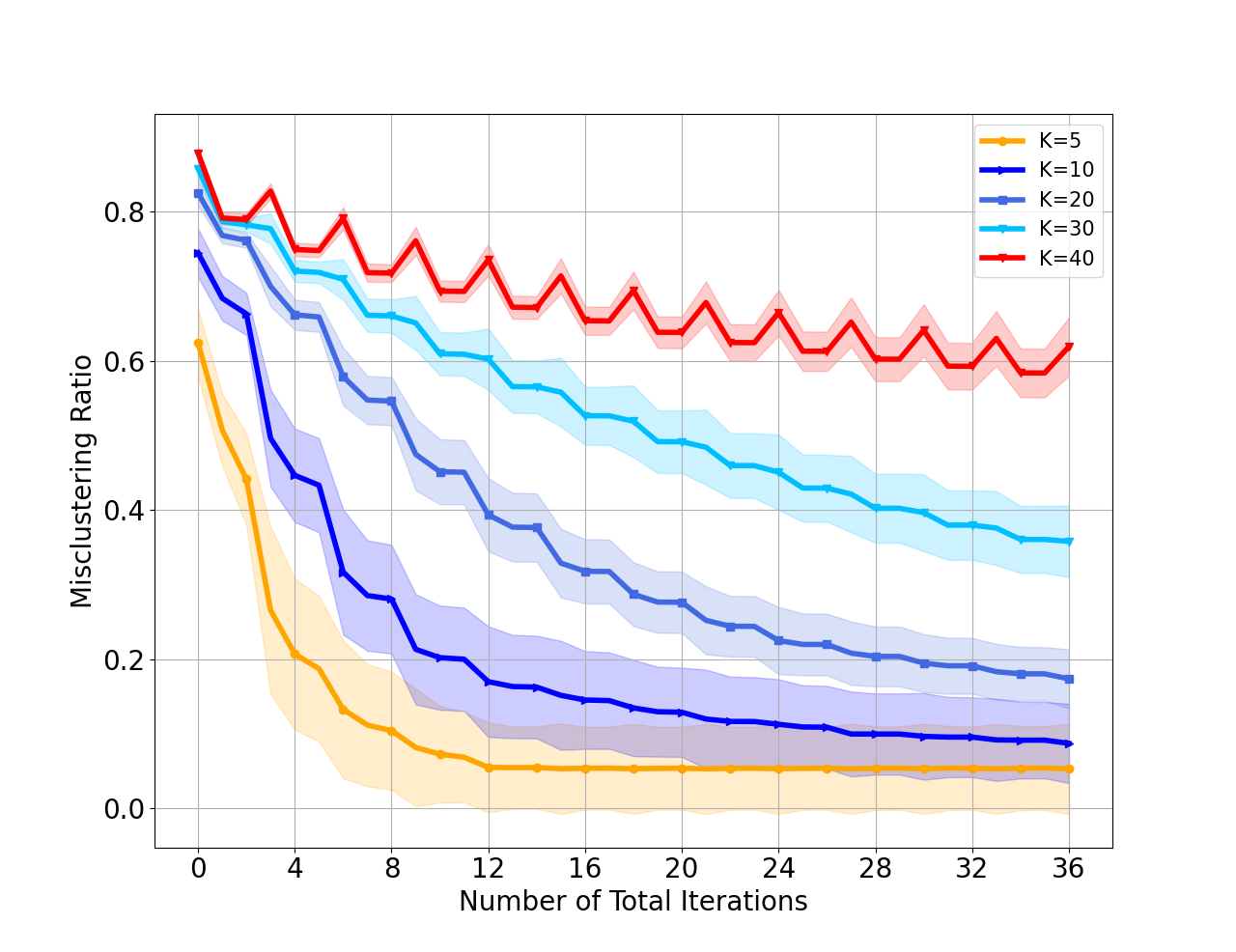}
    }
    \subfigure[Different Number of Compute Nodes.]
    {
    \includegraphics[width=0.3\textwidth]{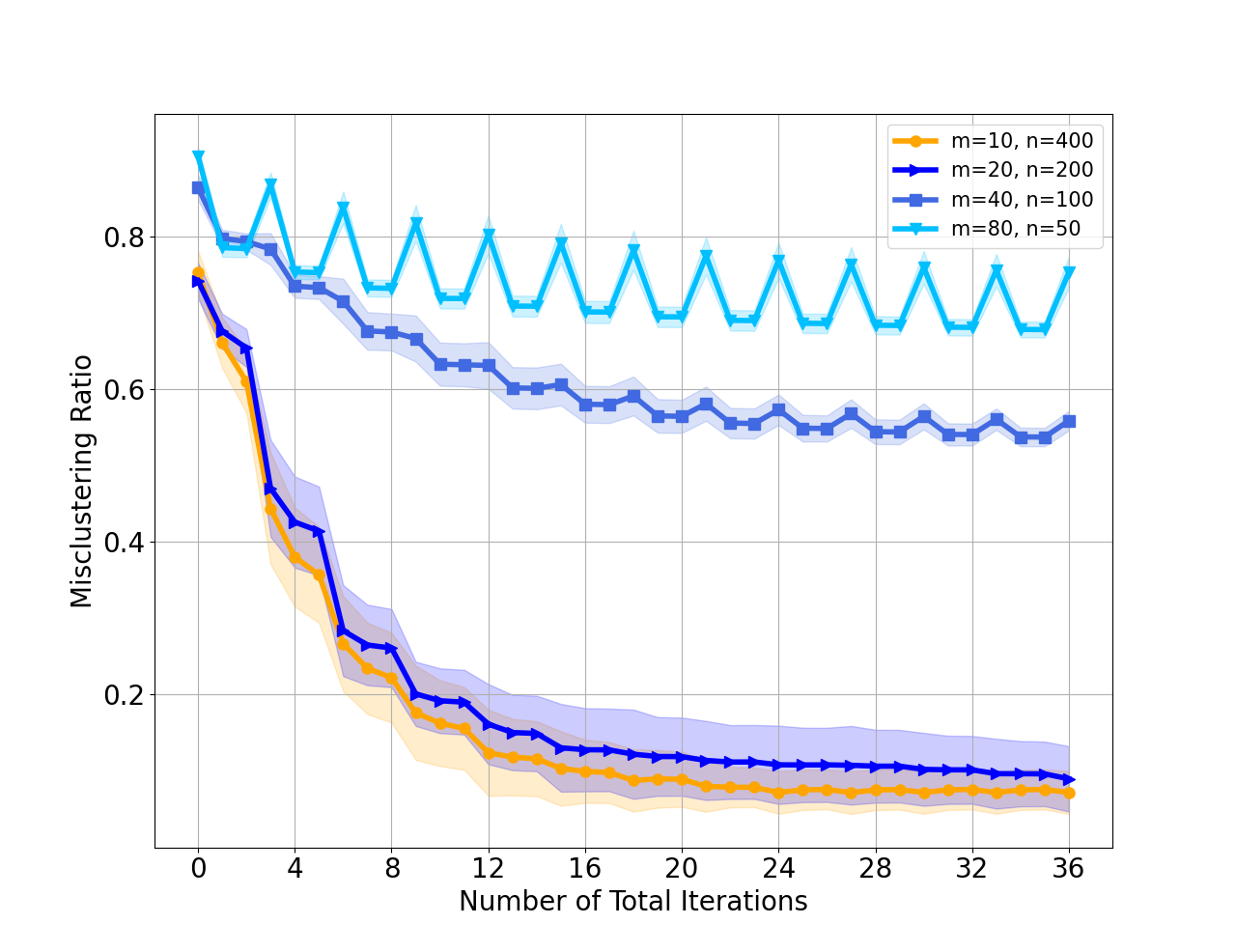}
    }
    \caption{Performance of LocalKmeans on Gaussian Data with Different Factors. Horizontal Axis: Number of Total Iterations. Vertical Axis: Misclustering Ratio.}
    \label{fig: KCluster_SNR}
    \vspace{-2mm}
\end{figure*}

\section{Proofs for Section~\ref{sec:2_cluster} }
\label{sec:2_cluster_proof}

The proof is based on virtual iterates from ~\cite{karimireddy_scaffold_2020}. We extend ~\citep[Section~7.1, 7.2]{lu_statistical_2016} to virtual iterates for this proof. For the sake of convenience, we redefine $z_{i,j}\vx_{i,j} = \theta^\star + \vw_{i,j}, \forall i\in [m], j\in [n]$. Note that this multiplies the original noise vector $\vw$ by either a $+1$ or a $-1$, which doesn't change it's distribution. We first state several technical lemmas that we use throughout our proof.

\subsection{Technical Lemmas}
    We first extend Lemmas used in ~\cite{lu_statistical_2016} for concentration of $\vw_{i,j}$, to $mn$ points in our federated settings.
    \begin{lemma}[Extension of Lemma 7.1 ~\cite{lu_statistical_2016}]\label{lem:noise_set}
        For any $S \subseteq [m]\times [n]$, define $W_S = \sum_{(i,j) \in S} \vw_{i,j}$. Then,
        \begin{align}
            \norm{W_S} \leq \sigma\sqrt{2(mn + 9d)\abs{S}}, \text{ for all } S \subseteq [m]\times [n]
        \end{align}
        wp $> 1 - \exp(-mn)$
        Further, if $\log(m) = \mathcal{O}(n)$, then taking a union bound over all machines individually, with probability $1 - \exp(-n)$, we have, for each $i\in [m]$, $S_i\subseteq[n]$, define $W_{S_i} = \sum_{j \in S_i} \vw_{i,j}$. Then,
        \begin{align}
            \norm{W_{S_i}} \leq \sigma\sqrt{2(n + 9d)\abs{S}}, \text{ for all } S_i \subseteq [n], i\in [m]
        \end{align}
    \end{lemma}
    This Lemma bounds the sum of subgaussian noise terms in a set of datapoints $S$ in terms of the size of $S$. We provide an alternative for a
    \begin{lemma}[Extension of Lemma 7.2 ~\cite{lu_statistical_2016}]\label{lem:avg_noise}
        Let $\bar{\vw} = \frac{1}{mn}\sum_{i\in [m]} \sum_{j\in [n]} \vw_{i,j}$. Then, we have $\ip{\bar{\vw}}{\theta^\star} \geq -\frac{\norm{\theta^\star}^2}{\sqrt{mn}}$
        and $\norm{\bar{\vw}}^2 \leq \frac{3d\sigma^2}{mn} + \frac{\norm{\theta^\star}}{mn}$ wp $> 1 - 2\exp\bigg(-\frac{\norm{\theta^\star}^2}{3\sigma^2}\bigg)$
    \end{lemma}
    This Lemma bounds the norm and the inner product with optimal model $\theta^\star$,  of the average noise over all datapoints.
    
    \begin{lemma}[Extension of Lemma 7.3 from ~\cite{lu_statistical_2016}]\label{lem:noise_eigen}
        $\lambda_{\max}\bigg(\sum_{i\in [m]}\sum_{j\in [n]} \vw_{i,j}\vw_{i,j}^\intercal\bigg) \leq 1.62(mn + 4d)\sigma^2$
        wp $ > 1 - \exp(-mn)$.
    \end{lemma}
    This Lemma bounds the max eigenvalue of the empirical noise covariance. 
    
    \begin{lemma}[Extension of Lemma 7.4~\cite{lu_statistical_2016}]\label{lem:noise_subset_subexp}
        For any fixed $(i,j) \in [m]\times [n]$, $S \subseteq [m]\times [n]$, $t > 0$ and $\delta > 0$, we have
        \begin{align}
            \bP\bigg\{ \ip{\vw_{i,j}}{\frac{1}{\abs{S}}\sum_{(i',j') \in S} \vw_{i',j'}} \geq \frac{3\sigma^2(t\sqrt{\abs{S}} + d + \log(1/\delta))}{\abs{S}}\bigg\}\leq \exp\left(-\min \bigg\{\frac{t^2}{4d}, \frac{t}{4}\bigg\}\right) + \delta
        \end{align}
    \end{lemma}
    This Lemma establishes subgaussian concentration for the inner product of average noise of a set $S$ and any datapoint's noise. 

    \begin{lemma}[Linear combination of $\chi^2$ random variables~\cite{laurent_massart}]\label{lem:chi_square_conc}
If $q_1,q_2,\ldots, q_n$ are independent $\chi^2_d$ random variables and $a = [a_1,a_2,\ldots, a_n]^\intercal \in \R^n_{+}$ where each $a_i$ is non-negative, then for any $t>0$, we have  
    \begin{equation}
    \bP\bigg[\sum_{i=1}^n a_i q_i \geq \norm{a}_1 d + 2\sqrt{\norm{a}_2^2 d t} + 2\norm{a}_\infty t\bigg]\leq \exp(-t)
    \end{equation}
    where $\norm{a}_1 = \sum_{i=1}^n \abs{a_i}, \norm{a}_2 = \sqrt{\sum_{i=1}^n a_i^2}$ and $\norm{a}_{\infty} = \max_{i \in [n]}\abs{a_i}$
\end{lemma}
This lemma on concentration of $\chi^2$ random variables as this appears in our proof.
\begin{lemma}[Average Noise on each machine]\label{lem:avg_noise_machine}
With probability $1 - \exp(-r^2)$, we have,
\begin{align*}
    \frac{1}{m}\sum_{i=1}^m \norm{\bar{\vw}_i}^2 \leq \norm{\theta^\star}^2 \left(\frac{3d}{nr^2} + \frac{1}{mn}\right)
\end{align*}
where $\bar{\vw}_i  = \frac{1}{n}\sum_{j\in [n]}\vw_{i,j}$.
\end{lemma}
\begin{proof}
    Each $\bar{\vw}_i \sim \cN(0,\frac{\sigma^2}{n}I_d)$. Using MGF of $\chi^2$ variables, we have,
    \begin{align*}
        \Pr[\frac{1}{m}\sum_{i\in [m]}\norm{\bar{\vw}_i}^2 \geq t]\leq& \exp( - \lambda m t  - 0.5 md\log(1 - 2 \lambda \sigma^2/n),\quad \forall \lambda \in (0, 0.5n/\sigma^2)\\
        \leq & \exp( - 0.4mn t/\sigma^2  + 0.34 md)\\
        \Pr[\frac{1}{m}\sum_{i\in [m]}\norm{\bar{\vw}_i}^2 \geq \frac{1}{n}(3d\sigma^2 + \norm{\theta^\star}^2] \leq & \exp(-0.4\norm{\theta^\star}^2/\sigma^2)\leq \exp(-r^2)
    \end{align*}
    We set $\lambda = 0.4 n/\sigma^2$ and $t = \frac{1}{n}(3d\sigma^2 + \frac{\norm{\theta^\star}^2}{m})$. We finally use $\frac{\sigma^2}{\norm{\theta^\star}^2} \leq \frac{1}{r^2}$
\end{proof}

We assume that $\cE$ is the intersection of high probability events defined in Lemmas~\ref{lem:avg_noise},\ref{lem:noise_eigen},\ref{lem:noise_set} and \ref{lem:avg_noise_machine}.
Therefore, $\bP[\cE] > 1-\zeta_1$ where $\zeta_1 \triangleq \exp(-mn)+ \exp(-r^2) + \exp(-n)$.

    \subsection{Error of Estimates $\hat{\theta}^{(t)}$ and $\hat{\theta}_i^{(t)}$}
    The error in estimating $\theta^\star$ at any step is defined as sum of error of estimating cluster identities and the statistical error of $\vw$.
    \begin{align*}
        \hat{\theta}^{(t)} - \theta^\star = \frac{1}{mn}\sum_{i\in [m]}\sum_{j\in [n]} ( (\hat{z}_{i,j}^{(t)} - z_{i,j})\vy_{i,j} + (z_{i,j}\vy_{i,j} - \theta^\star))
    \end{align*}
    Now, using $\hat{z}_{i,j}^{(t)} - z_{i,j} = -2\I\{\hat{z}_{i,j}^{(t)} \neq z_{i,j}\}z_{i,j}$, we obtain

    \begin{align}\label{eq:error_cent_2_cluster}
        \hat{\theta}^{(t)} - \theta^\star & = \frac{1}{mn}\sum_{i=1}^m\sum_{j=1}^n \bigg(-2 \I\{\hat{z}_{i,j}^{(t)} \neq z_{i,j}\}(\theta^\star + \vw_{i,j}) + \vw_{i,j} \bigg) = - 2 A^{(t)} \theta^\star   - 2 R^{(t)} + \bar{\vw}
    \end{align}
    where we define $c_{i,j}^{(t)} = \I\{\hat{z}_{i,j}^{(t)} \neq z_{i,j}\}$ as the indicator of misclustered datapoints at time $t$ and $R^{(t)} = \frac{1}{mn}\sum_{i\in [m]} \sum_{j\in [n]} c_{i,j}\vw_{i,j}$ as the sum of noise of misclustered datapoints at time $t$.
Similarly, we can define the error in estimating the local cluster centers $\hat{\theta}_i^{(t)}$, 
\begin{align}\label{eq:error_cent_local_2_cluster}
    \hat{\theta}_i^{(t)}  - \theta^\star = \frac{1}{n}\sum_{j=1}^n \left( - 2 \I\{\hat{z}_{i,j}^{(t)}\neq z_{i,j} \}(\theta^\star + \vw_{i,j}) + \vw_{i,j}\right) = -2A_i^{(t)}\theta^\star - 2R_i^{(t)} + \bar{\vw}_i
\end{align}
where $A_i^{(t)} = \frac{1}{n}\sum_{j=1}^n \I\{\hat{z}_{i,j}^{(t)} \neq z_{i,j}\}, R_i^{(t)} = \frac{1}{n}\sum_{j=1}^n \I\{\hat{z}_{i,j}^{(t)} \neq z_{i,j}\}\vw_{i,j}$ and $\bar{\vw}_i  = \frac{1}{n}\sum_{j=1}^n \vw_{i,j}$.

From Eq~\eqref{eq:error_cent_2_cluster} and ~\eqref{eq:error_cent_local_2_cluster}, we have  $\theta_i^{(t)} - \theta^{(t)} = 2(A^{(t)} - A_i^{(t)})\theta^\star + 2(R^{(t)} - R_i^{(t)}) + \bar{\vw}_i - \bar{\vw}$.

Using Lemma~\ref{lem:noise_set}, we can provide bounds on $\norm{R^{(t)}}^2$ and $\frac{1}{m}\sum_{i=1}^m \norm{R_i^{(t)}}^2$ which will be used throughout our analysis.
\begin{equation}    
\begin{aligned}\label{eq:bounds_R_i_2_clus}
    \norm{R^(t)}^{2} \leq& \frac{\norm{\sum_{i\in [m], j\in [n]}\vw_{i,j}\I\{\hat{z}^{(t)}_{i,j} \neq z_{i,j}\}}^2}{(mn)^2} \leq \frac{(mn + 9d)\sigma^2}{mn} A^{(t)}= \frac{\norm{\theta^\star}^2 A^{(t)}}{r^2}\\
    \frac{1}{m}\sum_{i=1}^m\norm{R_i^{(t)}}^2 \leq& \frac{1}{m}\sum_{i=1}^m\frac{\norm{\sum_{j\in [m]}\vw_{i,j}\I\{\hat{z}_{i,j}^{(t)} \neq z_{i,j}\}}^2}{n^2} \leq \frac{(n + 9d)\sigma^2}{n}A^{(t)}\leq \norm{\theta^\star}^2A^{(t)}\left(\frac{1}{r^2} + \frac{9d}{nr^2}\right)
\end{aligned}
\end{equation}

\subsection{Error of cluster identity $\hat{z}_{i,j}^{(t)}$}

To estimate cluster identities $\hat{z}_{i,j}^{(t+1)}$, we use the local estimate of $\theta^\star$ on each machine $i\in [m]$ ($\hat{\theta}_{i}^{(t)}$ ) at time $t+1$. 
    \begin{align*}
        \hat{z}_{i,j}^{(t+1)}  = \argmin_{g \in \{-1,1\}}\norm{g\vy_{i,j} - \hat{\theta}_i^{(t)}}_2^2 = \argmax_{g\in \{-1,1\}}\ip{g\vy_{i,j}}{\hat{\theta}_i^{(t)}} = \argmax_{g\in \{-1,1\}}\ip{\theta^\star + \vw_{i,j}}{\hat{\theta}_i^{(t)}}
    \end{align*}
    
Therefore, $\I\{\hat{z}_{i,j}^{(t+1)}\neq z_{i,j}\} = \I\{\ip{\theta^\star + \vw_{i,j}}{\hat{\theta}_{i}^{(t+1)}} \leq 0\}$.
Our goal now is to bound $\ip{\theta^\star + \vw_{i,j}}{\hat{\theta}_i^{(t)}}$ in terms of the virtual iterates $\hat{\theta}^{(t)}$ to get a single step recursion on $mn$ datapoints resembling KMeans at center, along with a deviation term depending on difference of virtual iterates and local estimates.

We crucially consider two cases -- i) when $t$ is a multiple of $L$ and ii) when $t$ is not a multiple of $L$.
\subsubsection{When $t$ is a multiple of $L$}
In this case, note that $\hat{\theta}_i^{(t)} = \hat{\theta}^{(t)}$ due to aggregation. Therefore, the single step progress is exactly same as that of ~\cite[Eq~(11)]{lu_statistical_2016}. We restate their proof here as we will use some of these terms for the case when $t$ is not a multiple of $L$.

\begin{align*}
    \ip{\theta^\star + \vw_{i,j}}{ \hat{\theta}^{(t)}} &= \ip{\theta^\star + \vw_{i,j}}{ \hat{\theta}_{i}^{(t)}}\\
    &=  \ip{\theta^\star + \vw_{i,j}}{(1-2A^{(t)})\theta^\star - 2R^{(t)} +\bar{\vw}}+\ip{\vw_{i,j}}{\theta^\star} + \ip{\vw_{i,j}}{-2A^{(t)}\theta^\star + 2R^{(t)} - \bar{\vw}}\\ 
    &= (1-2A^{(t)})\norm{\theta^\star}^2 - \ip{2R^{(t)} - \bar{\vw}}{\theta^\star}+\ip{\vw_{i,j}}{-2A^{(t)}\theta^\star - 2R^{(t)} + \bar{\vw}}\\
    &\geq (1-2A^{(t)})\norm{\theta^\star}^2 - 2\norm{R^{(t)}}\norm{\theta^\star} +\ip{\bar{\vw}}{\theta^\star}+\ip{\vw_{i,j}}{-2A^{(t)}\theta^\star - 2R^{(t)} + \bar{\vw}}\\
    &\geq \left(1-2A^{(t)} - \frac{2\sqrt{2A^{(t)}}}{r}-\frac{1}{\sqrt{mn}}\right)\norm{\theta^\star}^2  + \ip{\vw_{i,j}}{\theta^\star} +\ip{\vw_{i,j}}{-2A^{(t)}\theta^\star - 2R^{(t)} + \bar{\vw}}
\end{align*}
Note that $\ip{R^{(t)}}{\theta^\star}\leq \norm{R^{(t)}}\norm{\theta^\star}\leq \frac{2\sqrt{2A^{(t)}}}{r}\norm{\theta^\star}^2$ from Lemma~\ref{lem:noise_set}. Further, from Lemma~\ref{lem:avg_noise}, $\ip{\bar{\vw}}{\theta^\star}\leq \frac{\norm{\theta^\star}^2}{\sqrt{mn}}$. These bounds are obtained conditioned on the event $\cE$. Following the same proof, by using the inequalities $\I\{a+b\leq0\}\leq \I\{a\leq 0\}+\I\{b\leq 0\}$, followed by $\I\{b\leq -c\}\leq \frac{b^2}{c^2}$ which hold for all $a,b\in \R$ and $c>0$, we obtain the same bound as ~\cite[Eq~11]{lu_statistical_2016}. 
\begin{align*}
\I\{\hat{z}_{i,j}^{(t+1)} \neq z_{i,j}\} \leq& \I\{\beta_0\norm{\theta^\star}^2 + \ip{\vw_{i,j}}{\theta^\star} +  \ip{\vw_{i,j}}{-2A^{(t)}\theta^\star - 2R^{(t)} +\bar{\vw}} + \frac{3.12}{r}\norm{\theta^\star}^2 \leq 0\}\\
A^{(t+1)} \leq& \underset{I_1}{\underbrace{\frac{1}{mn}\sum_{i\in [m]}\sum_{j\in [n]}\I\{\beta_0\norm{\theta^\star}^2 \leq -\ip{\vw_{i,j}}{\theta^\star}\}}} + \underset{I_2}{\underbrace{\frac{1}{mn}\sum_{i\in [m]}\sum_{j\in [n]}\frac{\ip{\vw_{i,j}}{-2A^{(t)}\theta^\star - 2R^{(t)} +\bar{\vw}}^2}{\delta^2 \norm{\theta^\star}^4}}}\\
\end{align*}
Here, $\beta_0 = 1 - 2A^{(t)}  - \frac{5.12}{r} - \frac{1}{\sqrt{mn}}\leq 1 - 2A^{(t)} - \frac{3.12}{r} - \frac{2\sqrt{2A^{(t)}}}{r}-\frac{1}{\sqrt{mn}}$. The terms $I_1$ and $I_2$ are same as that in ~\citep[Section~7.2]{lu_statistical_2016} but here on $mn$ datapoints. We use the upper bounds on these terms from ~\cite{lu_statistical_2016}.

\paragraph{Upper Bound on $I_1$ ~\cite{lu_statistical_2016}}
With probability $1 -\zeta_2$ where $\zeta_2 = \zeta_1 + (mn)^{-3}$, we have,
\begin{align*}
    I_1 \leq \exp\bigg(-\frac{\gamma_{A^{(t)}}^2 \norm{\theta^\star}^2}{2\sigma^2}
    \bigg) + \sqrt{\frac{4\log(mn/2)}{mn}} 
\end{align*}
where $\gamma_{a} = 1- 2A^{(t)} - \frac{5.12}{r} -\frac{1}{\sqrt{mn}}$.

\paragraph{Upper Bound on $I_2$~\cite{lu_statistical_2016}}
Using Lemma~\ref{lem:noise_eigen}, we obtain
\begin{align*}
    I_2 \leq  A^{(t)}\bigg( \frac{8}{r^2} + A^{(t)}\bigg) + \frac{1}{r^2} + \frac{1}{mn}
\end{align*}

Combining these two bounds, if $A^{(t)}\leq \frac{1}{2} - \frac{2.56 + \sqrt{\log r}}{r} -\frac{1}{\sqrt{mn}}$, we obtain the single step-progress as the following,
\begin{align}\label{eq:1step_t_div_L_2_clus}
    A^{(t+1)}\leq A^{(t)}\left(A^{(t)} + \frac{8}{r^2}\right) + \frac{2}{r^2} + \sqrt{4\frac{\log(mn)}{mn}}
\end{align}
We use the fact that $\gamma_{A^{(t)}}^2 \geq \frac{2\log r}{r}$. From Assumption~\ref{assumption:local_dev}, the required condition is satisfied for $A^{(0)}$. After describing the single step progress when $t$ is not divisible by $L$, we will show via induction that this is indeed satisfied for all $t$.

Further, assuming that the required condition on $A^{(t)}$ is true, the value of $A^{(t+1)}$ has the following bound.
\begin{align}
    A^{(t+1)}\leq \frac{1}{2}A^{(t)} + \sqrt{4\frac{\log(mn)}{mn}} + \frac{2}{r^2} = \frac{1}{2}A^{(t)} + \Phi_1 
\end{align}
where the constant $\Phi_1 = \mathcal{O}\left(\frac{1}{r^2} + \sqrt{\frac{\log(mn)}{mn}}\right)$. Note that these conditions require $A^{(t)} + \frac{8}{r^2} \leq \frac{1}{2}$ which is possible for $r=\Omega(1)$ and $mn=\Omega(1)$. Additionally, under the same conditions, we obtain $A^{(t+1)}\leq A^{(t)}$ and $A^{(t+1)} \leq \frac{1}{4} - \epsilon'$ for some constant $\epsilon'$ depending on $mn$ and $r$. 

\subsubsection{When $t$ is not a multiple of $L$}
We now analyze the more complicated case of $t$ not being a multiple of $L$. Then, $\hat{\theta}_i^{(t)} \neq \hat{\theta}^{(t)}$, so we need to account for the difference $\hat{\theta}_i^{(t)} - \hat{\theta}^{(t)}$. Bounding this term for a single machine $i\in [m]$ requires us to analyze each machine's misclustering individually, however, if we use the virtual iterate method from ~\cite{stich_sparsified_2018,karimireddy_scaffold_2020}, we can accommodate an average measure of the differences between machines and central server. 

We bound single step progress by obtaining a lower bound on $\ip{\theta^\star + \vw_{i,j}}{\hat{\theta}_i^{(t)}}$. We first try to separate it into terms of $\hat{\theta}^{(t)}$ that we can bound by using ~\cite{lu_statistical_2016} and additional terms of $\hat{\theta}_i^{(t)} - \hat{\theta}^{(t)}$.

\begin{align}
    \ip{\theta^\star + \vw_{i,j}}{ \hat{\theta}_{i}^{(t)}} &= \ip{\theta^\star + \vw_{i,j}}{\hat{\theta}^{(t)}} + \ip{\theta^\star + \vw_{i,j}}{\hat{\theta}_i^{(t)} - \hat{\theta}^{(t)}}\nonumber\\
    & =  \ip{\theta^\star + \vw_{i,j}}{(1-2A^{(t)})\theta^\star - 2R^{(t)} +\bar{\vw}} + \ip{\theta^\star}{\hat{\theta}_i^{(t)} - \hat{\theta}^{(t)}} + \ip{\vw_{i,j}}{\hat{\theta}_i^{(t)} - \hat{\theta}^{(t)}}\nonumber\\
    &\geq (1-2A^{(t)} - \delta_1)\norm{\theta^\star}^2 - \ip{2R^{(t)} - \bar{\vw}}{\theta^\star} + \ip{\vw_{i,j}}{(1-2A^{(t)})\theta^\star - 2R^{(t)} +\bar{\vw}} \nonumber\\
    &  + \ip{\vw_{i,j}}{\hat{\theta}_i^{(t)} - \hat{\theta}^{(t)}} + (\delta_1 + 2(A^{(t)} - A_i^{(t)}))\norm{\theta^\star}^2  + \ip{2(R^{(t)} - R_i^{(t)}) + \bar{w}_i - \bar{w}}{\theta^\star}  \label{eq:misclus_2_center}
\end{align}
We can now write down an upper bound on misclustering of a single datapoint.

\begin{equation}
\begin{aligned}
\I\{\hat{z}_{i,j}^{(t+1)} \neq z_{i,j}\} \leq& \I\{(\beta_1' + \beta_2')\norm{\theta^\star}^2 + \ip{\vw_{i,j}}{\theta^\star} +  \ip{\vw_{i,j}}{-2A^{(t)}\theta^\star - 2R^{(t)} +\bar{\vw}}\nonumber\\
& +(\delta_1 + 2(A^{(t)} - A_i^{(t)}))\norm{\theta^\star}^2  + \ip{2(R^{(t)} - R_i^{(t)}) + \bar{w}_i - \bar{w}}{\theta^\star}\\
&+ \ip{\vw_{i,j}}{\hat{\theta}_i^{(t)} - \hat{\theta}^{(t)}} \leq 0\}    
\end{aligned}
\end{equation}
where $\beta_1' = 1 - \delta_1 - 2A^{(t)}  - \frac{2}{r} - \frac{1}{\sqrt{mn}}$ and $\beta_2' = \delta_1 + 2A^{(t)} - 2A_i^{(t)} \geq \frac{\sqrt{\log (r)}}{r}$.

We use the inequality $a,b \in \R$ and $c>0$, $\I\{a + b \leq c\}\leq \I\{a\leq c\} + \I\{b\leq -c\} \leq \I\{a\leq c\} + \frac{b^2}{c^2}$ three times to bound the remaining three terms of the above equation. Note that these are additional terms added due to local steps, so removing them from inside the indicator comes at the cost of constant terms $\delta_1, \delta_2$ and $\delta_3$ in the global misclustering $A^{(t)}$.
\begin{align*}
    \I\{\hat{z}_{i,j}^{(t+1)} \neq z_{i,j}\} \leq& \I\{\beta_1\norm{\theta^\star}^2 \leq - \ip{\vw_{i,j}}{\theta^\star}\} +  I\{\ip{\vw_{i,j}}{-2A^{(t)}\theta^\star - 2R^{(t)} +\bar{\vw}} \leq -\delta \norm{\theta^\star}^2\}\\
    &+ I\{\beta_2 \norm{\theta^\star}^2 \leq -\ip{\vw_{i,j}}{\theta^\star}\}  + \I\{\ip{2(R^{(t)} - R_i^{(t)}) + \bar{w}_i - \bar{w}}{\theta^\star} \leq - \delta_2 \norm{\theta^\star}^2\}\\
 & + \I\{\ip{\vw_{i,j}}{\hat{\theta}_i^{(t)} - \hat{\theta}^{(t)}} \leq -\delta_3 \norm{\theta^\star}\}
\end{align*}
where $\beta_1 = 2 \beta_1' - \delta - \delta_2 - \delta_3 = 2 - 2\delta_1 - 4A_i^{(t)} - \frac{7.12}{r}  - \delta_2  -\delta_3$, $\delta = \frac{3.12}{r}$, and $\beta_2 = 2\beta_2'$  $\delta_2,\delta_3$ are constants whose value is set to make their corresponding terms small. Summing the misclustering over all $mn$ datapoints, we obtain the following.
\begin{equation}
\begin{aligned}\label{eq:2_cent_decomp}
    A^{(t+1)} \leq& \underset{I_1'}{\underbrace{\frac{1}{mn}\sum_{i\in [m]}\sum_{j\in [n]}\I\{\beta_1\norm{\theta^\star}^2 \leq -\ip{\vw_{i,j}}{\theta^\star}\}}} + \underset{I_2}{\underbrace{\frac{1}{mn}\sum_{i\in [m]}\sum_{j\in [n]}\frac{\ip{\vw_{i,j}}{-2A^{(t)}\theta^\star - 2R^{(t)} +\bar{\vw}}^2}{\delta^2 \norm{\theta^\star}^4}}} \\
    &+ \underset{I_3}{\underbrace{\frac{1}{mn}\sum_{i\in [m]}\sum_{j\in [n]}I\{\beta_2 \norm{\theta^\star}^2 \leq -\ip{\vw_{i,j}}{\theta^\star}\}}}  + \underset{I_4}{\underbrace{\frac{1}{m}\sum_{i\in [m]}\frac{\ip{2(R^{(t)} - R_i^{(t)}) + \bar{w}_i - \bar{w}}{\theta^\star}^2}{\delta_2^2\norm{\theta^\star}^4}}}\\
 &  + \underset{I_5}{\underbrace{\frac{1}{mn}\sum_{i\in [m]}\sum_{j\in [n]} \frac{\ip{\vw_{i,j}}{\hat{\theta}_i^{(t)} - \hat{\theta}^{(t)}}^2}{\delta_3^2\norm{\theta^\star}^4}}} 
\end{aligned}
\end{equation}
Here, the term $I_2$ is exactly same as the term $I_2$ for the case when $t$ is divisible by $L$, and the term $I_1'$ differs from the corresponding term $I_1$ in $\beta_1$ instead of $\beta_0$. Therefore, we use their upper bounds and add an additional $(mn)^{-3}$ to the probability of 

\paragraph{Upper Bound on $I_1'$ ~\cite{lu_statistical_2016}}
With probability $1-\zeta_3$, where $\zeta_3 = \zeta_2 + (mn)^{-3}$, we obtain
\begin{align*}
    I_1 \leq \exp\bigg(-\frac{(\gamma_{A^{(t)}}')^2 \norm{\theta^\star}^2}{2\sigma^2}
    \bigg) + \sqrt{\frac{4\log(mn/2)}{mn}} 
\end{align*}
where $\gamma_{a}' = 2 - 2\delta_1 - 4A^{(t)} - \frac{7.12}{r}  - \delta_2  -\delta_3$.

Combining the bounds on $I_1$ and $I_2$, we obtain the following bound, which closely resembles the progress when $t$ is divisible by $L$.
\begin{align*}
    I_1 + I_2 \leq A^{(t)}\bigg( \frac{8}{r^2} + A^{(t)}\bigg) + \frac{1}{r^2} + \frac{1}{mn}  + \exp\bigg(-\frac{(\gamma_{A^{(t)}}')^2 \norm{\theta^\star}^2}{2\sigma^2}\bigg) + \sqrt{\frac{4\log(mn/2)}{mn}}
\end{align*}

The additional penalty due to local steps comes inside terms $I_3,I_4$ and $I_5$. Note that the terms $I_3$ and $I_4$ offer only constant bounds on the error terms

\paragraph{Upper Bound on $I_3$}

We bound $I_3$ similar to the upper bound of $I_1$.
We first ensure that the condition in Assumption~\ref{assumption:local_dev} holds throughout all iterations $t\geq 0$.
\begin{lemma}[Local Deviation]\label{lem:local_dev_all}
If Assumption~\ref{assumption:local_dev} is satisfied, then for all $t\geq 1$, $\exists \delta_1 \in (0, \frac{1}{2} - \epsilon)$ for some constant $\epsilon>0$ such that,
\begin{align*}
\max_{i,i'\in [m]}\abs{A_i^{(t)} - A_{i'}^{(t)}} \leq \frac{m}{m-1}\frac{\delta_1}{2} - \frac{\sqrt{\log r}}{2r}, \quad A_i^{(t)} \leq A^{(t)} + \frac{\delta_1}{2} -\frac{\sqrt{\log r}}{2r}
\end{align*}
\end{lemma}
We provide a proof of this Lemma in Appendix~\ref{sec:local_dev_all_proof}. Using this Lemma, $\beta_2 \geq \frac{2\sqrt{\log r}}{r}$.

Note that  $mn$ indicator variables,  $T_{i,j} = I\{\beta_2 \norm{\theta^\star}^2 \leq -\ip{\vw_{i,j}}{\theta^\star}\}$, which concentrate via Hoeffding's inequality. Further, as $\beta_2$ is constant, we can upper bound $\E[T_{i,j}]$. First, applying Hoeffding's inequality,
\begin{align*}
    \Pr\left[\frac{1}{mn}\sum_{i\in [m]}\sum_{j\in [n]}(T_{i,j} - \E[T_{i,j}])\right] \geq \sqrt{\frac{3\log(mn/2)}{mn}}] \leq (mn)^{-3}
\end{align*}
Additionally, as $\ip{\vw_{i,j}}{\theta^\star}\sim \cN(0,\sigma^2\norm{\theta^\star}^2)$, we have, 
\begin{align*}
    \E[T_{ij}] \leq \exp\left(-\beta_2^2\frac{\norm{\theta^\star}^2}{2\sigma^2}\right) \leq \exp\left(- 2\log r\right) \leq \frac{1}{r^2}, \,\forall i\in [m], j\in [n]
\end{align*}
Therefore, we obtain the following bound on $I_3$, with probability $1 - \zeta_4$, where $\zeta_4 = \zeta_3 + (mn)^{-3}$.
\begin{align}\label{eq:I_3_2_cluster}
    I_3 \leq \sqrt{\frac{3\log(mn/2)}{mn}} + \frac{1}{r^2} 
\end{align}

\paragraph{Upper Bound on $I_4$}
To bound $I_4$, we first use  Cauchy-Schwartz to change the inner product to norms, and Young's inequality ($\norm{a+b}^2 \leq 2\norm{a}^2 + 2\norm{b}^2$), to separate out the terms of $R_i^{(t)}$ and $\bar{\vw}_i$. Finally, we use $\Var(X) \leq \E[X^2]$, Eq~\eqref{eq:bounds_R_i_2_clus} with $A^{(t)}\leq 1$ and Lemma~\ref{lem:avg_noise_machine} to bound the terms of $\norm{R_i^{(t)}}^2$ and $\norm{\bar{\vw}_i}^2$. These steps are shown below.
\begin{align*}
   \frac{1}{m}\sum_{i\in [m]} \ip{2(R^{(t)} - R_i^{(t)}) + \bar{\vw}_i - \bar{\vw}}{\theta^\star}^2 \leq& \frac{1}{m}\sum_{i\in [m]}\norm{2(R^{(t)} - R_i^{(t)}) + \bar{\vw}_i - \bar{\vw}}^2\norm{\theta^\star}^2\\
    \leq& \frac{\norm{\theta^\star}^2}{m}\sum_{i\in [m]} (8\norm{R^{(t)} - R_i^{(t)}}^2 + 2\norm{\bar{\vw}_i - \bar{\vw}}^2)\\
    \leq& \frac{\norm{\theta^\star}^2}{m}\sum_{i\in [m]} (8\norm{R_i^{(t)}}^2 + 2\norm{\bar{\vw}_i}^2)\\
    \leq& \norm{\theta^\star}^2 \left(\frac{2}{mn} + \frac{6d}{nr^2} + \frac{8A^{(t)}}{r^2}\left(1 + \frac{9d}{n}\right)\right)
\end{align*}
\begin{align}\label{eq:I_4_2_cluster}
    I_4 \leq& \frac{1}{m}\sum_{i\in [m]}\frac{\ip{2(R^{(t)} - R_i^{(t)}) + \bar{\vw}_i - \bar{\vw}}{\theta^\star}^2}{\delta_2^2\norm{\theta^\star}^4} \leq \frac{1}{\delta_2^2} \left(\frac{2}{mn} + \frac{6d}{nr^2} + \frac{8A^{(t)}}{r^2}\left(1 + \frac{9d}{n}\right)\right)\nonumber\\
    \leq& \frac{8A^{(t)}}{r^2}\left(1 + \frac{9d}{n}\right) + \Phi_2
\end{align}
Here, the upper bound of $\Phi_2 \leq \cO{\frac{1}{mn} +\frac{d}{nr^2}}$ holds as long as both $n=\Omega(1), r=\Omega(\max\{1,\sqrt{\frac{d}{n}}\})$ and $\delta_2$ is a constant.

\paragraph{Upper bound on $I_5$}

Note that we need to separate out the terms of noise $\vw_{i,j}$ and $\hat{\theta}_i^{(t)} - \hat{\theta}^{(t)}$ to obtain a bound in terms of  $\Delta^{(t)}$. For this purpose, consider a single term in the summation.
\begin{align*}
    \ip{\vw_{i,j}}{\hat{\theta}_i^{(t)} - \hat{\theta}^{(t)}}^2 =  (\hat{\theta}_i^{(t)} - \hat{\theta}^{(t)})^\intercal \vw_{i,j}\vw_{i,j}^\intercal (\hat{\theta}_i^{(t)} - \hat{\theta}^{(t)}) \leq  \norm{\hat{\theta}_i^{(t)} - \hat{\theta}^{(t)}}^2 \lambda_{\max}(\sum_{j=1}^n \vw_{i,j}\vw_{i,j}^\top)
\end{align*}
From Lemma~\ref{lem:noise_eigen}, with probability $1 - \mathcal{O}(m\exp(-n))$, $\frac{1}{n\norm{\theta^\star}^2}\lambda_{\max}(\vw_{i,j}\vw_{i,j}^\top)\leq \frac{1.62 (n + 4d)\sigma^2}{n\norm{\theta^\star}^2}\leq \frac{9}{r^2}\max\{\frac{d}{n},1\}$. We absorb the term of $m$ in the high probability term with $\log m = \mathcal{O}(n)$. This implies the following bound on with probbility $1 - \zeta_5$ where $\zeta_5 = \zeta_4 + \exp(-n)$,
\begin{align}\label{eq:dev_second_term}
    I_5 \leq \frac{9}{\delta_3^2 r^2}\max\{\frac{d}{n},1\}\Delta^{(t)}
\end{align}
where $\Delta^{(t)} \equiv \frac{1}{m\norm{\theta^\star}^2}\sum_{i=1}^m \norm{\hat{\theta}_i^{(t)} - \hat{\theta}^{(t)}}^2$. Note that $\Delta^{(t)}$ is the variance of $\hat{\theta}_i^{(t)}$.

To bound $I_5$, we now need to bound $\Delta^{(t)}$. Existing analysis~\cite{karimireddy_scaffold_2020} of the virtual iterate method for supervised learning bound this via a recursion depending on all steps from $t$ until the last aggregation step, i.e., $\floor{\frac{t}{L}}\cdot L$. However, these methods have the advantage of tuning the step-size according to $L$ to prevent an exponential blow-up. In our case, we do not have advantage, as the term $\delta_3$ is the only parameter under our control. If the term $\Delta^{(t)}$ blows up exponentially in terms of $L$, we will need $\delta_3$ to be very large, thus imposing a strong upper bound on the misclustering $A^{(t)}$. To fix this, we unroll $\Delta^{(t)}$ only until the previous step, i.e., $(t-1)$. This allows both a fast decrease for the single-step progress of $A^{(t+1)}$ and not too large requirement on $\delta_3$.

\begin{lemma}[Bound on $\Delta^{(t)}$]\label{lem:dev_step}
For any step $t$ that is not divisible by $L$, if $A^{(t)} \leq A^{(0)}$ and $A^{(t-1)}\leq A^{(0)}$, then ,
\begin{align*}
    \Delta^{(t)}\leq 9A^{(t)} + 6A^{(t-1)} + \Phi_1 + \Phi_2
\end{align*}
\end{lemma}
The proof is provided in Appendix~\ref{sec:dev_step_proof}.
We set $\delta_3 = \frac{\sqrt{270}\max\{\frac{d}{n},1\}}{r}$, to obtain the following bound on $I_5$.
\begin{align*}
    I_5 \leq \frac{3}{10}A^{(t)} + \frac{1}{5}A^{(t-1)} + \Phi_1 + \Phi_2
\end{align*}
Note that we can absorb the constant terms inside $\Phi_1$ and $\Phi_2$

\paragraph{Completing the single step progress proof}
Now, to complete the proof for single step-progress, we plug in the terms of $I_1',I_2, I_3, I_4$ and $I_5$
\begin{align*}
    A^{(t+1)} \leq& I_1 + I_2 + I_3 + I_4 + I_5 \\
    \leq & A^{(t)}\bigg( \frac{16}{r^2} + \frac{72d}{nr^2}+ \frac{3}{10} + A^{(t)}\bigg) + \frac{1}{5}A^{(t-1)}  + \exp\bigg(-\frac{(\gamma_{A^{(t)}}')^2 \norm{\theta^\star}^2}{2\sigma^2}\bigg) + \Phi_1 + \Phi_2\\
    \leq & A^{(t)}\bigg( \frac{16}{r^2} + \frac{72d}{nr^2}+ \frac{3}{10} + A^{(t)}\bigg) + \frac{1}{5}A^{(t-1)}  + \Phi_1 + \Phi_2
\end{align*}
The last inequality holds if  $A^{(t)} \leq \frac{1}{2} - \frac{1}{\sqrt{mn}} - \frac{3.56 + 0.5\sqrt{\log(r)}}{r} - \frac{\delta_1}{2} - \frac{\delta_2}{4} - \frac{\delta_3}{4}$, and $A^{(t-1)}\leq A^{(0)}$ for all $t$ not divisible by $L$.
 Setting $\delta_1$ to be close to $\frac{1}{2} -\epsilon$, $\delta_2 = \epsilon/2$ and $\delta_3 = \frac{\sqrt{270\max\{\frac{d}{n},1\}}}{r}$, we require $A^{(t)} \leq \frac{1}{4}$. We can show this via induction. For the base step, note that $A^{(1)} \leq 0.5A^{(0)} \leq \frac{1}{4}$ by the single-step progress when $t$ is divisible by $L$, and $A^{(0)} \leq \frac{1}{2}$. 
 Then assuming the condition, $A^{(t)}\leq \frac{A^{(0)}}{2}\leq \frac{1}{4}$ and $A^{(t-1)} \leq \frac{1}{2}$ holds for all steps till $t$. If $t$ is divisible by $t$, then $A^{(t+1)}\leq A^{(t)}/2 \leq \frac{1}{4}$, and  $A^{(t-1)} \leq \frac{1}{4}< \frac{1}{2}$. If $t$ is not divisible by $L$, we use the upper bound on $A^{(t+1)}$ obtained here, which boils down to the following for large $r$, $mn$, and small $\sqrt{\frac{d}{n}}$.
 
 \begin{align}
     A^{(t+1)}&\leq \frac{11}{20}A^{(t)} + \frac{1}{5}A^{(t-1)} + \Phi_1 + \Phi_2\label{eq:single_step_rec}\\
     &\leq \frac{11}{80} + \frac{1}{10} + \Phi_1 + \Phi_2 = \frac{19}{80} + \Phi_1 + \Phi_2 <\frac{1}{4}\nonumber
 \end{align}

Therefore, for all $t\geq 1$, by induction, we have $A^{(t)} \leq \frac{1}{4}$. Note that this holds with probability $1-\zeta_5 = 1 - 3(mn)^{-3} - \exp(-r^2) - \exp(-n)$,This completes the proof.

\subsection{Final statistical error}
\label{sec:final_error_2_cluster}
To obtain the final error for the $2$ cluster case, we follow the proof technique of  ~\cite[Section~7.3]{lu_statistical_2016}. Their proof has two stages, the first $\log(mn)$ steps, where the misclustering decreases from $A^{(0)}$ to $\Phi_1$, then after the next $2\log(mn)$, the misclustering decreases from $\Phi_1$ to the value $\exp(-r^2)$. In our case, we have a different recursion for the two cases when $t$ is divisible by $L$ and $t$ is not divisible by $L$.
\begin{align*}
    A^{(t+1)}\leq& \frac{1}{2}A^{(t)} + \Phi_1 \hfill \tag{When $L$ divides $t$}\\
    A^{(t+1)}\leq& \frac{11}{20}A^{(t)} + \frac{1}{5}A^{(t-1)} + \Phi_1 + \Phi_2\hfill \tag{When $L$ doesn't divide $t$}
\end{align*}
We will show via induction that $A^{(t)}$ satisfied the following inequality,$\forall t\geq 0$
\begin{align}\label{eq:a_t_rec}
    A^{(t)} \leq \frac{1}{2^{t}} A^{(0)} + 4\Phi_1 + 4\Phi_2
\end{align}

To prove the base step, note that $t=0$ satisfies this inequality, due to the update when $t$ divides $L$. Assuming that the inequality is satisfied for all iterations till a specific step $t$, we can use the update equations to show it for the iteration $t+1$. If $L$ divides $t$, we have,
\begin{align*}
    A^{(t+1)} \leq \frac{1}{2^{t+1}}A^{(0)} + 2\Phi_2 + 3\Phi_1 \leq \frac{1}{2^{t+1}}A^{(0)} + 4\Phi_1 + 4\Phi_2 
\end{align*}

If $L$ doesn't divide $t$, we have, 
\begin{align*}
    A^{(t+1)} \leq \frac{1}{2^{t-1}}A^{(0)}\left(\frac{11}{40} + \frac{1}{5}\right) + \Phi_1\left(1 + \frac{4}{5} + \frac{11}{5}\right) + \Phi_2\left(1+\frac{4}{5} + \frac{11}{5}\right) = \frac{1}{2^{(t-1)}}\frac{19}{80}A^{(0)}  + 4\Phi_1 + 4\Phi_2
\end{align*}
Note that $\frac{19}{80}\leq \frac{1}{4}$. Therefore, $A^{(t+1)}$ always satisfies the required inequality, and by induction $A^{(t)}$ always satisfies Eq~\eqref{eq:a_t_rec}.
Note that Eq~\eqref{eq:a_t_rec} is the same recursion that we would have obtained for centralized Lloyd's on $mn$ points without any local steps, with an additional error term of $\Phi_2$.
Unrolling Eq~\eqref{eq:a_t_rec} till iterations $t\geq 2\ceil{\log(mn)} + 2L$, we have, 
\begin{align*}
    A^{(t)} \leq \frac{1}{(mn)^{2\log(2)}}\frac{1}{2^{2L}} + 4\Phi_1 + 4\Phi_2\leq \frac{1}{\sqrt{mn}} + 4\Phi_1 + 4\Phi_2 \leq 5\Phi_1 + 4\Phi_2
\end{align*}

Let $\Phi_3 = \mathcal{O}(\frac{1}{r^2} + \frac{d}{nr^2} + \sqrt{\log(mn)}{mn})$, then \begin{align*}
    A^{(t)}\leq \Phi_3, \quad \forall t\geq 2(\ceil{\log(mn)}+L)
\end{align*}
Now, we can use a better upper bound on single step misclustering to obtain the final error. Using ~\eqref{eq:misclus_2_center} with $\beta_1' = 1 - \frac{3\delta_1}{2} - 2A^{(t)} - \frac{2}{r} - \frac{1}{\sqrt{mn}}$ and $\beta_2' = \frac{3\delta_1}{2} + 2A^{(t)} - 2A_i^{(t)}$.
\begin{equation}
\begin{aligned}
\I\{\hat{z}_{i,j}^{(t+1)} \neq z_{i,j}\} \leq& \I\{(\beta_1' + \beta_2')\norm{\theta^\star}^2 + \ip{\vw_{i,j}}{\theta^\star} +  \ip{\vw_{i,j}}{-2A^{(t)}\theta^\star - 2R^{(t)} +\bar{\vw}}\nonumber\\
& \ip{2(R^{(t)} - R_i^{(t)}) + \bar{\vw}_i - \bar{\vw}}{\theta^\star}+ \ip{\vw_{i,j}}{\hat{\theta}_i^{(t)} - \hat{\theta}^{(t)}} \leq 0\}
\end{aligned}
\end{equation}
We again use the inequalities $\I\{a + b \leq c\}\leq \I\{a\leq c\} + \I\{b\leq -c\} \leq \I\{a\leq c\} + \frac{b^2}{c^2}$ for $a,b \in \R$ and $c>0$,  to separate the terms $I_4,I_5$ reducing $\beta_1'$ by the same coefficients $\delta_2, \delta_3$ as before. We also separate out terms $I_1''$ and $I_3'$ which differ from the corresponding $I_1$ and $I_3$ only in terms of $\beta_1=2\beta_1'-\frac{3.72}{r} - \delta_2 - \delta_3$ and $\beta_2=2\beta_2 = \delta_1 + 4(A^{(t)} - A_i^{(t)})$. Further, the term $I_2$ is now split into two terms $I_2'$ and $I_2''$ which exactly match the terms $J_2$ and $J_3$ in ~\cite[Section~7.3]{lu_statistical_2016}. 
This implies the following tighter bound on $A^{(t+1)}$ when $t \geq 2(\ceil{\log(mn)} + L)$.
\begin{align*}
    A^{(t+1)} \leq I_1'' + I_2'  + I_2'' + I_3' + I_4 + I_5
\end{align*}
The only new terms introduced are $I_2'$ and $I_2''$ which we define below for the sake of completeness.
\begin{align*}
    I_2' &= \frac{r^2}{8.1mn\norm{\theta^\star}^2}\sum_{i=1}^m\sum_{j=1}^n \ip{\vw_{i,j}}{2R^{(t)} - 2A^{(t)}\theta^\star}^2, \\
    I_2''&=\frac{1}{mn}\sum_{i=1}^m\sum_{j=1}^n \I\left\{\left(\frac{1}{2r} + \frac{2}{\sqrt{mn}}\right)\leq -\ip{\vw_{i,j}}{\bar{\vw}}\right\}
\end{align*}

Following the proof sketch from ~\cite{lu_statistical_2016}, we try to obtain a recursion for $\E[A^{(t+1)}]$.

\paragraph{Bound on $\E[I_2']$ and $\E[I_2'']$ from ~\cite{lu_statistical_2016}}
For these two terms, we utilize the bounds from ~\cite[Section~7.3]{lu_statistical_2016}.
To bound $I_2'$, we only need Lemma~\ref{lem:noise_set} to bound $\norm{R^{(t)}}^2$. Therefore, with probability $1 - \exp(-mn)$, we have,
\begin{align*}
I_2'\leq A^{(t)}\left(\frac{8}{r^2} + A^{(t)}\right)     
\end{align*}
The bound on $I_2'$ is obtained from Lemma~\ref{lem:noise_subset_subexp}.
\begin{align*}
    \E[I_2'] \leq \cO{\exp\left(-\frac{\norm{\theta^\star}^2}{2\sigma^2}\right)}
\end{align*}

\paragraph{Bound on $\E[I_1'']$ and $\E[I_3']$}

\begin{align*}
    \E[I_1''] = \Pr_{a\sim \cN(0, \norm{\theta^\star}^2\sigma^2}[a \geq \beta_1 \norm{\theta^\star}^2], \quad     \E[I_3'] = \Pr_{a\sim \cN(0, \norm{\theta^\star}^2\sigma^2}[a \geq \beta_2 \norm{\theta^\star}^2]
\end{align*}
For large $r, mn$ and small $\sqrt{\frac{d}{n}}$, $A^{(t)}$ is negligible compared to a constant term. Therefore,  we can bound both the terms $\beta_1$ and $\beta_2$ by positive quantities. 
\begin{align*}
 \beta_1 \geq& 2\beta_1' -\delta_2 \geq 2 - 3\delta_1 - \delta_2 \geq \frac{1}{2} -\delta_2 \geq \frac{1}{4}\\
 \beta_2 =& 2\beta_2' = \delta_1 + 2(\delta_1 + 2(A^{(t)} - A_i^{(t)})) \geq \frac{1}{5} 
\end{align*}
To bound $\beta_2$, we use Lemma~\ref{lem:local_dev_all}. 
Therefore, 
\begin{align*}
    \E[I_1'']\leq \cO{\exp\left(-\frac{\norm{\theta^\star}^2}{2\sigma^2}\right)}, \quad \E[I_3''] \leq \cO{\exp\left(-\frac{\norm{\theta^\star}^2}{2\sigma^2}\right)}
\end{align*}

Note that we changed these bounds to eliminate terms of $\sqrt{\frac{\log(mn)}{mn}}$ and $\frac{1}{r^2}$ by smaller terms of the order of $\exp(-r^2)$. This can ensure a tighter bound on the final statistical error.

\paragraph{Bound on the term $I_4$}
Note that for the single step progress, we use a bound of $\Phi_2$. From ~\eqref{eq:I_4_2_cluster}, we can obtain a bound for $I_4$ in terms of $A_i^{(t)}$ and $\Phi_2$. However, to avoid the $\frac{1}{mn}$ term in $\Phi_2$, we can compute $\frac{1}{m}\sum_{i=1}^m \E[\norm{\bar{\vw}_i}^2] = \frac{\sigma^2}{r^2}$ instead of a high probability bound. 
\begin{align*}
    \frac{1}{m}\sum_{i=1}^m \E[\norm{\bar{\vw}_i}^2] = \frac{\sigma^2 d}{n} \leq \frac{d\norm{\theta^\star}^2}{nr^2}
\end{align*}

With probability $1-\exp(-n)$, we can bound the term of $\norm{R_i^{(t)}}^2$.
\begin{align*}
    \frac{1}{m}\sum_{i=1}^m \norm{R_i^{(t)}}^2 \leq \frac{A^{(t)}\norm{\theta^\star}^2}{r^2} \left(1 + \frac{9d}{n}\right)
    \end{align*}

\paragraph{Bound on the term $I_5$}
We can use the bounds obtained for $I_5$ earlier, but we need to handle the terms containing $\Phi_1$ and $\Phi_2$ separately. Specifically, the terms $\norm{\bar{\vw}}^2$ and $\frac{1}{m}\sum_{i=1}^m \norm{\bar{\vw}_i}^2$ appear in the proof which have high probability bounds. We can use their expected values, however, we condition on a high probability event of probability $1-\exp(-n)$ to remove the terms of $\vw_{i,j}$ from the coefficients using $\delta_3 =\frac{\sqrt{150}}{r}$. Let this event be $\cE_1$. Then, 
\begin{align*}
 \E[\norm{\bar{\vw}}^2 | \cE_1] \leq \frac{\E[\norm{\bar{\vw}^2]}}{\Pr[\cE_1]}\leq 2\E[\norm{\bar{\vw}}^2] = \frac{2\sigma^2 d}{mn} \leq \frac{2\norm{\theta^\star}^2}{r^2}
\end{align*}
Note that $1 - \exp(-n) \geq \frac{1}{2}$. We can obtain a similar upper bound for $\E[\frac{1}{m}\sum_{i=1}^m \norm{\bar{\vw}_i}^2]\leq \frac{2d\norm{\theta^\star}^2}{nr^2}$.

Note that we do not need to resort to the analysis when $t$ is divisible by $L$ or not divisible by $L$ as $A^{(t)}$ is small enough. Let $\cE'$ be the high probability event for all the terms required for the bound of $A^{(t+1)}$. Note that $\Pr[\cE'] \geq 1 - \exp(-n) - \exp(-mn) - \exp(-r^2)$. Therefore, 
\begin{align*}
    \E[A^{(t+1)}] \leq \E[I_1''] + \E[I_2''] + \E[I_3'] + \E[I_2' + I_4 + I_5|\cE'] + \Pr[\cE'^\complement]
    \leq& \frac{11}{20}\E[A^{(t)}] + \frac{1}{5}\E[A^{(t-1)}] + \Phi_4 
\end{align*}

where $\Phi_4  = \mathcal{O}\left(\max\{1,\frac{d}{n}\}\frac{1}{r^2} + \exp(-n) + \exp(-r^2)\right)$.
Unrolling the recursion until $t > 2(\ceil{\log(mn)}+L)$ steps, we obtain,
\begin{align*}
    \E[A^{(t)}] \leq \frac{1}{2^{t-2(\ceil{\log(mn)} + L)}}\Phi_3 + 4\Phi_4 
\end{align*}
To prove this, consider the base case at $t = 2(\ceil{\log(mn)}+L)$, where $A^{(t)} \leq \Phi_3$. Then, assume this holds for all iterations until a certain iteration $t$. We use the single-step update to show that it holds for the next iteration.
\begin{align*}
    \E[A^{(t+1)}]\leq \frac{1}{2^{t-1-2(\ceil{\log(mn)}}}\left(\frac{19}{40}\right)\Phi_3 + 4\Phi_4
\end{align*}
Therefore, by induction, the recursion holds for all $t\geq 2(\ceil{\log(mn)} + L)$. Now, unrolling the recursion till $t = (\ceil{q/\log(2)} + 2)(\ceil{\log(mn)} + L)$ steps, we obtain, 
\begin{align*}
    \E[A^{(t)}] \leq \frac{1}{2^{\ceil{q/\log(2)}}(\ceil{\log(mn)} + L)} \leq \left(\frac{\exp(-L)}{(mnL)}\right)^{q} + 4\Phi_4
\end{align*}
This completes the proof by applying a Markov's inequality with probability $\zeta$.

\subsection{Proof of Lemma~\ref{lem:dev_step} (Bound on $\Delta^{(t)}$)}
\label{sec:dev_step_proof}
We first introduce the previous iterate $\hat{\theta}^{(t-1)}$ into the analysis.
\begin{align*}
    \frac{1}{m}\sum_{i\in [m]}\norm{\hat{\theta}_i^{(t)} - \hat{\theta}^{(t)}}^2  &= \frac{1}{m}\sum_{i\in [m]}\norm{\hat{\theta}_i^{(t)} - \hat{\theta}^{(t-1)}  - (\hat{\theta}^{(t)} - \hat{\theta}^{(t-1)})}^2 \leq \frac{1}{m}\sum_{i\in [m]}\norm{\hat{\theta}_i^{(t)} - \hat{\theta}^{(t-1)}}^2
\end{align*}
We use $\mathbb{V}ar(X) \leq \E[X^2]$.
Now, we decompose $\hat{\theta}_i^{(t)}$ in terms of $A_i^{(t)}, R_i^{(t)}$ and $\bar{\vw}_i$.

Throughout the proof, we use $\norm{\sum_{i=1}^k a_i}^2 \leq k\sum_{i=1}^k \norm{a_i}^2$, to decompose square terms.
\begin{align*}
    \frac{1}{m}\sum_{i\in [m]}\norm{\hat{\theta}_i^{(t)} - \hat{\theta}^{(t-1)}}^2 &= \frac{1}{m}\sum_{i\in [m]}\norm{-2A_i^{(t)}\theta^\star -2R_i^{(t)} +\bar{\vw}_i + \theta^\star- \hat{\theta}^{(t-1)}}^2\\
    &\leq 2\norm{\theta^\star- \hat{\theta}^{(t-1)}}^2 + \frac{2}{m}\sum_{i\in [m]}\norm{2A_i^{(t)}\theta^\star +2 R_i^{(t)} -\bar{\vw}_i}^2
\end{align*}
We bound each term of the last inequality individually.
Consider the first term. Note that this is a similar bound as $I_2$ in ~\cite{lu_statistical_2016}, therefore, we use it's corresponding value.
\begin{align*}
2\norm{\hat{\theta}^{(t-1)} - \theta^\star}^2 = 2\norm{-2A^{(t-1)} -2R^{(t-1)} +\bar{\vw}}^2 \leq 12\norm{\theta^\star}^2 \left(A^{(t-1)}\left(\frac{8}{r^2} + A^{(t-1)}\right) + \frac{1}{r^2} + \frac{1}{mn}\right)    
\end{align*}
Note that as $A^{(t-1)} + \frac{8}{r^2}\leq A^{(0)} + \frac{8}{r^2}\leq \frac{1}{2}$, we have the following bound for the first term,
\begin{align*}
    2\norm{\hat{\theta}^{(t-1)} - \theta^\star}^2 \leq 6A^{(t)} + \Phi_1
\end{align*}
The bound for the second term has a similar decomposition as the first, however, it is for local quantities and the iteration $(t)$. We use Eq~\eqref{eq:bounds_R_i_2_clus} and Lemma~\eqref{lem:avg_noise_machine} for the local bounds.
\begin{align*}
    \frac{2}{m}\sum_{i\in [m]}\norm{2A_i^{(t)}\theta^\star +2 R_i^{(t)} -\bar{\vw}_i}^2 \leq \frac{12\norm{\theta^\star}^2}{m} \sum_{i=1}^m\left(A_i^{(t)}\left(A_i^{(t)} + \frac{8}{r^2} + \frac{72d}{nr^2}\right) + \frac{1}{mn} + \frac{3d}{nr^2}\right)
\end{align*}
Now, we use $A_i^{(t)} \leq A^{(t)} + \frac{1}{4}$ from Lemma~\ref{lem:local_dev_all}, and $A^{(t)} + \frac{1}{4} + \frac{8}{r^2} + \frac{72d}{nr^2}\leq A^{(0)} + \frac{1}{4} \leq \frac{3}{4}$.

\begin{align*}
    \frac{2}{m}\sum_{i\in [m]}\norm{2A_i^{(t)}\theta^\star +2 R_i^{(t)} -\bar{\vw}_i}^2 \leq 9A^{(t)} + \Phi_2
\end{align*}

Therefore, the bound on $\Delta^{(t)}$ can be obtained by combining these two terms,
\begin{align*}
    \Delta^{(t)} \leq 9A^{(t)} + 6A^{(t-1)} + \Phi_1 + \Phi_2
\end{align*}

\subsection{Proof of Lemma~\ref{lem:local_dev_all}}
\label{sec:local_dev_all_proof}

We need to show that for any $t\geq 0 $ with probability atleast $1 - \exp(-n)$, 
\begin{align*}
    A_i^{(t)} \leq A_{i'}^{(t)} + \frac{m}{m-1}\left(\frac{\delta_1}{2} - \frac{\sqrt{\log r}}{2r}\right)
\end{align*}
As a consequence of this inequality, for all $t\geq 1$, with the same probability, we have, 
\begin{align*}
    A_i^{(t)} \leq A^{(t)} + \frac{\delta_1}{2} - \frac{\sqrt{\log r}}{r}
\end{align*}
Here, $\delta_1 < \frac{1}{2}-\epsilon$ for some constant $\epsilon$.

We will use induction to prove this inequality for all $t$.

\paragraph{Base Case $t=1$}
At time $t=1$, $A_i^{(1)} = \frac{1}{n}\sum_{j=1}^n\I\{\ip{\theta^\star + \vw_{i,j}}{\hat{\theta}^{(0)}}\leq 0\}$. Let $Q_{i,j} = \I\{\ip{\theta^\star}{\hat{\theta}^{(0)}}\leq -\ip{\vw_{i,j}}{\hat{\theta}^{(0)}}\}$. Since $Q_{i,j}$ are indicator random variables, by Hoeffding's inequality, we have, with probability $1 - 2m\exp(-2nC^2)$, $\forall i\in [m]$ 
\begin{align*}
    \abs{A_i^{(1)} - \mu} \leq C
\end{align*}
for any $C \leq 1$, where $\mu = \E[Q_{i,j}], \forall i\in [m],j\in [n]$. By setting $C = \frac{\delta_1}{6}$, and using $\log m = \mathcal{O}(n)$ and for large $r=\Omega(1)$, with probability $1 - \exp(-\frac{\delta_1^2 n}{18})$, $\forall i, i'\in [m], i\neq i'$, we have,
\begin{align*}
    \abs{A_i^{(1)} - A_{i'}^{(1)}}\leq 2\frac{\delta_1}{6} \leq \frac{\delta_1}{2} - \frac{\sqrt{\log r}}{2r}
\end{align*}
This proves the base case for large $m$.

\paragraph{Induction Step}
The induction step argument is more complicated as single-step progress in Lemma~\ref{lem:single_step} can be established only when this condition always holds. We assume that the single-step progress is satisfied, in turn satisfying $A^{(t)} \leq \frac{1}{4} - \epsilon$ for all iterations till iteration $t$. Further, we assume that  $\max_{i,i'\in [m], i\neq i'}\abs{A_i^{(t)} - A_{i'}^{(t)}} \leq \frac{\delta_1}{2} - \frac{\sqrt{\log r}}{2r}$ is also satisfied. As a consequence of this, we show that this condition is also satisfied for iteration $t+1$. By induction, both Lemma~\ref{lem:local_dev_all} and the single-step progress in Lemma~\ref{lem:single_step} is satisfied for all $t\geq 1$.

For iteration $t+1$, we compute the misclustering for a single machine, $A_i^{(t+1)}$. The decomposition is similar to that used for $A^{(t)}$ for the case when $t$ is not divisible by $L$ but with certain differences to ensure that we can separate out the difference $A_i^{(t+1)} - A_{i'}^{(t+1)}$.  A single datapoint $\vy_{i,j}$ is misclustered if $\ip{\theta^\star + \vw_{i,j}}{\hat{\theta}_i^{(t)}}\leq 0$. For any $i'\in [m], i'\neq i$, if we subtract the term $\ip{\theta^\star + \vw_{i,j}}{\hat{\theta}_{i'}^{(t)}}$, we can obtain the difference between $A_i^{(t+1)}$ and $A_{i'}^{(t+1)}$. 
\begin{align*}
   \ip{\theta^\star + \vw_{i,j}}{ \hat{\theta}_{i}^{(t)}} &= \ip{\theta^\star + \vw_{i',j}}{\hat{\theta}_{i'}^{(t)}} + \ip{\vw_{i,j}}{\hat{\theta}_i^{(t)}} - \ip{\vw_{i',j}}{\hat{\theta}_{i'}^{(t)}} + \ip{\theta^\star}{\hat{\theta}_i^{(t)} - \hat{\theta}_{i'}^{(t)}}\\
   &= \ip{\theta^\star + \vw_{i',j}}{\hat{\theta}_{i'}^{(t)}} + \ip{\vw_{i,j}}{\hat{\theta}_i^{(t)} - \hat{\theta}^{(t)}} - \ip{\vw_{i',j}}{\hat{\theta}_{i'}^{(t)} - \hat{\theta}^{(t)}} \\
   &\quad + \ip{\theta^\star}{\hat{\theta}_i^{(t)} - \hat{\theta}_{i'}^{(t)}} + \ip{\vw_{i,j} - \vw_{i',j}}{\hat{\theta}^{(t)}}
\end{align*}
Now, averaging the above inequality over all $j\in [n]$, we obtain the $A_i^{(t+1)}- A_{i'}^{(t+1)}$ using $\I\{a+ b \leq 0\} \leq \I\{a\leq 0\} + \I\{b\leq 0\}\forall a, b\in \R$ to separate the additional terms.
\begin{align*}
A_i^{(t+1)}  - A_{i'}^{(t+1)} \leq \underset{I_6}{\underbrace{\frac{1}{n}\sum_{j=1}^n \I\{\ip{\vw_{i,j}}{\hat{\theta}_i^{(t)} - \hat{\theta}^{(t)}} - \ip{\vw_{i',j}}{\hat{\theta}_{i'}^{(t)} - \hat{\theta}^{(t)}} + \ip{\theta^\star}{\hat{\theta}_i^{(t)} - \hat{\theta}_{i'}^{(t)}} + \ip{\vw_{i,j} - \vw_{i',j}}{\hat{\theta}^{(t)}}\leq 0\}}}
\end{align*}
The additional term, $I_6$ is bounded separately, with techniques similar to that used for single-step progress.
We bound the second term separately. Decomposing $\hat{\theta}^{(t)},\hat{\theta}_i^{(t)}$ and $\hat{\theta}_{i'}^{(t)}$, with  $\hat{\theta}_i^{(t)} - \hat{\theta}_{i'}^{(t)} = 2(A_{i'}^{(t)}  - A_i^{(t)})\theta^\star + 2(R_{i'}^{(t)} - R_i^{(t)}) + \bar{\vw}_i - \bar{\vw}_{i'}$, we obtain the following decomposition.
\begin{align*}
    &\ip{\vw_{i,j}}{\hat{\theta}_i^{(t)} - \hat{\theta}^{(t)}} - \ip{\vw_{i',j}}{\hat{\theta}_{i'}^{(t)} - \hat{\theta}^{(t)}} + \ip{\theta^\star}{\hat{\theta}_i^{(t)} - \hat{\theta}_{i'}^{(t)}} + \ip{\vw_{i,j} - \vw_{i',j}}{\hat{\theta}^{(t)}}\\
    &= \ip{\vw_{i,j}}{\hat{\theta}_i^{(t)} - \hat{\theta}^{(t)}} - \ip{\vw_{i',j}}{\hat{\theta}_{i'}^{(t)} - \hat{\theta}^{(t)}} + 2(A_{i'}^{(t)} - A_i^{(t)})\norm{\theta^\star}^2 \\
    &\quad +\ip{\theta^\star}{2(R_{i'}^{(t)} - R_i^{(t)}) + (\bar{\vw}_i - \bar{\vw}_{i'})} + \ip{\vw_{i,j} - \vw_{i',j}}{\theta^\star} + \ip{\vw_{i,j} - \vw_{i',j}}{-2A^{(t)}\theta^\star - 2R^{(t)} + \bar{\vw}}\\
\end{align*}
For some constants $\delta_1',\delta_2', \delta_3'>0$, we will use the inequality $\I\{a + b\leq 0\} \leq \I\{a \leq c\} + \frac{b^2}{c^2}$ for $c>0$ four times and average  over $j\in [n]$ to obtain a bound on $I_6$.

\begin{align*}
    I_6 \leq& \underset{I_7}{\underbrace{\frac{1}{n}\sum_{j=1}^n \I\{(2A_{i'}^{(t)} - 2A_{i}^{(t)} - \delta_1' - \delta_2' - 2\delta_3')\norm{\theta^\star}^2 \leq \ip{\vw_{i',j} - \vw_{i,j}}{\theta^\star}\}}} \\
    &+ \underset{I_8}{\underbrace{\frac{1}{n}\sum_{j=1}^n\frac{\ip{\vw_{i,j} - \vw_{i',j}}{-2A^{(t)}\theta^\star - 2R^{(t)} + \bar{\vw})}^2}{(\delta_1')^2\norm{\theta^\star}^4}}} \\
    & + \underset{I_9}{\underbrace{\frac{\norm{2(R_{i'}^{(t)} - R_i^{(t)}) + (\bar{\vw}_i - \bar{\vw}_{i'})}^2}{(\delta_2')^2\norm{\theta^\star}^2}}} + \underset{I_{10}}{\underbrace{\frac{1}{n}\sum_{j=1}^n \frac{\ip{\vw_{i',j}}{\hat{\theta}_{i'}^{(t)} - \hat{\theta}^{(t)}}^2}{(\delta_3')^2\norm{\theta^\star}^4} + \frac{1}{n}\sum_{j=1}^n \frac{\ip{\vw_{i,j}}{\hat{\theta}_{i}^{(t)} - \hat{\theta}^{(t)}}^2}{(\delta_3')^2\norm{\theta^\star}^4}}} 
\end{align*}
We bound the terms $I_7 - I_{10}$ individually. 

\textbf{Bound on $I_7$}
Using the argument of the induction step, thus, $2A_{i'}^{(t)} - 2A_{i}^{(t)} - \delta_1' - \delta_2' - 2\delta_3' \geq \delta_1  -\frac{\sqrt{\log r}}{r} - \delta_1' - \delta_2' - 2\delta_3'$. We need to choose the constants $\delta_1', \delta_2'$ and $\delta_3'$ such that $\delta_1 -\frac{\sqrt{\log r}}{r} -  \delta_1' - \delta_2' - 2\delta_3'\geq \frac{2\sqrt{2\log r}}{r}$. In this case, we can bound $I_7$ by a term  similar to $I_3$ with a summation over $n$ points instead of $mn$ and double the noise variance, as $\vw_{i',j} -\vw_{i,j}\sim \cN(0, 2\sigma^2)$. Let $Q_{i,i', j}' = \I\{\frac{2\sqrt{2\log r}}{r}\norm{\theta^\star}^2 \leq \ip{\vw_{i',j} - \vw_{i,j}}{\theta^\star}\}$. As this is sum of $n$ random variables by Hoeffding's inequality, and a union bound over all $\frac{m(m-1)}{2}$ pairs of $i,i'\in [m], i'\neq i$, we have with probability $1 - \frac{m(m-1)}{2}\exp(-\frac{n}{\delta_1^2}{128})$
\begin{align*}
    I_7 \leq \frac{1}{r^2} + \frac{\delta_1}{16} \leq \frac{\delta_1}{16} + \Phi_1
\end{align*}
Here, $\E[Q_{i,i',j}] \leq \exp(-\frac{8\log r}{r^2}\frac{\norm{\theta^\star}^2}{4\sigma^2}) \leq \frac{1}{r^2}$. Note that the high probability term is of the form $\exp(-n)$ as $\log m = \mathcal{O}(n)$.

\textbf{Bound on $I_8$}
To bound $I_8$, we use the same bound as $I_2$ from Lemma~\ref{lem:noise_eigen} with $n$ datapoints and noise variance $2\sigma^2$. Therefore, if we set $\delta_1' = \frac{3.12}{20}$, we have, with probability $1-\frac{m(m-1)}{2}\exp(-n)$,
\begin{align*}
I_8\leq \left(\frac{400}{r^2} + \frac{3600d}{nr^2}\right)\left(\frac{8A^{(t)}}{r^2} + (A^{(t)})^2 + \frac{1}{r^2} + \frac{1}{mn}\right)\leq \frac{50}{r^2}\left(1 + \frac{9d}{n}\right) \leq \Phi_1 + \Phi_2
\end{align*}
We use the fact that $A^{(t)} \leq \frac{1}{4}$ to bound this term and $\log m = \mathcal{O}(n)$ to change the high probability term to $\exp(-n)$.

\textbf{Bound on $I_9$}
To bound $I_9$, we first use $\norm{\sum_{i=1}^q a_i}^2 \leq q \sum_{i=1}^q \norm{a_i}^2$, which is a consequence of Cauchy-Schwartz. 
\begin{align*}
    I_9 \leq \frac{4}{\delta_2'^2}(4\norm{R_i^{(t)}}^2 + 4\norm{R_{i'}^{(t)}}^2 +\norm{\bar{\vw}_i}^2+\norm{\bar{\vw}_{i'}}^2)
\end{align*}

To bound each $\norm{R_i^{(t)}}^2$, from Lemma~\ref{lem:noise_set}, we have with probability $1 - m\exp(-n)$, 
\begin{align*}
    \norm{R_i^{(t)}}^2 \leq \frac{1}{r^2}\left(1 + \frac{9d}{n}\right),\quad \forall i\in [m]
\end{align*}
Further, using Lemma~\ref{lem:chi_square_conc} with a union bound over all machines $i\in [m]$, with $t=1$ and $a$ being a scalar set to $1$, with probability $1-m\exp(-n)$, we obtain,
\begin{align*}
    \norm{\bar{\vw}_i}^2 \leq \frac{\sigma^2 d}{n} + 2\sigma^2\sqrt{\frac{d}{n}} + 2\sigma^2 \leq \frac{\norm{\theta^\star}^2}{r^2}\left(\frac{d}{n}+ \sqrt{\frac{d}{n}} + 2\right)
\end{align*}
Therefore, if we set $\delta_2' = \frac{1}{20}$, we obtain, the folllwing bound on $I_9$.
\begin{align*}
    I_9 \leq \Phi_1 + \Phi_2
\end{align*}

\textbf{Bound on $I_{10}$}
Note that both terms inside $I_{10}$ will have same bounds. So, we will only bound the term corresponding to $i$, and double it for the bound on $I_{10}$. We first bound $\lambda_{\max}(\sum_{i=1}^n \vw_{i,j}\vw_{i,j})$ for all $i\in [m]$ using Lemma~\ref{lem:noise_set}. Therefore,
\begin{align*}
    \frac{1}{n}\sum_{j=1}^n \frac{\ip{\vw_{i,j}}{\hat{\theta}_{i}^{(t)} - \hat{\theta}^{(t)}}^2}{(\delta_3')^2\norm{\theta^\star}^4} \leq \frac{1.62}{r^2(\delta_3')^2}\left(1 + \frac{4d}{n}\right)\frac{\norm{\hat{\theta}_i^{(t)} - \hat{\theta}^{(t)}}^2}{\norm{\theta^\star}^2}
\end{align*}
Using the decomposition for $\hat{\theta}_i^{(t)}$ and $\hat{\theta}^{(t)}$, we can bound the remaining term.
\begin{align*}
    \norm{\hat{\theta}_i^{(t)} - \hat{\theta}^{(t)}}^2 =& \norm{2(A^{(t)} - A_i^{(t)} + 2(R^{(t)} - R_i^{(t)}) + \bar{\vw}_i -\bar{\vw}}^2\\
    \leq& 5 (4\norm{\theta^\star}^2 \norm{A_i^{(t)} - A^{(t)}}^2 + 4\norm{R_i^{(t)}}^2 + \norm{\bar{\vw}_i}^2 + \norm{\bar{\vw}}^2)
\end{align*}
We use the bounds for $R^{(t)}$ and $\bar{\vw}$ used in the single-step proof, and the bounds for $R_i^{(t)}$ and $\bar{\vw}_i$ used for the bound on $I_9$. By the induction hypothesis, $ -A^{(t)} \geq A_i^{(t)} - A^{(t)}\leq \frac{\delta_1}{2} - \frac{\sqrt{\log r}}{2r}$, therefore, $\norm{A_i^{(t)} - A^{(t)}}^2 \leq \max\{(A^{(t)})^2, \frac{\delta_1^2}{2} + \frac{\log r}{2r^2}\} \leq \frac{1}{8}$, as $A^{(t)} \leq \frac{1}{4}$ and $\delta_1 \leq \frac{1}{2}$.

Adding all the bounds and setting $\delta_3' = \frac{1}{20}$, we obtain,
\begin{align*}
    I_{10} \leq \Phi_1 + \Phi_2
\end{align*}

Therefore, 
\begin{align*}
    A_i^{(t+1)} -A_{i'}^{(t+1)}\leq I_7+  I_8 + I_9 + I_{10} \leq \Phi_1 + \Phi_2 
\end{align*}
We consume the additional constants inside $\Phi_1$ and $\Phi_2$. We need this term to be $\leq \frac{\delta_1}{2} - \frac{\sqrt{\log r}}{2r}$

Further, we need $\delta_1' + \delta_2' + \delta_3' = \frac{8.12}{20} \leq \delta_1 - \frac{(2\sqrt{2} + 1)\sqrt{\log r}}{r}$. We can always find a $\delta_1 < \frac{1}{2}-\epsilon$ for some constant $\epsilon$ such that both these inequalities are satisfied for large $r=\Omega(\max\{1, \sqrt{\frac{d}{n}}\}), n=\Omega(1)$ and $\log m  = \mathcal{O}(n)$. This proves the induction step, and therefore Lemma~\ref{lem:local_dev_all} holds for all $t\geq 1$. Note that the bounds on all terms adds $\exp(-n)$ to the high probability bound.

\section{Proof for Section~\ref{sec:k_cluster_summarized} and Appendix~\ref{sec:k_cluster}}
\label{sec:k_cluster_proof}

The proof idea is similar to the $2$-cluster case. The key difference in this case, is that for $2$-clusters $\hat{\theta}^{(t)} = (1-2A^{(t)})\theta^\star  - 2R^{(t)} + \bar{\vw}$, but here the error in cluster centers $\Lambda$ and the misclustering $G$ can only be bounded in terms of each other. We will establish different recursions for $\Lambda$ and $G$ when $t$ is divisible by $L$ and when it isn't. Further, we also show bounds on the local to global deviation in terms of $\Lambda$ and $G$. We first provide technical Lemmas that we will use throughout the proof.

\subsection{Technical Lemmas}
Several of these Lemmas have been borrowed from ~\cite[Appendix~A.1]{lu_statistical_2016} with extension to $mn$ point. We state them for the sake of completeness.

\begin{lemma}(Extension of ~\cite[Lemma~A.1]{lu_statistical_2016})
\label{lem:noise_set_subg}
With probability $ 1 - \exp(-0.3mn)$ for any set $S \subseteq [m] \times [n]$, where $W_S = \sum_{(i,j)\in S} \vw_{i,j}$, we have,
    \begin{align*}
    \norm{W_S} \leq \sigma\sqrt{3(mn + d) \abs{S}}
    \end{align*}
Further if $\log m = \mathcal{O}(n)$, for any set $S_i \subseteq [n]$, we have $W_{S_i} = \sum_{j\in S_i} \vw_{i,j}$, then with probability $1  - \exp(-0.2 n)$, we have,
\begin{align*}
    \norm{W_{S_i}} \leq \sigma \sqrt{3(n + d)\abs{S}}, \forall i \in [m], 
\end{align*}
\end{lemma}
\begin{lemma}(Extension of ~\cite[Lemma~A.2]{lu_statistical_2016})\label{lem:noise_eigen_sub}
With probability $1 - \exp(-0.5 mn)$,
\begin{align*}
    \lambda_{\max}(\sum_{i\in [m]}\sum_{j\in [n]}\vw_{i,j}\vw_{i,j}^\intercal) \leq 6\sigma^2 (mn + d) 
\end{align*}
\end{lemma}
\begin{lemma}(Extension of ~\cite[Lemma~A.3]{lu_statistical_2016})\label{lem:noise_subset_sub}
For fixed $(i,j) \in [m]\times [n]$ and $S\subseteq [m] \times [n], t>0,s>0$
\begin{align*}
    \Pr\left[\ip{\vw_{i,j}}{\frac{1}{\abs{S}}W_{S}} \geq \frac{3\sigma^2(t\sqrt{S} + d + \log(1/\delta)}{\abs{S}}\right] \leq \exp\left(-\min\left\{\frac{t^2}{4d},\frac{t}{4}\right\}\right) + \delta
\end{align*}
\end{lemma}
\begin{lemma}(Extension of ~\cite[Lemma~A.4]{lu_statistical_2016})\label{lem:noise_correct_set}
With probability $1 - (mn)^{-3}$, $\forall k \in [K]$ we have,
\begin{align*}
    \norm{W_{S_k^\star}} \leq 3\sigma \sqrt{(d + \log(mn))\nu_k^\star}
\end{align*}
\end{lemma}
\begin{lemma}(Extension of ~\cite[Lemma~A.5]{lu_statistical_2016})\label{lem:indicator_inner_prod}
    For fixed $\theta_1, \theta_2, \ldots, \theta_k \in \R^d$, and any $a >0$, with probability $1 - (mn)^{-3}$, we have,
    \begin{align*}
        \sum_{(i,j)\in S_k^\star} \I\{a\norm{\theta_h - \theta_k}^2 \leq \ip{\vw_{i,j}}{\theta_h - \theta_k}\} \leq \nu_k^\star \exp\left( - \frac{a^2\Gamma^2}{2\sigma^2}\right) + \sqrt{5\nu_k^\star \log(mn)}
    \end{align*}
    \end{lemma}

We condition on the event that Lemmas~\ref{lem:noise_set_subg},\ref{lem:noise_eigen_sub}, \ref{lem:noise_correct_set} and \ref{lem:indicator_inner_prod} hold. Let this event be $\cE$. Then, $\Pr[\cE^\complement] \leq \zeta_1 \triangleq 3(mn)^{-3} +\exp(-n)$.
For the remainder of this proof, we will condition on the event $\cE$.

\subsection{Error of Centers}
In this section, we establish an upper bound on $\Lambda^{(t+1)}$ following ~\cite[Lemma~A.6]{lu_statistical_2016}.  Consider a cluster $k\in [K]$. Let $\bar{Y}_B = \frac{1}{\abs{B}} \sum_{(i,j)\in B} \vy_{i,j}$ for any set $B \subseteq [m] \times [n]$.

\begin{align}\label{eq:err_centers_k_mid}
    \hat{\theta}_k^{(t)} - \theta_k = \frac{1}{\nu_k^{(t)}}W_{S_{kk}^{(t)}} + \sum_{h\neq k , h\in [K]}\frac{\nu_{hk}^{(t)}}{\nu_k^{(t)}} (\bar{Y}_{S_{hk}^{(t)}} - \theta_k) 
\end{align}
Further, for local cluster center $\hat{\theta}_{k,i}^{(t)}$ a similar expression holds.
\begin{align}\label{eq:err_centers_k_mid_local}
    \hat{\theta}_{k,i}^{(t)} - \theta_k = \frac{1}{\nu_{k,i}^{(t)}}W_{S_{kk,i}^{(t)}} + \sum_{h\neq k , h\in [K]}\frac{\nu_{hk,i}^{(t)}}{\nu_{k,i}^{(t)}} (\bar{Y}_{S_{hk,i}^{(t)}} - \theta_k) 
\end{align}
Using these equations, we can establish an equivalent of ~\cite[Lemma~A.6]{lu_statistical_2016} for our two cases -- i) when $t$ divides $L$ and when $t$ doesn't divide $L$. Note that we can prove one part of ~\cite[Lemma~A.6]{lu_statistical_2016}, for all $t\geq 0$.
\begin{lemma}\label{lem:err_centers_k_all_t}
For all $t\geq 0$, we have, 
\begin{align*}
    \Lambda^{(t)} \leq \frac{3}{r_K} + \lambda G^{(t)}
\end{align*}
\end{lemma}
\begin{proof}
For any cluster $k\in [K]$, at any time $t\geq 0$, we have, 
\begin{align*}
    \hat{\theta}_k^{(t)} - \theta_k &= \frac{1}{\nu_k^{(t)}}W_{S_k^{(t)}} + \sum_{h\neq k , h\in [K]}\frac{\nu_{hk}^{(t)}}{\nu_k^{(t)}} (\theta_h -\theta_k)\leq\norm{\frac{1}{\nu_k^{(t)}}W_{S_k^{(t)}}} + \norm{\sum_{h\neq k , h\in [K]}\frac{\nu_{hk}^{(t)}}{\nu_k^{(t)}} (\theta_h -\theta_k)}\\
    &\leq\norm{\frac{1}{\nu_k^{(t)}}W_{S_k^{(t)}}} + \sum_{h\neq k , h\in [K]}\frac{\nu_{hk}^{(t)}}{\nu_k^{(t)}} \norm{\theta_h -\theta_k)}\leq\frac{3\Gamma}{r} + G^{(t)} \lambda \Gamma
\end{align*}
We use triangle inequality for the first and second inequalities and plug in the expressions of $\lambda$ and $r$ from Lemma~\ref{lem:noise_set_subg} for the last inequality.
\end{proof}

Note that Assumption~\ref{assumption:local_dev_k} implies that $G^{(t)}\leq \frac{1}{2}$. Further, the condition on $G^{(t)}$ in Assumption~\ref{assumption:local_dev_k}  also implies that $\Lambda^{(0)} \leq \frac{1}{2} - \frac{6}{\sqrt{r_K}} + \frac{3}{\sqrt{r_K}}$, so $\Lambda^{(0)} \leq \frac{1}{2} - \frac{3}{\sqrt{r_K}}$. This condition is also true if the initialization on $\Lambda^{(0)}$ is satisfied in Assumption~\ref{assumption:local_dev_k}. 

The above Lemma does not incorporate a Lloyd's update, only the expression for $\hat{\theta}_k^{(t)}$. Therefore, it is the same if $t$ divides $L$ or it doesn't. To establish progress in error of centers, we need a bound on $\Lambda^{(t+1)}$, which requires us to check if the Lloyd's update was made after at an aggregation step, i.e., $L$ divides $t$ or at a local step, i.e., when $L$ doesn't divide $t$. We provide bounds for these two cases separately.

\paragraph{When $L$ divides $t$}
In this case, the following equivalent of ~\cite[Lemma~A.6]{lu_statistical_2016} holds, as each update step is performed on the aggregated global models $\{\hat{\theta}_k^{(t)}\}_{k\in [K]}$. We state it here for the sake of completeness.
\begin{lemma}[Error of Centers]\label{lem:err_centers_k_L}
If $G^{(t)} \leq \frac{1}{2}$, with probability $1 - (mn)^{-3} - \mathcal{O}(\exp(-r_K^2))$, 
\begin{align*}
    \Lambda^{(t+1)} \leq& \frac{3}{r_K} + \frac{3}{r_K}\sqrt{K G^{(t+1)}} + 2G^{(t+1)}\Lambda^{(t)}
\end{align*}
\end{lemma}

\paragraph{When $L$ doesn't divide $t$.}
Under this condition, we need an additional condition on the difference between the local error of centers $\Lambda_i^{(t)}$ and $\Lambda^{(t)}$.
\begin{lemma}[Error of Local Misclustering]\label{lem:local_dev_all_k}
     $\exists \delta_1 \in (0, \frac{1}{10}-\epsilon'')$ for some small constant $\epsilon''>0$, such that, $\forall t\geq 1$
    \begin{align*}
        G_i^{(t)} \leq \frac{1}{2}, \quad \Lambda_i^{(t)} \leq \Lambda^{(t)} + \frac{\delta_1}{2},\quad \forall i \in [m].
    \end{align*}
\end{lemma}
The above lemma resembles Lemma~\ref{lem:local_dev_all} for the $2$-cluster case, and imposes the local error of centers to be atmost a constant more than the global error of centers. We provide a proof for this Lemma in Appendix~\ref{sec:local_dev_all_k_proof}. Using this Lemma, we can prove the following equivalent of Lemma~\ref{lem:err_centers_k_L} for the case when $L$ doesn't divide $t$.

\begin{lemma}[Error of Centers]\label{lem:err_centers_k}
If $G^{(t+1)} \leq \frac{1}{2}$, with probability $1 - (mn)^{-3} - \cO{\exp(-r_K^2)}$
    \begin{align*}
        \Lambda^{(t+1)} \leq \frac{3}{r_K} + \frac{4}{r_K}\sqrt{\frac{K}{\alpha}}\left(\sqrt{\frac{d}{n}} + 1\right)   + G^{(t+1)}(2\Lambda^{(t)} + \delta_1)
    \end{align*}
\end{lemma}
\begin{proof}
Note that the second term in upper bound follows directly from the proof of ~\cite{lu_statistical_2016}. In the first term's bounds, we apply the Lloyd's step directly and therefore have to deal with deviation term. 
Consider a cluster $k\in [K]$. We bound the two terms of Eq~\eqref{eq:err_centers_k_mid} separately.

For the first term in ~\eqref{eq:err_centers_k_mid}, we use the identity  $W_{S_{kk}^{(t+1)}} = W_{S_k^\star} - \sum_{h\neq k, h\in [K]} W_{S_{kh}^{(t+1)}}$, which expresses $S_{kk}^{(t+1)}$ as the difference of noise due to the all the correct points and the noise due to false negatives.  From Lemmas ~\ref{lem:noise_subset_sub} and ~\ref{lem:noise_correct_set}, and triangle inequality, we obtain,
\begin{align*}
    \norm{W_{S_{kk}^{(t+1)}}} \leq&    
    3\sigma \sqrt{d + \log(mn)}\sqrt{\nu_k^\star} + \sigma\sqrt{3(mn+d)}\sqrt{\nu_k^\star - \nu_{kk}^{(t+1)}}\\
    \leq& 3\sigma \sqrt{d + \log(mn)}\sqrt{\nu_k^\star} + \sigma\sqrt{3(mn+d)}\sqrt{\frac{\nu_k^\star}{2}}
\end{align*}
For the last inequality, 
since $G^{(t+1)} \leq 1/2$, $\nu_{kk}^{(t+1)} \geq \nu_k^\star(1 - G^{(t+1)}) \geq \frac{1}{2}\nu_k^\star$.

For the second term in ~\eqref{eq:err_centers_k_mid}, we fix $h\neq k\in [K]$ and analyze the term $\bar{Y}_{S_{hk}^{(t+1)}}= \sum_{i\in [m]}\frac{\nu_{hk,i}^{(t+1)}}{\nu_{hk}^{(t+1)}}\bar{Y}_{S_{hk,i}^{(t+1)}}$. 
\begin{align*}
\norm{\bar{Y}_{S_{hk,i}^{(t+1)}} - \theta_k} \leq& \norm{\bar{Y}_{S_{hk,i}^{(t+1)}} - \hat{\theta}_{k,i}^{(t)}} + \norm{\hat{\theta}_{k,i}^{(t)} - \hat{\theta}_{k}}\leq  \norm{\bar{Y}_{S_{hk,i}^{(t+1)}} - \hat{\theta}_{h,i}^{(t)}} + \norm{\hat{\theta}_{k,i}^{(t)} - \theta_k}\\
\leq & \norm{\bar{Y}_{S_{hk,i}^{(t+1)}} -\theta_h} + \norm{\hat{\theta}_{h,i}^{(t)} - \theta_h} + \norm{\hat{\theta}_{k,i}^{(t)} - \theta_k}\leq  \frac{\norm{W_{S_{hk,i}^{(t+1)}}}}{\nu_{hk,i}^{(t+1)}} +  (2\Lambda^{(t)} + \delta_1)\Gamma 
\end{align*}
For the first step we use triangle inequality and for the second step we use local Lloyd's algorithm's  update step. In the third step, we again use triangle inequality. In the fourth step, we use Lemma~\ref{lem:local_dev_all_k} as $t\geq 1$ when $L$ doesn't divide $t$. 

Summing this up over $i\in [m]$ and $h\neq k, h \in [K]$, we get,
\begin{align*}
\sum_{h\neq k \in [K]}\frac{\nu_{hk}^{(t)}}{\nu_{k}^{(t)}}\norm{\bar{Y}_{S_{hk}^{(t)}} - \theta_k} \leq& \sum_{h\neq k \in [K]} \frac{\sum_{i\in [m]}\norm{W_{S_{hk,i}^{(t)}}}}{\nu_{k}^{(t)}} + G^{(t)} (2\Lambda^{(t-1)} +\delta_1)\Gamma 
\end{align*}
We now handle the first and second terms separately.

We now analyze the first term using Lemma~\ref{lem:noise_set_subg}.
\begin{align*}
    &\sum_{h\neq k \in [K]} \frac{\sum_{i\in [m]}\norm{W_{S_{hk,i}^{(t)}}}}{\nu_{k}^{(t)}} \leq \sum_{h\neq k \in [K]} \frac{\sum_{i\in [m]}\sigma\sqrt{3(n + d)}\sqrt{\nu_{hk,i}^{(t)}}}{\nu_{k}^{(t)}}\\
&\quad \leq  \frac{\sum_{i\in [m]}\sigma\sqrt{3K(n + d)}\sqrt{\sum_{h\neq k \in [K]}\nu_{hk,i}^{(t)}}}{\nu_{k}^{(t)}}
    \leq  \frac{\sum_{i\in [m]}\sigma\sqrt{3K(n + d)n}}{\nu_{k}^{(t)}}
\end{align*}
We use Cauchy-Schwarz to obtain a term of $\sqrt{K}$ and then bound the term inside square root by $n$.

Substituting everything to obtain the single step progress, we get, 
\begin{align*}
    \frac{\norm{\hat{\theta}_k^{(t)} - \theta_k}}{\Gamma}\leq& \frac{1}{\nu_k^{(t)}}\bigg(3\frac{\sigma}{\Gamma} \sqrt{d + \log(mn)}\sqrt{\nu_k^\star} + \frac{\sigma}{\Gamma}\sqrt{3(mn+d)}\sqrt{\frac{\nu_k^\star}{2}}\\
    &+ \frac{m}{r_K}\sqrt{3\alpha K(n + d)n} \bigg) + G^{(t)}(2\Lambda^{(t)} + \delta_1)
\end{align*}
Now, using $\nu_k^{(t)} \geq \nu_{kk}^{(t)} \geq \nu_k^\star(1-G^{(t)}) \geq \frac{1}{2}\nu_k^\star \geq \frac{1}{2}\alpha mn$, we obtain,
\begin{align*}
    \frac{\norm{\hat{\theta}_k^{(t)} - \theta_k}}{\Gamma}\leq& \frac{3}{r_K}+ \frac{4}{r_K}\sqrt{\frac{K}{\alpha}}\left(\sqrt{\frac{d}{n}} + 1\right)  + G^{(t)}(2\Lambda^{(t)} + \delta_1)
\end{align*}
We need $r_K \geq \sqrt{\frac{K}{\alpha}}\max\left\{\sqrt{\frac{d}{n}},1\right\}$ for this term to be smaller than $1$.
\end{proof}

\subsection{Error in Labels}
Now that we have established the progress in terms of error of centers, we establish progress in terms of misclustering. We will write down progress in terms of $G^{(t)}$ instead of the actual misclustering $A^{(t)}$, as a recursion for it is easier to derive. Since we are concerned with progress at every local step, we need to break down our analysis into two cases -- i) when $L$ divides $t$ and ii) when $L$ doesn't divide $t$.

\paragraph{When $L$ divides $t$}
\begin{lemma}\label{lem:err_labels_k_L}
If $\Lambda^{(t)} \leq \frac{1 - \epsilon'}{2}$ and $r_K \geq 36(\epsilon')^{-2}$, for some $\epsilon'>0$, then, 
\begin{align*}
    G^{(t+1)}\leq \frac{2}{(\epsilon')^4 r_K^2} + \left(\frac{28}{(\epsilon')^2 r_K}\Lambda^{(t)}\right)^2  + \Psi_1
\end{align*}
\end{lemma}
Note that this Lemma is a direct adaptation of ~\cite[Lemma~A.7]{lu_statistical_2016}, but for $mn$ datapoints.

\paragraph{When $L$ doesn't divide $t$}
In this case, we need to account for the error between $\hat{\theta}_{k,i}^{(t)}-\hat{\theta}_k^{(t)}$. Our approach follows that of the $2$-cluster case, where we try to separate out this term and bound it individually.

\begin{lemma}\label{lem:err_labels_k}
If $\Lambda^{(t)} \leq \frac{1 - \epsilon' - \delta_1 - \delta_2}{2}$ and $r_K =\Omega(\max\{(\epsilon')^{-2}, (\delta_2)^{-2}\sqrt{K\max\{\frac{d}{n},1\}}\})$, for some $\epsilon', \delta_2>0$, then, with probability $1 - \mathcal{O}(\exp(-n))$, we have,
\begin{align*}
G^{(t+1)}\leq& \frac{2}{(\epsilon')^4 r_K^2} + \left(\left(\frac{28}{(\epsilon')^2 r_K}\right)^2 + \left(\frac{40\max\{\sqrt{\frac{d}{n}},1\}}{(\delta_2)^2 r_K}\right)^2\right)(\Lambda^{(t)})^2 + \Psi_1 \\
&\quad + \left(\frac{14\max\{\sqrt{\frac{d}{n}},1\}\delta_1 }{\delta_2^2 r_K}\right)^2 + \left(\frac{28 \max\{\sqrt{\frac{d}{n}},1\}}{\delta_2^2 r_K}\Lambda^{(t-1)}\right)^{2}
\end{align*}
\end{lemma}
\begin{proof}
Our analysis follows that of 2-cluster case single step progress where we try to separate out the deviation $\Delta^{(t)}$. Let $\Delta_{k,i}^{(t)} =\hat{\theta}_{k,i}^{(t)} -  \hat{\theta}_k^{(t)}$ and $\Gamma_k^{(t)} = \hat{\theta}_k^{(t)} - \theta_k$ for all $k\in [K]$ and $i\in [m]$.
    For any $k\neq h \in [K]^2$, for the $(i,j)$ datapoint, we have,
    \begin{align*}
        \I\{z_{i,j}& = k, \hat{z}_{i,j}^{(t)} = h\} = \I\{\norm{\theta_k + \vw_{i,j} - \hat{\theta}_{h,i}^{(t)}}^2 \leq \norm{\theta_k + \vw_{i,j} - \hat{\theta}_{k,i}^{(t)}}^2\}\\
        \leq& \I\{\norm{\theta_k  - \hat{\theta}_{h,i}^{(t)}}^2 - \norm{\theta_k  - \hat{\theta}_{k,i}^{(t)}}^2   \leq  2\ip{\vw_{i,j}}{\hat{\theta}_{k,i}^{(t)} - \hat{\theta}_{h,i}^{(t)}}\} \\
        \leq& \I\{\norm{\theta_k  - \hat{\theta}_{h,i}^{(t)}}^2 - \norm{\theta_k  - \hat{\theta}_{k,i}^{(t)}}^2   \leq  2\ip{\vw_{i,j}}{\theta_h  - \theta_k} +2\ip{\vw_{i,j}}{\Gamma_h^{(t)} - \Gamma_k^{(t)}}+ 2\ip{\vw_{i,j}}{\Delta_{h,i}^{(t)} - \Delta_{k,i}^{(t)}}\} \\
        \leq& \I\{(1 - 2\Lambda_i^{(t)})^2\norm{\theta_k  - \theta_h}^2  \leq  2\ip{\vw_{i,j}}{\theta_h  - \theta_k} +2\ip{\vw_{i,j}}{\Gamma_h^{(t)} - \Gamma_k^{(t)}}+ 2\ip{\vw_{i,j}}{\Delta_{h,i}^{(t)} - \Delta_{k,i}^{(t)}}\} \\
        \leq& \I\{((\epsilon')^2 + \delta_2^2)\norm{\theta_k  - \theta_h}^2  \leq  2\ip{\vw_{i,j}}{\theta_h  - \theta_k} +2\ip{\vw_{i,j}}{\Gamma_h^{(t)} - \Gamma_k^{(t)}}+ 2\ip{\vw_{i,j}}{\Delta_{h,i}^{(t)} - \Delta_{k,i}^{(t)}}\}
    \end{align*}
    In the second step, we separate the terms with the noise $\vw_{i,j}$. In the third step, we separate out the terms of $\Gamma_{k}^{(t)}$ and $\Delta_{k,i}^{(t)}$. In the fourth step, we remove the terms of $\Delta_{k,i}^{(t)}$. In the third step, we follow ~\citep[Eq~56]{lu_statistical_2016} to upper bound the terms on the LHS. First, we use triangle inequality to obtain $\norm{\theta_k - \theta_{h,i}^{(t)}}^2 \geq (\norm{\theta_k - \theta_h}^2 - \norm{\theta_h- \theta_{h,i}^{(t)}}^2) \geq (1 - \Lambda_i^{(t)})^2 \norm{\theta_k - \theta_h}^2$. Then, we apply $(1-x)^2 - y^2 \geq (1 - x - y)^2$ for $y(1 - x - y) \geq 0$ using $\Lambda_i^{(t)} \leq \Lambda^{(t)} + \frac{\delta_1}{2} \leq \frac{1 - \epsilon' - \delta_2}{2}$(Lemma~\ref{lem:local_dev_all_k}) to obtain, $\norm{\theta_k - \theta_{h,i}^{(t)}}^2 - \norm{\theta_k - \hat{\theta}_{k,i}^{(t)}}^2  \geq (1 - 2\Lambda_i^{(t)})^2 \norm{\theta_k - \theta_h}^2\geq ((\epsilon')^2 + \delta_2^2) \norm{\theta_k - \theta_h}^2$. 
    We split the above equation into 3 parts by applying $\I\{a + b \leq 0\} \leq \I\{a \leq c\} + \I\{b \leq - c\} \leq \I\{a \leq c\} + \frac{b^2}{c^2}$ for $c>0$ two times and use $\norm{\theta_k - \theta_h}^2 \geq \Gamma^2$

    \begin{align*}
        \I\{z_{i,j} = k, \hat{z}_{i,j}^{(t)} = h\} \leq&  \I\{\frac{(\epsilon')^2}{4}\norm{\theta_k  - \theta_h}^2  \leq  2\ip{\vw_{i,j}}{\theta_h  - \theta_k}\}+\I\{\frac{(\epsilon')^2}{4}\Gamma^2  \leq \ip{\vw_{i,j}}{\Gamma_h^{(t)} - \Gamma_k^{(t)}}\}\\
        &+\I\{\frac{\delta_2^2}{2}\Gamma^2  \leq  2\ip{\vw_{i,j}}{\Delta_{h,i}^{(t)} - \Delta_{k,i}^{(t)}}\}
    \end{align*}
Now, summing over $(i,j) \in S_k^\star$, we obtain,
\begin{equation}
\begin{aligned}\label{eq:k_cluster_decomp}
    \nu_{kh}^{(t+1)} \leq& \underset{I_1}{\underbrace{\sum_{(i,j) \in S_k^\star}I\{\frac{(\epsilon')^2}{4}\norm{\theta_k  - \theta_h}^2  \leq  2\ip{\vw_{i,j}}{\theta_h  - \theta_k}\}}} + \underset{I_2}{\underbrace{\sum_{(i,j) \in S_k^\star} \frac{16}{(\epsilon')^4\Gamma^4}\ip{\vw_{i,j}}{\Gamma_h^{(t)} - \Gamma_k^{(t)}}^2}}\\
    &\quad +\underset{I_3}{\underbrace{\sum_{(i,j) \in S_k^\star} \frac{4}{\delta_2^4 \Gamma^4}\ip{\vw_{i,j}}{\Delta_{h,i}^{(t)} - \Delta_{k,i}^{(t)}}^2}}
\end{aligned}
\end{equation}
Note that the terms $I_1$ and $I_2$ match those in ~\citep[Eq~57]{lu_statistical_2016}, so we use their upper bounds.

\paragraph{Upper Bound on $I_1$}
Using Lemma~\ref{lem:indicator_inner_prod},
\begin{align}\label{eq:I_1_k}
    I_1 \leq \nu_k^\star\exp(-\frac{(\epsilon')^4 \Gamma^2}{32\sigma^2}) + \sqrt{5\nu_k^\star \log (mn)} 
\end{align}

\paragraph{Upper bound on $I_2$}
By Lemma~\ref{lem:noise_eigen_sub},
\begin{align}\label{eq:I_2_k}
    I_2 \leq \frac{96(\nu_k^\star + d)\sigma^2}{(\epsilon')^4 \Gamma^4}\norm{\Gamma_h^{(t)} - \Gamma_k^{(t)}}^2 \leq \frac{384(\nu_k^\star + d)\sigma^2}{(\epsilon')^4 \Gamma^2}(\Lambda^{(t)})^2
\end{align}

\paragraph{Bounding deviation term $I_3$}
$I_3$ contains the deviation term in our case, so we will upper bound it in terms of $\Delta^{(t)}$. We use Lemma~\ref{lem:noise_eigen_sub} for each set $S_{k,i}^\star$  where $i\in [m]$. Note that this adds a $\sum_{i=1}^m\exp(-\nu_{k,i}^\star) \leq m \exp(-\beta \alpha n) = \exp(-n)$ to the high probability terms. 
\begin{align*}
    \sum_{(i,j) \in S_k^\star} \ip{\vw_{i,j}}{\Delta_{h,i}^{(t)} - \Delta_{k,i}^{(t)}} &\leq \sum_{i=1}^m \lambda_{\max}(\sum_{j\in S_{k,i}^\star}\vw_{i,j}\vw_{i,j}^\top)\norm{\Delta_{h,i}^{(t)} - \Delta_{k,i}^{(t)}}^2\\
    &\leq 6\sigma^2 \sum_{i=1}^m (\nu_{k,i}^\star  + d)\norm{\Delta_{h,i}^{(t)} - \Delta_{k,i}^{(t)}}^2\\
    &\leq 12 \sigma^2 (n  + d)\sum_{i=1}^m (\norm{\Delta_{h,i}^{(t)}}^2 + \norm{\Delta_{k,i}^{(t)}}^2) \\
    &\leq 24 \sigma^2(n+d) \max_{h\in [K]}\sum_{i=1}^m \norm{\Delta_{h,i}^{(t)}}^2\\
    &= 48 \sigma^2 \Gamma^2 mn\max\{\frac{d}{n},1\} \Delta^{(t)}\leq 48 \sigma^2 mn \max\{\frac{d}{n},1\}
\end{align*}
We use the definition of $\Delta^{(t)} \triangleq \max_{h\in [K]}\frac{1}{m\Gamma^2}\norm{\hat{\theta}_{h,i}^{(t)} - \hat{\theta}_h^{(t)}}^2$. Additionally, $\nu_{k,i}^\star \leq n$.

Therefore, the bound on $I_3$ is 
\begin{align*}
   I_3 \leq \frac{192 mn\sigma^2\max\{\frac{d}{n},1\}}{\delta_2^4 \Gamma^2}\Delta^{(t)}
\end{align*}
where $\Delta^{(t)} = \max_{k\in [K]}\frac{1}{m}\sum_{i=1}^m \norm{\hat{\theta}_{k,i}^{(t)} - \hat{\theta}_k^{(t)}}^2$.

To bound the deviation term $\Delta^{(t)}$, note that by $\norm{a+b}^2 \leq 2\norm{a}^2 + 2\norm{b}^2$ and $\Var(X) \leq \E[X^2]$, we have,

\begin{align*}
    \frac{1}{m}\sum_{i=1}^m\norm{\hat{\theta}_{k,i}^{(t)} - \hat{\theta}_k^{(t)}}^2 =&\frac{1}{m}\sum_{i=1}^m\norm{\hat{\theta}_{k,i}^{(t)} - \hat{\theta}_k^{(t-1)}- (\hat{\theta}_k^{(t)} - \hat{\theta}_k^{(t-1)})}^2  \\
    \leq& \frac{1}{m}\sum_{i=1}^m\norm{\hat{\theta}_{k,i}^{(t)}-\hat{\theta}_k^{(t-1)}}^2\\
    \leq & \frac{2}{m}\sum_{i=1}^m\norm{\hat{\theta}_{k,i}^{(t)}-\theta_k}^2  + 2\norm{\hat{\theta}_k^{(t-1)} - \theta_k}^2\\
\implies \Delta^{(t)}    \leq & \frac{2}{m}\sum_{i=1}^m(\Lambda_i^{(t)})^2  + 2(\Lambda^{(t-1)})^2\\
\leq & \frac{2}{m}\sum_{i=1}^m(\Lambda{(t)} + \frac{\delta_1}{2})^2  + 2(\Lambda^{(t-1)})^2\\
\leq & 4(\Lambda{(t)})^2 + \frac{\delta_1^2}{2}  + 2(\Lambda^{(t-1)})^2
\end{align*}

This implies the following bound on $I_3$.
\begin{align}
    I_3 \leq \frac{192mn\sigma^2\max\{\frac{d}{n},1\}}{\delta_2^4 \Gamma^2}(4(\Lambda{(t)})^2 + \frac{\delta_1^2}{2}  + 2(\Lambda^{(t-1)})^2)\label{eq:dev_term_k}
\end{align}

\paragraph{Completing the proof}
We complete the proof by first dividing both sides by $\nu_k^\star$ and summing over all the remaining clusters. This adds an additional coefficient of $K$ to the high probability term, which doesn't affect it's value as long as $K = \mathcal{O}(1)$.

For the true negative rate,
\begin{align*}
    \max_{k \in [K]}\sum_{h\neq k}\frac{\nu_{kh}^{(t+1)}}{\nu_k^\star} \leq& K\exp(-\frac{(\epsilon')^4 \Gamma^2}{32 \sigma^2}) + K\sqrt{\frac{5\log (mn)}{\alpha mn}}  + \frac{384}{(\epsilon')^4 r_K^2}(\Lambda^{(t)})^2 \\
    &+ \frac{192 K \sigma^2\max\{\frac{d}{n},1\}}{\alpha\delta_2^4 \Gamma^2}(4(\Lambda{(t)})^2 + \frac{\delta_1^2}{2}  + 2(\Lambda^{(t-1)})^2)
\end{align*}
We use $\nu_k^\star \geq \alpha mn$. We therefore need $\alpha mn \geq 32 K^2 \log(mn)$ and $r_K =\Omega(\min\{(\epsilon')^{-2}, (\delta_2)^{-2}\max\{\sqrt{\frac{d}{n}},1\})$ for the RHS to be $\leq\frac{1}{2}$. This gives us a bound on the false positive rate.

For the true negative rate, if the RHS above is $\leq \frac{1}{2}$, we have, $\nu_h^{(t+1)} \geq \nu_{hh}^{(t+1)} \geq \frac{1}{2}\nu_h^\star \geq \frac{1}{2}\alpha mn$.  
\begin{align*}
    \max_{h \in [K]}\sum_{k\neq h}\frac{\nu_{hk}^{(t+1)}}{\nu_h^{(t+1)}} \leq& \frac{2}{\alpha}\exp\left(-\frac{(\epsilon')^4 \Gamma^2}{32 \sigma^2}\right) + \Psi_1  + \left(\frac{768}{(\epsilon')^4 r_K^2}+\frac{1536\max\{\frac{d}{n},1\} }{\delta_2^4 r_K^2}\right)\Lambda^{(t)} \\
    &+ \frac{192\max\{\frac{d}{n},1\}\delta_1^2 }{\delta_2^4 r_K^2} + \frac{768 \max\{\frac{d}{n},1\} }{\delta_2^4 r_K^2}(\Lambda^{(t-1)})^{2}
\end{align*}
Combining the false positive and true negative rates and using $\sqrt{1536}\leq 40, \sqrt{768}\leq 28, \sqrt{192}\leq 14$, we complete the proof.

\end{proof}

\subsection{Single Step Progress}
Note that we obtained $\Lambda^{(0)}\leq \frac{1}{2}-\frac{3}{\sqrt{r_K}}$ and $G^{(0)}\leq \frac{1}{2}$ from Assumption~\ref{assumption:local_dev_k} and Lemma~\ref{lem:err_centers_k_all_t}. Applying Lemma~\ref{lem:err_labels_k_L}, for $t=0$, with $\epsilon' = \frac{6}{\sqrt{r_K}}$, with $\Lambda^{(0)}\leq \frac{1}{2}$, we obtain $G^{(1)} \leq 0.18$ for large $mn$. Plugging this value into Lemma~\ref{lem:err_centers_k_L}, we obtain,  $\Lambda^{(1)} \leq 0.2$ for large $r_K = \Omega(\sqrt{K})$.

Similar to the $2$-cluster case, we use these small values of $\Lambda^{(1)}$ and $G^{(1)}$ to show that $G^{(t)} \leq 0.18, \Lambda^{(t)} \leq 0.2 , \forall t\geq 1$. Further, this will allow us to use a constant $\epsilon'$ in Lemma~\ref{lem:err_labels_k} and use a finite constant $\delta_1 \leq 0.1$ in Lemma~\ref{lem:local_dev_all_k}. We use an inductive argument to prove this statement. First, note that the base case for $t=1$ is satisfied by the previous paragraph.

Assume that this is true for all iterations $\leq t$ for the inductive hypothesis.
Then, from Lemmas~\ref{lem:err_centers_k_L} and Lemmas~\ref{lem:err_centers_k} for iteration $\tau-1$, by setting $\Lambda^{(\tau-1)} \leq 0.5$ and using $\delta_1\leq 0.1$, we have, 
\begin{align}\label{eq:lambda_val_k}
    \Lambda^{(\tau)} \leq \sqrt{\Psi_2} + G^{(\tau)}, \quad \forall 1 \leq \tau\leq t 
\end{align}
Using this inequality in Lemmas~\ref{lem:err_labels_k_L} and ~\ref{lem:err_labels_k}, with $\delta_1 \leq 0.1$ and $\delta_2=\epsilon'=0.1$ in , we find that for some constant $D_1,D_2>0$, we have,
\begin{align*}
    G^{(t+1)} \leq \Psi_1+ \frac{D_1 G^{(t)}}{r_K^2}
\end{align*}
when $L$ divides $t$, and 
\begin{align*}
    G^{(t+1)} \leq \Psi_1+ \Psi_2 +  \frac{D_2 G^{(t)}}{r_K^2} + \frac{D_2 G^{(t-1)}}{r_K^2}
\end{align*}
when $L$ doesn't divide $t$. We use the fact that $r_K$ is large and consume all the constants inside $D_1, D_2, \Psi_1$ and $\Psi_2$. Therefore, the inequality $\Lambda^{(t+1)} 0.2, G^{(t+1)} \leq 0.18$ is satisfied. Therefore, by induction, it is satisfied for all time $t\geq 1$.

\subsection{Final Statistical Error}
We need to unroll a single-step recursion in terms of $\Lambda^{(t)}$ to show final error. We first write down the recursion, in the two cases.

When $L$ divides $t$, we have, for some constant $D_3 > 0$,
\begin{align*}
    \Lambda^{(t+1)}\leq 0.36 \Lambda^{(t)} + \frac{D_3}{r_K} + \sqrt{\Psi_1} \leq 0.4 \Lambda^{(t)} + \sqrt{\Psi_1}
\end{align*}
This is obtained from setting $\frac{3}{r_K}\sqrt{KG^{(t+1)}} + G^{(t+1)}\leq \sqrt{G^{(t+1)}}$, for large $r_K$ and $G^{(t)} \leq 0.18, \forall t\geq 1$. Further, we plug this bound into Lemma~\ref{lem:err_centers_k_L} with single step progress bound for $\sqrt{G^{(t+1)}}$. For a large enough $r_K$, $0.4 + \frac{D_3}{r_K}\leq 0.5$.

When $L$ doesn't divide $t$, we use Lemma~\ref{lem:err_centers_k}.
\begin{align*}
    \Lambda^{(t+1)} \leq \sqrt{\Psi_1} + \sqrt{\Psi_2} + 0.26 \Lambda^{(t)} + 2\delta_1 G^{(t+1)}
\end{align*}
We use the bound on $\delta_1 \leq 0.1$, and Lemma~\ref{lem:err_labels_k} for a bound on $G^{(t+1)}$ in terms of $\Lambda^{(t-1)}$ and $\Lambda^{(t)}$. Finally, using the fact that $\Lambda^{(t)}\leq 0.2$ and $\Lambda^{(t-1)}\leq 0.5, \forall t\geq 1$, we obtain the following update equation for some constants $D_4, D_5>0$.
\begin{align*}
    \Lambda^{(t+1)}\leq \sqrt{\Psi_1} + \sqrt{\Psi_2} + D_4 \Lambda^{(t)} + D_5 \Lambda^{(t-1)}
\end{align*}

Here, the value of $D_4 \geq 0.36$ and $D_5$ can be set to a small constant for large $r_K$. Let $r_K$ be large enough that $D_4 \leq 0.4$ and $D_5 \leq 0.2$. Then, we have, the following single-step update in $\Lambda^{(t)}$.

\begin{align*}
    \Lambda^{(t+1)} \leq& 0.4 \Lambda^{(t)} + \sqrt{\Psi_1}\tag{When $L$ divides $t$}\\
    \Lambda^{(t+1)} \leq& 0.4 \Lambda^{(t)} + 0.2 \Lambda^{(t-1)}+ \sqrt{\Psi_1} + \sqrt{\Psi_2}\tag{When $L$ doesn't divide $t$}
\end{align*}

Similar to the $2$-cluster case, we will show that $\Lambda^{(t)}$ obeys the following geometric decrease by combining the above two update equations.
\begin{align}\label{eq:rec_total}
    \Lambda^{(t)} \leq \gamma^{t-1} \Lambda^{(1)} + \sqrt{\Psi_1} + \sqrt{\Psi_2}, \forall t\geq 1
\end{align}
We will prove that $\Lambda^{(t)}$ satisfies Eq~\eqref{eq:rec_total} by induction. At $t=1$, for the base step, $\Lambda^{(1)} \leq \Lambda^{(1)} + \sqrt{\Psi_1} + \sqrt{\Psi_2}$. Now, assume that Eq~\eqref{eq:rec_total} is satisfied for all iterations $\leq t$. At iteration $t$,  we have the following recursions, if $L$ divides $t$ and if it doesn't,  
\begin{align*}
    \Lambda^{(t+1)} \leq& 0.4\gamma^{(t-1)} \Lambda^{(1)} + \sqrt{\Psi_1}+ \sqrt{\Psi_2}\\
    \Lambda^{(t+1)} \leq& (0.4\gamma + 0.2)\gamma^{(t-2)}\Lambda^{(1)} + \sqrt{\Psi_1} + \sqrt{\Psi_2}
\end{align*}
Note that we need the coefficient of $\Lambda^{(1)}$ in both the equations to be $\leq \gamma^{(t)}$. Therefore, we have $\gamma \geq 0.4$ and $\gamma^2 \geq 0.4 \gamma  + 0.2$. Note that setting $\gamma = 0.5$ satisfies the inequality. 

Further, even for the $2$-cluster case, the linear decrease in $A^{(t)}$ at every single step was by $\frac{1}{2}$. Therefore, similar to the $2$-cluster case, we require at least $2(\ceil{\log(mn)}+L)$ iterations, until,
\begin{align*}
    \Lambda^{(t)} \leq \sqrt{\Psi_1} + \sqrt{\Psi_2}
\end{align*}
Note that we absorb all constants inside $\Psi_1$ and $\Psi_2$.

To compute the final statistical error, we use a tighter analysis of the single-step progress for $t\geq 2(\ceil{\log(mn)} + L)$. This resembles that of the $2$-cluster case. Note that as $\Lambda^{(t)}$ is small, we only use the decomposition of progress in $A^{(t)}$ when $L$ doesn't divide $t$ as it is worse than when $L$ divides $t$.

We will compute the misclustering $A^{(t+1)} = \frac{1}{mn}\sum_{i=1}^m \sum_{j=1}^n \sum_{h, k\in [K], h\neq k} \I\{\hat{z}_{i,j}^{(t+1)} = h, z_{i,j} = k\}$. We will use a different decomposition for Eq~\eqref{eq:k_cluster_decomp} following ~\cite[Section~A.3]{lu_statistical_2016}. 
\begin{align*}
&\I\{\hat{z}_{i,j}^{(t+1)} = h, z_{i,j} = k\}\\
&\leq \I\{(1 - 2\Lambda_i^{(t)})^2\norm{\theta_k  - \theta_h}^2  \leq  2\ip{\vw_{i,j}}{\theta_h  - \theta_k} +2\ip{\vw_{i,j}}{\Gamma_h^{(t)} - \Gamma_k^{(t)}}+ 2\ip{\vw_{i,j}}{\Delta_{h,i}^{(t)} - \Delta_{k,i}^{(t)}}\} \\
&\leq \I\{\beta_1 \norm{\theta_k - \theta_h}^2 \leq 2\ip{\vw_{i,j}}{\theta_h - \theta_k}\} + \I\{\beta_2\Gamma^2 \leq 2\ip{\vw_{i,j}}{u_h - u_k}\} + \I\{\beta_3\Gamma^2 \leq 2\ip{\vw_{i,j}}{v_k - v_h}\}\\
&+ \I\{\delta_2^2 \norm{\theta_k - \theta_h}^2 \leq 2\ip{\vw_{i,j}}{\Delta_{h,i}^{(t)} - \Delta_{k,i}^{(t)}}\}
\end{align*}
Here, $u_k = \frac{1}{\nu_k^\star} W_{S_k^\star}, \forall k\in [K]$ and $v_k = \Gamma_k^{(t)} - u_k$. Further, the term $\beta_1 \leq (1-2\Lambda_i^{(t)})^2 - \beta_2 - \beta_3 - \delta_2$. Since $\Lambda_i^{(t)} \leq \Lambda^{(t)} + \frac{\delta_1}{2} \leq \sqrt{\Psi_1} + \sqrt{\Psi_2} + 0.05, \forall t\geq 2(\ceil{\log(mn)} + L)$, $\beta_1 \geq 0.5 - \beta_2 - \beta_3 -\delta_2$ for small $\Psi_1$ and $\Psi_2$. Note that $\Psi_1$ and $\Psi_2$ are small for large $n$ and $r_K$.

Further, averaging the above inequality over all $k\in [K]$ and all datapoints $(i,j)\in S_k^\star$, we obtain $A^{(t+1)}$.
\begin{align*}
    A^{(t+1)} \leq& I_1' + J_2 + J_3 + I_3'
\end{align*}
Here, the term $I_1'$ is the same as the term $I_1$ except for the value of $\beta_1$. The terms $J_2$ and $J_3$ are exactly equal to the corresponding terms in ~\cite[Section~A.3]{lu_statistical_2016}. The additional term of $I_3'$ depends on the deviation term $\Delta^{(t)}$ and is similar to $I_3$ in Eq~\eqref{eq:k_cluster_decomp}.

\paragraph{Bounds on $I_1', J_2, J_3$}
For these terms, we use the bounds from ~\cite[Eq~(63) - (65)]{lu_statistical_2016}. Following their analysis, we set $\beta_2 = \sqrt{\frac{8K}{r_K}}, \beta_3 = \frac{64}{r_K}$ and $\delta_2$ as a constant $\leq 0.2$, we find that $\beta_1\geq 0.25$.
\begin{align*}
    &\E[I_1'] \leq \exp(-r_K^2),\quad \E[J_3] \leq \exp(-r_K^2)\\
    &\text{wp } 1 - \exp(-r_K^2), \quad J_2 \leq \frac{12\sqrt{K}}{r_K}A^{(t)}
\end{align*}

\paragraph{Bound on $I_3'$}
Note that the bounds on $I_3'$ can be obtained from the bound on $I_3$. Since, $I_3$ is only a summation over $(i,j)\in S_k^\star$, and $I_3$ is the average over $k\in [K]$ and $(i,j)\in S_k^\star$, we obtain,
\begin{align*}
    I_3' \leq \frac{K}{mn} I_3 \leq \frac{192 K\sigma^2 \max\{\frac{d}{n},1\}}{\delta_2^4 \Gamma^2}(4 (\Lambda^{(t)})^2 + \frac{\delta_1^2}{2} + 2(\Lambda^{(t-1)})^2)
\end{align*}
We plug in the value of $I_3$ from Eq~\eqref{eq:dev_term_k}. For each of the terms $(\Lambda^{(t-1)})^2$ and $(\Lambda^{(t)})^2$, we use Eq~\eqref{eq:lambda_val_k}.
\begin{align*}
    (\Lambda^{(t)})^2 \leq \Psi_2 + 2(G^{(t)})^2 \leq \Psi_2 + \frac{\Psi_1 + \Psi_2}{\alpha} A^{(t)}
\end{align*}
We first use $\norm{a+b}^2 \leq 2\norm{a}^2 + 2\norm{b}^2$. Then, we use the fact $G^{(t)} \leq \frac{A^{(t)}}{\alpha}$. Further, to bound the value of $G^{(t)}$ after $t\geq 2(\ceil{\log(mn)} + L)$, we use the fact that Lemma~\ref{lem:single_step_k} implies a similar progress as $\Lambda^{(t)}$. We plug in the value of $\Lambda^{(t)} \leq \sqrt{\Psi_1} +\sqrt{\Psi_2}$ in Lemmas~\ref{lem:err_labels_k_L} and ~\ref{lem:err_labels_k}, for any $t\geq 2(\ceil{\log(mn)} + L)$, to obtain that $G^{(t)} \leq \Psi_1 + \Psi_2$ for the same condition on $t$. Therefore,
\begin{align*}
    I_3' \leq \Psi_2 + \frac{K\max\{\frac{d}{n},1\}(\Psi_1 + \Psi_2)}{\alpha r_K^2}(A^{(t)} + A^{(t-1)})
\end{align*}

Note that the high probability terms corresponding to $(mn)^{-3}$ are replaced by $\exp(-r_K^2)$ in the bound of $I_1$ in the proof of Lemma~\ref{lem:err_labels_k}. Let $\cE$ be the intersection of the high probability events under which all our bounds hold. Then, for any constant $\delta_2$, we have, 
\begin{align*}
    \E[A^{(t+1)}] \leq& \E[J_1] + \E[J_2| \cE] + \E[J_3] + \E[I_3' | \cE] + \Pr[\cE^\complement]\\
    \leq& \left(\frac{12K}{r_K} + \frac{K(\Psi_1 + \Psi_2)\max\{\frac{d}{n},1\}}{\alpha r_K^2}\right)\E[A^{(t)}] +  \frac{K(\Psi_1 + \Psi_2)\max\{\frac{d}{n},1\}}{\alpha r_K^2}\E[A^{(t-1)}]\\
    &\quad + \exp(-r_k^2) + \exp(-n) + \Psi_2
\end{align*}
Note that the term $\Pr[\cE^\complement] \leq \exp(-r_K^2) + \exp(-n)$. For large $r_K$ and $n$, we obtain a geometric decrease in $\E[A^{(t)}]$ for $t\geq 2(\ceil{\log(mn)}+L)$. We will again use induction to show this. Specifically, we want to show that for some $\gamma \in (0,1)$, we have,
\begin{align*}
    \E[A^{(t)}] \leq \gamma^{t-t_0} \E[A^{(t_0)}] + \Psi_2 + \exp(-n) + \exp(-r_K^2), \quad, \forall t \geq t_0 = 2(\ceil{\log(mn)} + L)
\end{align*}
For $t=t_0 = 2(\ceil{\log(mn)} + L)$, this bound is automatically satisfied. Assuming that this is true for all iterations between $t_0$ and $t$, by the single-step progress in $\E[A^{(t)}]$, we obtain,
\begin{align*}
    \E[A^{(t+1)}] \leq \left(\frac{12K\gamma}{r_K} + \frac{K(\Psi_1 + \Psi_2)\max\{\frac{d}{n},1\}}{\alpha r_K^2}(\gamma + 1)\right)\gamma^{t-1-t_0}\E[A^{(t_0)}] + \Psi_2 + \exp(-n) + \exp(-r_K^2)
\end{align*}
As long as $\gamma^2 \geq \left(\frac{12K\gamma}{r_K} + \frac{K(\Psi_1 + \Psi_2)\max\{\frac{d}{n},1\}}{\alpha r_K^2}(\gamma + 1)\right)$, the geometric decrease is satisfied. Note that if $a = \frac{12K}{r_K}$ and $b=\frac{K(\Psi_1 + \Psi_2)\max\{\frac{d}{n},1\}}{\alpha r_K^2}$, then, we require, $\gamma^2 \geq (a + b)\gamma+ b$. Therefore, if we select $\gamma \geq \frac{(a+b) + \sqrt{(a+b)^2 + 4b^2}}{2}$, then we get geometric decrease. Note that both $a$ and $b$ can be made arbitarily small for large $r_K$ and $n$, implying that we can find a $\gamma < 1$ that satisfies the geometric decrease.

Under the geometric decrease condition on $\E[A^{(t)}]$, if we run for additional $t - t_0 = q \cdot (\ceil{\log(mn)}+L)$ iterations, and use $\E[A^{(t_0)}] \leq 1$ along with the Markov inequality, we obtain Theorem~\ref{thm:k_cluster}.

\subsection{Proof of Lemma~\ref{lem:local_dev_all_k}}
\label{sec:local_dev_all_k_proof}
To prove Lemma~\ref{lem:local_dev_all_k}, we need to show that $G_i^{(t)}\leq \frac{1}{2}$ and $\Lambda_i^{(t)} \leq \Lambda^{(t)} + \frac{\delta_1}{2}$. As Lemma~\ref{lem:local_dev_all_k} are required to prove Lemmas~\ref{lem:err_centers_k} and ~\ref{lem:err_labels_k}, we use an inductive argument, so that all three of Lemma~\ref{lem:local_dev_all}, ~\ref{lem:err_centers_k} and ~\ref{lem:err_labels_k} are satisfied upto some iteration $t$. Then, we show that Lemma~\ref{lem:local_dev_all_k} is true at iteration $t+1$ and as a consequence, both Lemmas~\ref{lem:err_centers_k} and ~\ref{lem:err_labels_k} is also true.

\subsubsection{Base Case $t=1$}

First, we prove the base case for Lemma~\ref{lem:local_dev_all_k} at iteration $t=1$. 

\paragraph{Bound on $G_i^{(1)}$}
To compute $G_i^{(1)}$, note that we can write a version of Lemma~\ref{lem:err_labels_k_L} at $t=0$ for a single machine. Note that directly applying Lemma~\ref{lem:err_labels_k_L} for each machine $i\in [m]$ would not suffice, as we would require a union bound over all machines which would result in high probability terms of $mn^{-3}$ and $m\exp(-r_K^2)$. As the only condition required on $m$ is that $\log m = \mathcal{O}(n)$, these terms might be larger than $1$ making the guarantee vacuous. To overcome this issue, we use a slightly modified bound on $G^{(1)}$ obtaining high probability terms of the form $\exp(-n)$.

We use the same decomposition of $\nu_{kh,i}^{(t)}$ used in the proof of Lemma~\ref{lem:err_labels_k_L} which is presented in ~\cite{lu_statistical_2016}. If $\Lambda^{(0)} \leq \frac{1 - \epsilon'}{2}$, then we have, 
\begin{align*}
    \nu_{kh,i}^{(1)} \leq \underset{I_1'}{\underbrace{\sum_{j\in S_{k,i}^\star}\I\{\frac{(\epsilon')^2}{4}\norm{\theta_h - \theta_k}^2 \leq \ip{\vw_{i,j}}{\theta_h - \theta_k}\}}} + \underset{I_2'}{\underbrace{\sum_{j\in S_{k,i}^\star}\frac{16}{(\epsilon')^4\Gamma^4}\ip{\vw_{i,j}}{\Gamma_{h,i}^{(0)} - \Gamma_{k,i}^{(0)}}^2}}
\end{align*}
Note that the terms $I_1'$ and $I_2'$ resemble the corresponding terms $I_1$ and $I_2$ in proof of Lemma~\ref{lem:err_labels_k_L}.
For the term $I_2$, we use the same bound as Lemma~\ref{lem:err_labels_k}, but take a union bound over all machines. Therefore, with probability $1 - m\exp(-n)$, we have,
\begin{align*}
    I_2' \leq \frac{384(\nu_{k,i}^\star + d)}{(\epsilon')^4\Gamma^2} (\Lambda^{(0)})^2
\end{align*}
Note that we could do this as $\Gamma_{k,i}^{(0)} = \Gamma_{k}^{(0)}, \forall i\in [m], k\in [K]$ as the iteration $t=0$ is divisible by $L$, so the same model is used for all machines.

For the term $I_1'$, as it is a sum of $\nu_{k,i}^\star$ Bernoulli random variables with bias $\Pr[\frac{(\epsilon')^2}{4}\norm{\theta_h - \theta_k}^2 \leq \ip{\vw_{i,j}}{\theta_h - \theta_k}]$, by Hoeffding's inequality, we have, with probability $1-m\exp(-\frac{\delta_3}{K^2}\nu_{k,i}^\star )$, $\forall i\in [m]$, for some constant $\delta_3>0$. 
\begin{align*}
    I_1' \leq \nu_{k,i}^\star \exp(-r_K^2) + \frac{\delta_3}{K}\nu_{k,i}^\star
\end{align*}
Further, $\Pr[\frac{(\epsilon')^2}{4}\norm{\theta_h - \theta_k}^2 \leq \ip{\vw_{i,j}}{\theta_h - \theta_k}] \leq \exp(-r_K^2)$. Note that we can handle the high probability bound by using $\nu_{k,i}^\star \geq \beta m \nu_{k}^\star \geq \alpha \beta m n$. Since $\log m = \mathcal{O}(n)$, $\beta m = \mathcal{O}(1)$ and $\delta_1 \leq 0.1$, the high probability bound is $\exp(-n)$.

Summing over $h\in [K], h\neq k$, and dividing by $\nu_{k,i}^\star$ we obtain the false positive rate. 

\begin{align*}
    \frac{\sum_{h\in [K], h\neq k}\nu_{kh,i}^{(1)}}{\nu_{k,i}^\star} \leq K\exp(-r_K^2) + \delta_3 + \left(\frac{\sqrt{384\max\{\frac{d}{n},1\}}}{(\epsilon')^2r_K}\Lambda^{(0)}\right)^2
\end{align*}
Now, choosing $\epsilon' = \frac{6}{\sqrt{r_K}}$ and $\sqrt{\frac{d}{n}}\leq 3.3$, we can choose a small $\delta_3$ such that the false positive rate is $\leq \frac{1}{2}$ for large $r_K$.

Further, this implies $\nu_{k,i}^{(1)} \geq \frac{1}{2}\nu_{kk,i}^\star \geq \frac{1}{2}\alpha \beta mn$. Under these conditions for $\beta m = \mathcal{O}(1)$ and $\frac{d}{n} = \mathcal{O}(1)$, we can again obtain that the false positive rate, $\frac{\sum_{h\in [K] h\neq k}\nu_{hk,i}^{(1)}}{\nu_{k,i}^{(1)}}\leq \frac{1}{2}$. This implies that $G_i^{(1)}\leq \frac{1}{2}$.

\paragraph{Bound on $\Lambda_i^{(1)}$}
We prove a stronger condition that implies the bound on $\Lambda_i^{(t)}$
\begin{align}\label{eq:new_lamba_dev}
    \max_{i\neq i' \in [m], k\in [K]}\norm{\hat{\theta}_{k,i}^{(t)} - \hat{\theta}_{k,i'}^{(t)}} \leq \frac{\delta_1\Gamma}{2}, \quad \forall t\geq 1
\end{align}
Note that if this equation is satisfied, then, by applying triangle inequality, $\forall i\in [m]$,
\begin{align*}
    \norm{\hat{\theta}_{k,i}^{(t)} - \theta_k} - \norm{\hat{\theta}_{k,i} - \theta_k} \leq& \norm{\hat{\theta}_{k,i}^{(t)} - \hat{\theta}_{k,i'}^{(t)}} = \norm{\sum_{i'\in [m], i'\neq i}\frac{\nu_{k,i'}^{(t)}}{\nu_{k}^{(t)}}(\hat{\theta}_{k,i}^{(t)} - \hat{\theta}_{k,i'}^{(t)})}\\
    \leq& \sum_{i'\in [m], i'\neq i}\frac{\nu_{k,i'}^{(t)}}{\nu_{k}^{(t)}}\norm{\hat{\theta}_{k,i}^{(t)} - \hat{\theta}_{k,i'}^{(t)}} \leq \frac{\delta_1}{2}\\
\end{align*}
This implies $\Lambda_i^{(t)} \leq \Lambda^{(t)} + \frac{\delta_1}{2}, \forall i\in [m]$.

To prove the base case for this inequality, consider the following decomposition.
\begin{align*}
    \norm{\hat{\theta}_{k,i}^{(1)} - \hat{\theta}_{k,i'}^{(1)}} = \left\|\sum_{h\in [K], h\neq k}\left(\frac{\nu_{hk,i}^{(1)}}{\nu_{k,i}^{(1)}} - \frac{\nu_{hk,i'}^{(1)}}{\nu_{k,i'}^{(1)}}\right)(\theta_h - \theta_k)\right\| + \frac{\norm{W_{S_{k,i}^{(1)}}}}{\nu_{k,i}^{(1)}} + \frac{\norm{W_{S_{k,i'}^{(1)}}}}{\nu_{k,i'}^{(1)}}
\end{align*}
Note that the last two noise terms here can be bounded in the following way using Lemma~\ref{lem:noise_set_subg}, with probability $1 - \frac{m(m-1)}{2}\exp(-n)$.
\begin{align*}
    \frac{\norm{W_{S_{k,i}^{(1)}}}}{\nu_{k,i}^{(1)}} \leq \sigma \sqrt{\frac{3(n + d)}{\nu_{k,i}^{(1)}}} \leq \frac{\Gamma}{r_K} \cdot \sqrt{\frac{12\alpha\max\{\frac{d}{n},1\}}{\beta m}}
\end{align*}
Note that the high probability terms can be handled by $\log m = \mathcal{O}(n)$. Further, $\nu_{k,i}^{(1)} \geq \nu_{kk,i}^{(1)} \geq (1-G_i^{(1)})\nu_{k,i}^\star \geq \frac{1}{2}\alpha\beta mn$

For the first term, we again use triangle inequality,
\begin{align*}
    &\left\|\sum_{h\in [K], h\neq k}\left(\frac{\nu_{hk,i}^{(1)}}{\nu_{k,i}^{(1)}} - \frac{\nu_{hk,i'}^{(1)}}{\nu_{k,i'}^{(1)}}\right)(\theta_h - \theta_k)\right\|\leq \sum_{h\in [K], h\neq k}\left|\frac{\nu_{hk,i}^{(1)}}{\nu_{k,i}^{(1)}} - \frac{\nu_{hk,i'}^{(1)}}{\nu_{k,i'}^{(1)}}\right| \norm{\theta_h - \theta_k}\\
    &\leq \sum_{h\in [K], h\neq k}\left(\frac{\abs{\nu_{k,i}^{(1)}-\nu_{k,i'}^{(1)}}\nu_{hk,i'}^{(1)}}{\nu_{k,i}^{(1)}\nu_{k,i'}^{(1)}} + \frac{\abs{\nu_{hk,i'}^{(1)} - \nu_{hk,i}^{(1)}}}{\nu_{k,i}^{(1)}}\right)\lambda \Gamma
\end{align*}
Consider the first term. Note that the summation for the first term gives us $G_{i'}^{(1)}\leq \frac{1}{2}$.
Therefore, we need to bound the following for the first term.
\begin{align*}
\abs{\nu_{k,i}^{(1)}-\nu_{k,i'}^{(1)}} \leq \abs{\nu_{kk,i'}^{(1)} - \nu_{kk,i}^{(1)}} + \sum_{h\in [K], h\neq k}\abs{\nu_{hk,i'}^{(1)} - \nu_{hk,i}^{(1)}}
\end{align*}
Therefore, we only need to bound the terms $\abs{\nu_{kk,i'}^{(1)} - \nu_{kk,i}^{(1)}}$ and $\abs{\nu_{hk,i'}^{(1)} - \nu_{hk,i}^{(1)}}$ individually.
The bound for these terms is similar.
Note that the following bound holds for any $h,k\in [K]$ including the case of $h=k$.
\begin{align*}
    \nu_{hk,i'}^{(1)} - \nu_{hk,i}^{(1)} = \sum_{j\in [n]} \I\{\hat{z}_{i',j}^{(1)} =k, z_{i',j} = h\} -  \sum_{j\in [n]} \I\{\hat{z}_{i,j}^{(1)} =k, z_{i,j} = h\} 
\end{align*}
Note that this is difference of sums of $n$ Bernoulli random variables having the same bias. Note that the bias is the same because for each client, the cluster index at $t=1$ is determined by the same global model $\{\hat{\theta}_k^{(0)}\}_{k\in [K]}$. Therefore, by Hoeffding's inequality, $\forall h,k\in [K], \forall i,i'\in [m], i\neq i'$, with probability $1- \frac{K(K+1)m(m-1)}{2}\exp( - 2C^2 n)$, we have, 
\begin{align*}
    \abs{\nu_{hk,i'}^{(1)} - \nu_{hk,i}^{(1)}} \leq  4Cn
\end{align*}
for some constant $C>0$. Note that the high probability term is $\exp(-n)$ for constant $K=\mathcal{O}(1)$.

Plugging this bound into the expression for Eq~\eqref{eq:new_lamba_dev} and setting $\min\{\nu_{k,i}^{(1)}, \nu_{k,i'}^{(1)}\} \geq \frac{1}{2}\alpha \beta m n$
\begin{align*}
    \Gamma^{-1}\norm{\hat{\theta}_{k,i}^{(1)} - \hat{\theta}_{k,i'}^{(1)}} \leq \frac{12KC}{\alpha \beta m}\lambda  + \frac{2}{r_K} \cdot \sqrt{\frac{12\alpha\max\{\frac{d}{n},1\}}{\beta m}}
\end{align*}
For constant $\beta m$, we can choose $C = \frac{\delta_1\alpha \beta m}{48 K}$ so that the high probability term is $\exp(-n)$ and both the terms in the RHS are $\leq \frac{\delta_1}{4}$. Note that this also requires large $r_K$.

\subsubsection{Induction Step}
We assume that $G_i^{(\tau)} \leq \frac{1}{2}$ and Eq~\eqref{eq:new_lamba_dev} hold for all iterations from $\tau=1$ to $\tau=t$. Now, we prove that these inequalities are true for iteration $t+1$. 
Note that the bounds on $G_i^{(t)}$ and $\Lambda_i^{(t)}$ are same as those for the $2$-cluster case, which allows us to show,
\begin{align*}
    G_i^{(t+1)}\leq\frac{1}{2}, \quad \norm{\hat{\theta}_{k,i}^{(t+1)} - \hat{\theta}_{k,i'}^{(t+1)}}\leq \frac{\delta_1\Gamma}{2}
\end{align*}
\paragraph{Bound on $G_i^{(t)}$}
To compute a bound on $G_i^{(t)}$, we compute $\nu_{kh,i}^{(t+1)}$ by a decomposition for  Lemma~\ref{lem:err_labels_k}.  As long as $\Lambda^{(t)} \leq \frac{1 - \epsilon' - \delta_2}{2}$ and $\Lambda_i^{(t)}\leq \Lambda^{(t)} + \frac{\delta_1}{2}$, we have, 

\begin{align*}
\nu_{kh,i}^{(t+1)} \leq& \underset{I_1''}{\underbrace{\sum_{ \in S_{k,i}^\star}I\{\frac{(\epsilon')^2}{4}\norm{\theta_k  - \theta_h}^2  \leq  2\ip{\vw_{i,j}}{\theta_h  - \theta_k}\}}} + \underset{I_2''}{\underbrace{\sum_{j \in S_{k,i}^\star} \frac{16}{(\epsilon')^4\Gamma^4}\ip{\vw_{i,j}}{\Gamma_h^{(t)} - \Gamma_k^{(t)}}^2}}\\
    &\quad +\underset{I_3''}{\underbrace{\sum_{j \in S_{k,i}^\star} \frac{4}{\delta_2^4 \Gamma^4}\ip{\vw_{i,j}}{\Delta_{h,i}^{(t)} - \Delta_{k,i}^{(t)}}^2}}
\end{align*}
Note that $I_1'', I_2''$ and $I_3''$ correspond to the terms $I_1, I_2$ and $I_3$ respectively in the proof of Lemma~\ref{lem:err_labels_k}. We slightly modify the bounds on these inequalities to ensure that we obtain a high probability term of $\exp(-n)$. 

Consider the term $I_1''$. By using Hoeffding's inequality, for all machines and all pairs of clusters, with probability $1-\frac{mK(K-1)}{2}\exp(-2C^2n)$, we have for some constant $C>0$

\begin{align*}
I_1'' \leq \nu_{k,i}^\star\exp(-(\epsilon' r_K)^2) + 2Cn     
\end{align*}

For the term $I_2''$, we use the corresponding bound from proof of Lemma~\ref{lem:err_labels_k}, but on $n$ points instead. Further, we take a union bound over all machines and pairs of clusters. Therefore, with probability $1 - \frac{K(K-1)m}{2}\exp(-n)$, we have,
\begin{align*}
    I_2'' \leq \frac{384(\nu_{k,i}^\star + d)\sigma^2}{(\epsilon')^4 \Gamma^2}(\Lambda^{(t)})^2 \leq \frac{768\alpha n \max\{\frac{d}{n},1\}}{(\epsilon')^4 r_K^2} (\Lambda^{(t)})^2 \leq \frac{16 \alpha n \max\{\frac{d}{n},1\}}{(\epsilon')^4 r_K^2} 
\end{align*}
We finally use the fact that $\Lambda^{(t)}\leq 0.2, \forall t\geq 1$.

For the term $I_3''$, we again use the induction hypothesis (Eq~\eqref{eq:new_lamba_dev}).

\begin{align*}
    I_3'' \leq& \frac{192 n\alpha \max\{\frac{d}{n},1\} }{\delta_2^4 r_K^2}\max_{k\in [K], i\in [m]}\norm{\hat{\theta}_i^{(t)} - \hat{\theta}^{(t)}}^2 \\
    \leq& \frac{192 n\alpha \max\{\frac{d}{n},1\} }{\delta_2^4 r_K^2}\max_{k\in [K], i\in [m]}\sum_{i'\in [m], i'\neq i} \frac{\nu_{k,i'}^{(t)}}{\nu_{k}^{(t)}}\norm{\hat{\theta}_{k,i}^{(t)} - \hat{\theta}_{k,i'}^{(t)}}^2\\
    \leq & \frac{48 n\alpha \max\{\frac{d}{n},1\}\delta_1^2 }{\delta_2^4 r_K^2}\leq \frac{ n\alpha \max\{\frac{d}{n},1\}}{2\delta_2^4 r_K^2}
\end{align*}
We we Cauchy-Schwartz inequality followed by plugging in the value of $\delta_1\leq 0.1$.

The true negative rate is the following,
\begin{align*}
    \sum_{h\in [K], h\neq k} \frac{\nu_{kh,i}^{(t+1)}}{\nu_{k,i}^\star} \leq \frac{K(I_1'' + I_2'' + I_3'')}{\nu_{k,i}^\star}
\end{align*}
Since $\nu_{k,i}^\star \geq \alpha \beta mn$, we have, 
\begin{align*}
    \sum_{h\in [K], h\neq k} \frac{\nu_{kh,i}^{(t+1)}}{\nu_{k,i}^\star} \leq \frac{K}{\beta m }\left(\exp(-(\epsilon' r_K)^2) + 2C + \frac{16 \alpha \max\{\frac{d}{n},1\}}{(\epsilon')^4 r_K^2} + \frac{\alpha \max\{\frac{d}{n},1\}}{2\delta_2^4 r_K^2}\right) 
\end{align*}
Note that for large $r_K$, constant $\beta m$ and $K$, small $\frac{d}{n}$, we can choose constants $C,\epsilon'$ and $\delta_2$, since $\Lambda^{(t)} \leq 0.2$ such that the RHS is $\leq \frac{1}{2}$.
Further, for the false positive rate, we change the denominator from $\nu_{k,i}^\star$ to $\nu_{k,i}^{(t)}$. By the above bound on the true negative rate, $\nu_{k,i}^{(t)} \geq \nu_{kk,i}^{(t)}\frac{1}{2}\nu_{k,i}^\star \geq \frac{1}{2}\alpha \beta mn$. Note that this only changes the numerator of the true negative rate by a factor of $2$. Again, the constants $C,\epsilon'$ and $\delta_2$ can be chosen such that the false positive rate is also $\leq \frac{1}{2}$. This implies $G_i^{(t+1)}\leq \frac{1}{2}$.

\paragraph{Bound on $\Lambda_i^{(t)}$}
From the analysis in the base case, we only need to bound the following term $\forall h,k\in [K], i\neq i'\in [m]$.
\begin{align*}
    \abs{\nu_{hk,i}^{(t+1)} - \nu_{hk,i'}^{(t+1)}}
\end{align*}

First, consider the case when $h\neq k$. Using the decomposition of indicator functions, 
\begin{align*}
    &\I\{\hat{z}_{i,j}^{(t+1)} = h, z_{i,j} = k\} \leq \I\{\hat{z}_{i,j}^{(t+1)} = h, z_{i,j} = k\} + \Xi_{i,i',h,k,j},\quad \text{ where, }\\
    & \Xi_{i,i',h,k,j} \leq \I\{\norm{\theta_k - \hat{\theta}_{h,i}^{(t)}}^2 + \norm{\theta_k - \hat{\theta}_{k,i'}^{(t)}}^2 - \norm{\theta_k - \hat{\theta}_{h,i'}^{(t)}}^2 - \norm{\theta_k - \hat{\theta}_{k,i}^{(t)}}^2 \leq 2\ip{\vw_{i,j}}{\hat{\theta}_{k,i}^{(t)}   - \hat{\theta}_{h,i}^{(t)}}\\
    &\quad\quad\quad\quad\quad\quad - 2\ip{\vw_{i',j}}{\hat{\theta}_{k,i'}^{(t)}  - \hat{\theta}_{h,i'}^{(t)}}\}
\end{align*}
We consider the LHS and RHS of the components inside the indicator function individually.
Consider the LHS
\begin{align*}
    &\norm{\theta_k - \hat{\theta}_{h,i}^{(t)}}^2 + \norm{\theta_k - \hat{\theta}_{k,i'}^{(t)}}^2 - \norm{\theta_k - \hat{\theta}_{h,i'}^{(t)}}^2 - \norm{\theta_k - \hat{\theta}_{k,i}^{(t)}}^2 \geq ((\epsilon')^2 + \delta_2^2)\norm{\theta_k - \theta_h}^2
\end{align*}
Note that to show this is true as long as $(\epsilon')^2 + (\delta_2)^2 \leq \frac{\delta_1^2}{2}$ by the induction step.

For the term on RHS, we decompose it as 
\begin{align*}
    & 2\ip{\vw_{i,j}}{\hat{\theta}_{k,i}^{(t)}   - \hat{\theta}_{h,i}^{(t)}} - 2\ip{\vw_{i',j}}{\hat{\theta}_{k,i}^{(t)}  - \hat{\theta}_{h,i}^{(t)}} \\
    &= 2\ip{\vw_{i,j}-\vw_{i',j}}{\theta_k - \theta_h} + 2\ip{\vw_{i,j} - \vw_{i',j}}{\Gamma_{k}^{(t)} - \Gamma_{h}^{(t)}} + 2\ip{\vw_{i,j}}{\Delta_{k,i}^{(t)} - \Delta_{h,i}^{(t)}} - 2 \ip{\vw_{i',j}}{\Delta_{k,i'}^{(t)} - \Delta_{h,i'}^{(t)}}
\end{align*}
Now, we split $\delta_2^2$ into two parts one for each of the terms of $\Delta$ and $(\epsilon')^2$ into two parts one for the terms of $\Gamma$ and the other for the terms of $\theta_k - \theta_h$. This gives a decomposition similar to the $2$-cluster case if we sum over all $j\in [n]$.
\begin{align*}
    \abs{\nu_{hk,i}^{(t+1)} - \nu_{hk,i'}^{(t+1)}} \leq I_4 + I_5 + I_6 + I_7
\end{align*}

For the term $I_4$ which resembles $I_1''$, we use Hoeffding's to bound it for some constant $C$. Note that the noise variance is doubled now as the noise is $\vw_{i,j} - \vw_{i',j}$.
\begin{align*}
    I_4 = \sum_{i=1}^n \I\{(\epsilon')^2 \norm{\theta_k - \theta_h}^2 \leq 2\ip{\vw_{i,j}-\vw_{i',j}}{\theta_k - \theta_h}\} \leq n\exp(-r_K^2) + 8Cn
\end{align*}

For the term $I_5$, we can use the bound of $I_2''$ with double the noise variance.
\begin{align*}
    I_5 = \sum_{j \in [n]} \frac{4}{\delta_2^4 \Gamma^4}\ip{\vw_{i,j}}{\Delta_{h,i}^{(t)} - \Delta_{k,i}^{(t)}}^2 \leq \frac{32 \alpha n \max\{\frac{d}{n},1\}}{(\epsilon')^4 r_K^2}
\end{align*}

For the terms $I_6$ and $I_7$ which differ in only $i$ and $i'$, we use the bound of $I_3''$.
\begin{align*}
    I_6 = \sum_{j\in [n]}\frac{4}{\delta_2^4 \Gamma^4}\ip{\vw_{i,j}}{\Delta_{h,i}^{(t)} - \Delta_{k,i}^{(t)}}^2\leq \frac{ n\alpha \max\{\frac{d}{n},1\}}{2\delta_2^4 r_K^2}
\end{align*}
This gives us a bound on $\abs{\nu_{hk,i}^{(t+1)} - \nu_{hk,i'}^{(t+1)}}$. Note that for the bound of $\delta_1$, we need to divide by $\nu_{k,i}^{(t+1)}$ and $\nu_{k,i'}^{(t+1)}$ each of which is $\geq \frac{1}{2}\alpha \beta m n$ and multiply by $\frac{3K\lambda}{2}$. Each of the terms $I_4 - I_7$ can be made arbitrarily small for large $r_K$ and $n$. If $\beta m, K, \lambda$ are constants, we can choose $C$ such that the bound on $\norm{\hat{\theta}_{k,i}^{(t+1)}- \hat{\theta}_{k,i'}^{(t+1)}}\leq \frac{\delta_1 \Gamma}{2}$.
This completes the proof.

\end{document}